\documentclass{article}

\usepackage{iclr2022_conference,times}

\usepackage[utf8]{inputenc}
\usepackage[T1]{fontenc}
\usepackage{hyperref}
\usepackage{url}
\usepackage{booktabs}
\usepackage{amsfonts}       
\usepackage{nicefrac}       
\usepackage{microtype}      
\usepackage{xcolor}         

\usepackage{amsmath}       
\usepackage{amssymb}       
\usepackage{amsmath}       
\usepackage{amsthm}       
\usepackage{mathtools}   
\usepackage{bm}            
\usepackage{graphicx}
\usepackage{subcaption}  
\usepackage{cleveref}  
\usepackage{pifont}  
\usepackage{tabularx}
\usepackage{wrapfig}
\usepackage{thmtools}
\usepackage{thm-restate}
\usepackage{stmaryrd}  
\usepackage{enumerate}

\hypersetup{%
    colorlinks,
    linkcolor={red!50!black},
    citecolor={blue!50!black},
    urlcolor={blue!80!black}
}

\theoremstyle{plain}

\newtheorem{proposition}{Proposition}
\newtheorem{corollary}{Corollary}

\theoremstyle{plain}
\newtheorem{definition}{Definition}
\theoremstyle{plain}
\newtheorem{lemma}{Lemma}
\DeclareMathOperator{\E}{\mathbb{E}}    
\DeclareBoldMathCommand{\x}{x}
\DeclareBoldMathCommand{\X}{X}
\DeclareBoldMathCommand{\y}{y}
\DeclareBoldMathCommand{\Y}{Y}
\DeclareBoldMathCommand{\z}{z}
\DeclareBoldMathCommand{\Z}{Z}
\DeclareBoldMathCommand{\s}{s}
\DeclareBoldMathCommand{\n}{n}
\newcommand\given[1][]{\:#1\vert\:}
\newcommand\divfrom[1][]{\:#1\vert\vert\:}
\newcommand{\indep}{\rotatebox[origin=c]{90}{$\models$}}
\newcommand{\cmark}{\ding{51}}
\newcommand{\xmark}{\ding{55}}
\newcommand\labelAndRemember[2]
  {\expandafter\gdef\csname labeled:#1\endcsname{#2}\label{#1}#2}
\newcommand\recallLabel[1]
   {\csname labeled:#1\endcsname\tag{\ref{#1}}}
\makeatletter
\newenvironment{customcorollary}[1]
{\count@\c@corollary\global\c@corollary#1\global\advance\c@corollary\m@ne\corollary}
{\endcorollary\global\c@corollary\count@}
\makeatother

\title{On the Limitations of Multimodal VAEs}

\author{%
\footnotesize
Imant Daunhawer,\ \,Thomas M.~Sutter,\ \,Kieran Chin-Cheong,\ \,Emanuele Palumbo \& Julia E.~Vogt \\
Department of Computer Science\\
ETH Zurich\\
\texttt{dimant@ethz.ch}
\normalsize
}

\iclrfinalcopy%
\begin{document}

\maketitle

\vskip -1.4em  %
\begin{abstract}
  Multimodal variational autoencoders (VAEs) have shown promise as efficient
  generative models for weakly-supervised data. Yet, despite their advantage of
  weak supervision, they exhibit a gap in generative quality compared to
  unimodal~VAEs, which are completely unsupervised. In an attempt to explain
  this gap, we uncover a fundamental limitation that applies to a large family
  of mixture-based multimodal~VAEs. We prove that the sub-sampling of
  modalities enforces an undesirable upper bound on the multimodal~ELBO and
  thereby limits the generative quality of the respective models.  Empirically,
  we showcase the generative quality gap on both synthetic and real data and
  present the tradeoffs between different variants of multimodal
  VAEs. We find that none of the existing approaches fulfills all desired
  criteria of an effective multimodal generative model when applied on more
  complex datasets than those used in previous benchmarks. In summary, we
  identify, formalize, and validate fundamental limitations of VAE-based
  approaches~for~modeling~weakly-supervised~data~and~discuss~implications~for
  real-world~applications.
\end{abstract}

\section{Introduction}

In recent years, multimodal VAEs have shown great potential as efficient
generative models for weakly-supervised data, such as pairs of
images or paired images and captions.  Previous works
\citep{Wu2018,Shi2019,Sutter2020} demonstrate that multimodal VAEs leverage
weak supervision to learn generalizable representations, useful for downstream
tasks \citep{Dorent2019,Minoura2021} and for the conditional generation of missing
modalities \citep{Lee2021}.  However, despite the advantage of weak
supervision, state-of-the-art multimodal VAEs consistently underperform when
compared to simple unimodal VAEs in terms of generative quality.%
\footnote{%
  The lack of generative quality can even be recognized by visual inspection of
  the qualitative results from previous works; for instance, see the
  supplementaries of \citet{Sutter2021} or \citet{Shi2021}.
} 
This paradox serves as a
starting point for our work, which aims to explain the observed lack of
generative quality in terms of a fundamental limitation that underlies existing
multimodal VAEs.

What is limiting the generative quality of multimodal VAEs? We find that the
sub-sampling of modalities during training leads to a problem that affects all
\textit{mixture-based} multimodal VAEs---a family of models that subsumes
the MMVAE \citep{Shi2019}, MoPoE-VAE \citep{Sutter2021}, and a special case of
the MVAE \citep{Wu2018}.
We prove that modality sub-sampling enforces an undesirable upper bound on the
multimodal ELBO and thus prevents a tight approximation of the joint distribution
when there is modality-specific variation in the data.  Our experiments
demonstrate that modality sub-sampling can explain the gap in generative
quality compared to unimodal VAEs and that the gap typically increases with
each additional modality. Through extensive ablations on three different
datasets, we validate the generative quality gap between unimodal and
multimodal VAEs and present the tradeoffs between different
approaches.

Our results raise serious concerns about the utility of multimodal VAEs for
real-world applications. We show that none of the existing approaches fulfills
all desired criteria \citep{Shi2019,Sutter2020} of an effective multimodal
generative model when applied to slightly more complex datasets than used in
previous benchmarks. In particular, we demonstrate that generative coherence
\citep{Shi2019} cannot be guaranteed for any of the existing approaches, if the
information shared between modalities cannot be predicted in expectation across
modalities. Our findings are particularly relevant for applications on datasets
with a relatively high degree of modality-specific variation, which is a
typical characteristic of many real-world datasets \citep{Baltrusaitis2019}.

\section{Related work}

First, to put multimodal VAEs into context, let us point out that there is a
long line of research focused on learning multimodal generative models based on a wide
variety of methods. There are several notable generative models
with applications on pairs of modalities
\citep[e.g.,][]{Ngiam2011,Srivastava2014,Wu2020,Lin2021,Ramesh2021}, as well as
for the specialized task of image-to-image translation
\citep[e.g.,][]{Huang2018,Choi2018,Liu2019}. Moreover, generative models can
use labels as side information \citep{Ilse2019,Tsai2019,Wieser2020}; for
example, to guide the disentanglement of shared and modality-specific
information \citep{Tsai2019}. In contrast, multimodal VAEs do not require
strong supervision and can handle a large and variable number of modalities
efficiently.  They learn a joint distribution over multiple modalities, but
also enable the inference of latent representations, as well as the conditional
generation of missing modalities, given any subset of modalities
\citep{Wu2018,Shi2019,Sutter2021}.

Multimodal VAEs are an extension of VAEs \citep{Kingma2014} and they belong to
the class of multimodal generative models with encoder-decoder architectures
\citep{Baltrusaitis2019}. The first multimodal extensions of VAEs
\citep{Suzuki2016,Hsu2018,Vedantam2018} use separate inference networks for
every subset of modalities, which quickly becomes intractable as the number of
inference networks required grows exponentially with the number of modalities.
Starting with the seminal work of \citet{Wu2018}, multimodal VAEs were
developed as an \textit{efficient} method for multimodal learning.  In
particular, multimodal VAEs enable the inference of latent representations, as
well as the conditional generation of missing modalities, given any subset of
input modalities.  Different types of multimodal VAEs were devised by
decomposing the joint encoder as a product \citep{Wu2018}, mixture
\citep{Shi2019}, or mixture of products \citep{Sutter2021} of unimodal encoders
respectively.  A commonality between these approaches is the sub-sampling of
modalities during training---a property we will use to define the family of
\textit{mixture-based} multimodal VAEs. For the MMVAE and MoPoE-VAE, the
sub-sampling is a direct consequence of defining the joint encoder as a mixture
distribution over different subsets of modalities.  Further, our analysis
includes a special case of the MVAE \textit{without} ELBO sub-sampling, which
can be seen as another member of the family of mixture-based multimodal
VAEs~\citep{Sutter2021}.  The MVAE was originally proposed with ``ELBO
sub-sampling'', an additional training paradigm that was later found to result
in an incorrect bound on the joint distribution \citep{Wu2020}. While this
training paradigm is also based on the sub-sampling of modalities, the
objective differs from mixture-based multimodal VAEs in that the MVAE does not
reconstruct the missing modalities from the set of sub-sampled modalities.%
\footnote{%
  For completeness, in
  \Cref{app:experiments}, we also analyze the effect of ELBO sub-sampling.
}

\Cref{tab:overview_mvaes} provides an overview of the different variants of
mixture-based multimodal VAEs and the properties that one can infer from
empirical results in previous works \citep{Shi2019,Shi2021,Sutter2021}. Most
importantly, there appears to be a tradeoff between generative quality and
generative coherence (i.e., the ability to generate semantically related
samples across modalities). Our work explains \textit{why} the generative
quality is worse for models that sub-sample modalities (\Cref{sec:thm}) and
shows that a tighter approximation of the joint distribution can be achieved
without sub-sampling (\Cref{subsec:implications}). Through systematic
ablations, we validate the proposed theoretical limitations and showcase the
tradeoff between generative quality and generative coherence
(\Cref{subsec:generative_quality_gap}). Our experiments also reveal that
generative coherence cannot be guaranteed for more complex datasets than those
used in previous benchmarks (\Cref{subsec:lack_of_coherence}).

\section{Multimodal VAEs, in different flavors}
\label{sec:method}

\begin{table}[t]
\caption{%
  Overview of multimodal VAEs. Entries for generative quality and generative
  coherence denote properties that were observed empirically in previous works.
  The lightning symbol ($\lightning$) denotes properties for which our work
  presents contrary evidence. This overview abstracts technical details, such
  as importance sampling and ELBO~sub-sampling, which we address in
  \Cref{app:experiments}.
}
\label{tab:overview_mvaes}
\tiny{%
\begin{center}
  \begin{tabular}{lp{3.5cm}cccc}
  \toprule
    Model & 
    Decomposition of $p_{\theta}(\z \given \x)$ &
    Modality sub-sampling &
    Generative quality &
    Generative coherence
  \\
  \midrule
    \textbf{MVAE} \citep{Wu2018}&   %
    $\prod_{i=1}^{M}p_{\theta}(\z \given \x_i)$ &
    \xmark& good & poor \\
    \addlinespace[0.15cm]
    \textbf{MMVAE} \citep{Shi2019} & 
    $\frac{1}{M} \sum_{i=1}^{M} p_{\theta}(\z \given \x_i)$ &  
    \cmark& limited & \ \ \ good$\ \lightning$ \\
    \addlinespace[0.15cm]
    \textbf{MoPoE-VAE} \citep{Sutter2021} & 
    $ \frac{1}{\vert\mathcal{P}(M)\vert}\sum_{A \in \mathcal{P}(M)} \prod_{i \in A} p_{\theta}(\z \given \x_i)$ &  
    \cmark& limited & \ \ \ good$\ \lightning$ \\ 
  \bottomrule
  \end{tabular}
\end{center}
}
\end{table}

Let $\X \coloneqq \left\{ X_1, \ldots, X_M \right\} $ be a set of random
vectors describing $M$ modalities and let ${\x \coloneqq \left\{ \x_1, \ldots,
\x_M\right\}}$ be a sample from the joint distribution $p(\x_1, \ldots, \x_M)$.
For conciseness, denote subsets of modalities by subscripts; for example,
$\X_{\{1,3\}}$ or $\x_{\{1,3\}}$ respectively for modalities 1~and~3.

Throughout this work, we assume that all modalities are described by discrete
random vectors (e.g., pixel values), so that we can assume non-negative
entropy and conditional entropy terms.  Definitions for all required
information-theoretic quantities are provided in \Cref{app:infotheory}.

\subsection{The multimodal ELBO}

\begin{restatable}{definition}{defmultimodalelbo}  %
\label{def:multimodal_elbo}
Let $p_{\theta}(\z \given \x)$ be a stochastic encoder, parameterized by
$\theta$, that takes multiple modalities as input. Let $q_{\phi}(\x \given \z)$
be a variational decoder (for all modalities), parameterized by $\phi$, and let
$q(\z)$ be a prior. The multimodal evidence lower bound (ELBO) on
$\,\E_{p(\x)}[\log p(\x)]\,$ is defined as 
\begin{align}
  \mathcal{L}_{}(\x; \theta, \phi) \coloneqq \E_{p(\x)p_{\theta}(\z \given
  \x)} [ \log q_{\phi}(\x \given \z) ] - \E_{p(\x)}
  [D_{\text{KL}}(p_{\theta}(\z \given \x) \divfrom q(\z))]\;.
\end{align}
\end{restatable}

The multimodal ELBO (\Cref{def:multimodal_elbo}), first introduced by
\citet{Wu2018}, is the objective maximized by all multimodal VAEs and it forms
a variational lower bound on the expected log-evidence.%
\footnote{Even though we write the expectation over $p(\x)$, for the estimation
  of the ELBO we still assume that we only have access to a finite sample from
  the training distribution $p(\x)$. The notation is used for consistency with
  the well-established information-theoretic perspective on VAEs \citep{Alemi2017,Poole2019}.
}
The first term denotes the estimated log-likelihood of all modalities and the
second term is the KL-divergence between the stochastic encoder and the prior.
We take an information-theoretic perspective using the variational information
bottleneck (VIB) from \citet{Alemi2017} and employ the standard notation used
in multiple previous works \citep{Alemi2017,Poole2019}.  Similar to the latent
variable model approach, the VIB derives the ELBO as a variational lower bound
on the expected log-evidence, but, in addition, the VIB is a more general
framework for optimization that allows us to reason about the underlying
information-theoretic quantities of interest (for details on the VIB and its
notation, please see \Cref{app:elbo_derivation}). 

Note that the above definition of the multimodal ELBO requires that the complete
set of modalities is available. To overcome this limitation and to learn the
inference networks for different subsets of modalities, existing models use
different \textit{decompositions} of the joint encoder, as summarized in
\Cref{tab:overview_mvaes}. Recent work shows that existing models can be
generalized by formulating the joint encoder as a mixture of products of experts
\citep{Sutter2021}. Analogously, in the following, we generalize existing models
to define the family of mixture-based multimodal VAEs.

\subsection{The family of mixture-based multimodal VAEs}
\label{subsec:mixture_based_family}

Now we introduce the family of mixture-based multimodal VAEs, which subsumes
the MMVAE, MoPoE-VAE, and a special case of the MVAE without ELBO sub-sampling.
We first define an encoder that generalizes the decompositions used by existing
models:
\begin{definition}%
\label{def:mixture}
Let $\mathcal{S} = \{ (A, \omega_A) \given A \subseteq \{1,\ldots,M\},\ A \not = \emptyset,\ \omega_A \in [0, 1]\}$ 
be an arbitrary set of non-empty subsets $A$ of modalities and
corresponding mixture coefficients $\omega_A$, such that 
$\sum_{A \in \mathcal{S}} \omega_A = 1$. Define the stochastic encoder
to be a mixture distribution:
  ${p^{\mathcal{S}}_{\theta}(\z \given \x) \coloneqq \sum_{A \in \mathcal{S}} \omega_A \, p_{\theta}(\z \given \x_A)}$.
\end{definition}
In the above definition and throughout this work, we write $A \in
\mathcal{S}$ to abbreviate $(A, \omega_A) \in \mathcal{S}$.
To define the family of mixture-based multimodal VAEs, we restrict the
family of models optimizing the multimodal ELBO to the subfamily of models that
use a mixture-based stochastic encoder.
\begin{definition}  %
\label{def:mixture_based}
  The family of mixture-based multimodal VAEs is comprised of all models that
  maximize the multimodal ELBO using a stochastic encoder
  $p^{\mathcal{S}}_{\theta}(\z \given \x)$ that is consistent with
  \Cref{def:mixture}. In particular, we define the family in terms of all
  models that maximize the following objective:
  \vskip -1.5em
  \begin{equation}
    \label{eq:objective_s}
    \mathcal{L}_{\mathcal{S}}(\x; \theta, \phi) =
    \sum_{A \in \mathcal{S}} \omega_A \left\{ \E_{p(\x)p_{\theta}(\z \given \x_A)} [ \log q_{\phi}(\x \given \z) ] 
      - \E_{p(\x)} \left[ D_{\text{KL}}\left(p_{\theta}(\z \given \x_A) \divfrom q(\z)\right)
        \right]\right\} \;.
  \end{equation}
  \vskip 2.5em
\end{definition}
In \Cref{app:elbo_s_derivation}, we show that the objective
$\mathcal{L}_{\mathcal{S}}(\x; \theta, \phi)$ is a lower bound on
$\mathcal{L}(\x; \theta, \phi)$ (which makes it an ELBO) and 
explain how, for different choices of the
set of subsets $\mathcal{S}$, the objective $\mathcal{L}_{\mathcal{S}}(\x;
\theta, \phi)$ relates to the objectives of the MMVAE, MoPoE-VAE, and MVAE
without ELBO sub-sampling.

From a computational perspective, a characteristic of mixture-based multimodal
VAEs is the sub-sampling of modalities during training, which is a direct
consequence of defining the encoder as a mixture distribution over subsets of
modalities.  The sub-sampling of modalities can be viewed as the extraction of
a subset $\x_A \in \x$, where $A$ indexes one subset of modalities that is
drawn from the model-specific set of subsets $\mathcal{S}$.  The only member of
the family of mixture-based multimodal VAEs that forgoes sub-sampling, defines
a trivial mixture over a single subset, the complete set of modalities
\citep{Sutter2021}.

\section{Modality sub-sampling limits the multimodal ELBO}
\label{sec:thm}

\subsection{An intuition about the problem}
\label{subsec:intuition}

Before we delve into the details, let us illustrate how modality sub-sampling
affects the likelihood estimation, and hence the multimodal ELBO\@. 
Consider the likelihood estimation using the objective~$\mathcal{L}_{\mathcal{S}}$:
\begin{equation}
  \sum_{A \in \mathcal{S}} \omega_A \E_{p(\x)p_{\theta}(\z \given \x_A)} [ \log q_{\phi}(\x \given \z) ] \;,
\end{equation}
where $A$ denotes a subset of modalities and $\omega_A$ the respective
mixture weight. Crucially, the stochastic encoder $p_{\theta}(\z \given \x_A)$
encodes a \textit{subset} of modalities.  What seems to be a minute detail, can
have a profound impact on the likelihood estimation, because the precise
estimation of all modalities depends on information from \textit{all}
modalities.  In trying to reconstruct all modalities from incomplete
information, the model can learn an inexact, average prediction; however, it
cannot reliably predict modality-specific information, such as the background
details in an image given a concise verbal description of its content.

In the following, we formalize the above intuition by showing that, in the
presence of modality-specific variation, modality sub-sampling enforces an
undesirable upper bound on the multimodal ELBO and therefore prevents a tight
approximation of the joint distribution.

\subsection{A formalization of the problem}
\label{subsec:formalization}

\Cref{thm:irreducible_error} states our main theoretical result, which
describes a non-trivial limitation of mixture-based multimodal VAEs. Our result
shows that the sub-sampling of modalities enforces an undesirable upper bound
on the approximation of the joint distribution when there is modality-specific
variation in the data. This limitation conflicts with the goal of modeling
real-world multimodal data, which typically exhibits a considerable degree of
modality-specific variation.

\begin{restatable}{theorem}{theoremirreducible}
\label{thm:irreducible_error}
Each mixture-based multimodal VAE (\Cref{def:mixture_based}) approximates the
expected log-evidence up to an irreducible discrepancy $\Delta(\X, \mathcal{S})$ that
depends on the model-specific mixture distribution~$\mathcal{S}$ as well as on
the amount of modality-specific information in $\X$.

For the maximization of $\mathcal{L}_{\mathcal{S}}(\x; \theta, \phi)$ and
every value of $\theta$ and $\phi$, the following inequality holds:
\begin{align}
  \E_{p(\x)} [\log p_{}(\x)] 
  &\geq \mathcal{L}_{\mathcal{S}}(\x;\theta,\phi) + \Delta(\X, \mathcal{S})
\end{align}
where
\begin{equation}
  \Delta(\X, \mathcal{S}) = \sum_{A \in \mathcal{S}} \omega_A\, H(\X_{\{1, \ldots, M\} \setminus A} \given \X_A)\;.
\end{equation}
In particular, the generative discrepancy is always greater than or equal to
zero and it is independent of $\theta$ and $\phi$ and thus remains constant
during the optimization.
\end{restatable}

A proof is provided in \Cref{app:proof_of_thm} and it is based on 
\Cref{lemma:vib,lemma:decomposition}.
\Cref{thm:irreducible_error} formalizes the rationale that, in the general
case, cross-modal prediction cannot recover information that is specific to the
target modalities that are unobserved due to modality sub-sampling. In general,
the conditional entropy $H(\X_{\{1, \ldots, M\} \setminus A} \given \X_A)$
measures the amount of information in one subset of random vectors $\X_{\{1,
\ldots, M\} \setminus A}$ that is not shared with another subset $\X_A$. In our
context, the sub-sampling of modalities yields a discrepancy $\Delta(\X,
\mathcal{S})$ that is a weighted average of conditional entropies $H(\X_{\{1,
\ldots, M\} \setminus A} \given \X_A)$ of the modalities $\X_{\{1, \ldots, M\}
\setminus A}$ unobserved by the encoder given an observed subset $\X_A$. Hence,
$\Delta(\X, \mathcal{S})$ describes the modality-specific information
that cannot be recovered by cross-modal prediction, averaged over all subsets
of modalities.

\Cref{thm:irreducible_error} applies to the MMVAE, MoPoE-VAE, and a special
case of the MVAE without ELBO sub-sampling, since all of these models belong to
the class of mixture-based multimodal VAEs.  However, $\Delta(\X, \mathcal{S})$
can vary significantly between different models, depending on the mixture
distribution defined by the respective model and on the amount of
modality-specific variation in the data. In the following, we show that without
modality sub-sampling $\Delta(\X, \mathcal{S})$ vanishes, whereas for the MMVAE
and MoPoE-VAE, $\Delta(\X, \mathcal{S})$ typically increases with each
additional modality. In \Cref{sec:experiments}, we provide empirical support
for each of these theoretical statements.

\subsection{Implications of Theorem~\ref{thm:irreducible_error}}
\label{subsec:implications}

First, we consider the case of no modality sub-sampling, for which
it is easy to show that the generative discrepancy vanishes. 
\begin{restatable}{corollary}{corollarynomodalitysubsampling}
\label{cor:no_modality_subsampling}
  Without modality sub-sampling, $\Delta(\X, \mathcal{S}) = 0$\,.
\end{restatable}
A proof is provided in \Cref{app:cor:no_modality_subsampling}.  The result from
\Cref{cor:no_modality_subsampling} applies to the MVAE without ELBO
sub-sampling and suggests that this model should yield a tighter approximation
of the joint distribution and hence a better generative quality compared to
mixture-based multimodal VAEs that sub-sample modalities.  Note that this does
not imply that a model without modality sub-sampling is superior to one that
uses sub-sampling and that there can be an inductive bias that favors
sub-sampling despite the approximation error it incurs.  Especially,
\Cref{cor:no_modality_subsampling} does not imply that the variational
approximation is tight for the MVAE;\@\ for instance, the model can be
underparameterized or simply misspecified due to simplifying assumptions, such
as the PoE-factorization \citep{Kurle2018}.

Second, we consider how additional modalities might affect the generative
discrepancy. \Cref{cor:more_modalities} predicts an increased generative
discrepancy (and hence, a decline of generative quality) when we increase the
number of modalities for the MMVAE and MoPoE-VAE\@.
\begin{restatable}[informal]{corollary}{corollarymoremodalities}
\label{cor:more_modalities}
  For the MMVAE and MoPoE-VAE, the generative discrepancy increases with each
  additional modality, if the new modality is sufficiently diverse.
\end{restatable}
A proof is provided in \Cref{app:cor:more_modalities}.
The notion of \textit{diversity} requires a more formal treatment of the
underlying information-theoretic quantities, which we defer to
\Cref{app:cor:more_modalities}.  Intuitively, a new modality is sufficiently
diverse, if it does \textit{not} add too much redundant information with
respect to the existing modalities.  In special cases when there is a lot of
redundant information, $\Delta(\X, \mathcal{S})$ can decrease given an
additional modality, but it does not vanish in any one of these cases. Only if
there is very little modality-specific information in \textit{all} modalities,
we have $\Delta(\X, \mathcal{S}) \rightarrow 0$ for the MMVAE and MoPoE-VAE\@.
This condition requires modalities to be extremely similar, which does not
apply to most multimodal datasets, where $\Delta(\X, \mathcal{S})$ typically
represents a large part of the total variation.

In summary, \Cref{thm:irreducible_error} formalizes how the family of
mixture-based multimodal VAEs is fundamentally limited for the task of
approximating the joint distribution, and
\Cref{cor:no_modality_subsampling,cor:more_modalities} connect this result to
existing models---the MMVAE, MoPoE-VAE, and MVAE without ELBO sub-sampling.  We
now turn to the experiments, where we present empirical support for the
limitations described by \Cref{thm:irreducible_error} and its Corollaries.

\section{Experiments}
\label{sec:experiments}

\Cref{fig:dataset_examples} presents the three considered datasets. PolyMNIST
\citep{Sutter2021} is a simple, synthetic dataset with five image modalities
that allows us to conduct systematic ablations. ${\text{Translated-PolyMNIST}}$
is a new dataset that adds a small tweak---the downscaling and random
translation of digits---to demonstrate the limitations of existing methods when
shared information cannot be predicted in expectation across modalities.
Finally, Caltech Birds \citep[CUB;][]{Wah2011,Shi2019} is used to validate the
limitations on a more realistic dataset with two modalities, images and
captions. Please note that we use CUB with \textit{real images} and not the
simplified version based on precomputed ResNet-features that was used in
\citet{Shi2019} and \citet{Shi2021}.  For a more detailed description of the
three considered datasets, please see \Cref{app:description_of_data}.

In total, more than 400 models were trained, requiring approximately 1.5 GPU
years of compute on a single NVIDIA GeForce RTX 2080 Ti GPU\@. For the
experiments in the main text, we use the publicly available code from
\citet{Sutter2021} and in \Cref{app:additional_experimental_results} we also
include ablations using the publicly available code from \citet{Shi2019}, which
implements importance sampling and alternative ELBO objectives.  To provide a
fair comparison across methods, we use the same architectures and similar
capacities for all models. For each unimodal VAE, we make sure to decrease the
capacity by reducing the latent dimensionality proportionally with respect to
the number of modalities. Additional information on architectures,
hyperparameters, and evaluation metrics is provided in \Cref{app:experiments}. 

\begin{figure}[t]
\captionsetup[subfigure]{justification=centering}
\begin{center}
  \begin{subfigure}[t]{.33\linewidth}
  \centering
  \includegraphics[width=1.0\linewidth]{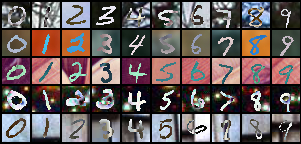}
  \caption{PolyMNIST\\(5 modalities)}
\label{fig:vanilla_polymnist_example}
\end{subfigure}%
\hskip +0.050in
\begin{subfigure}[t]{.33\linewidth}
  \centering
  \includegraphics[width=0.95\linewidth]{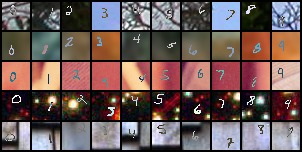}
  \caption{Translated-PolyMNIST\\(5 modalities)}
\label{fig:translated_polymnist_example}
\end{subfigure}%
\begin{subfigure}[t]{.33\linewidth}
  \centering
  \includegraphics[width=0.95\linewidth]{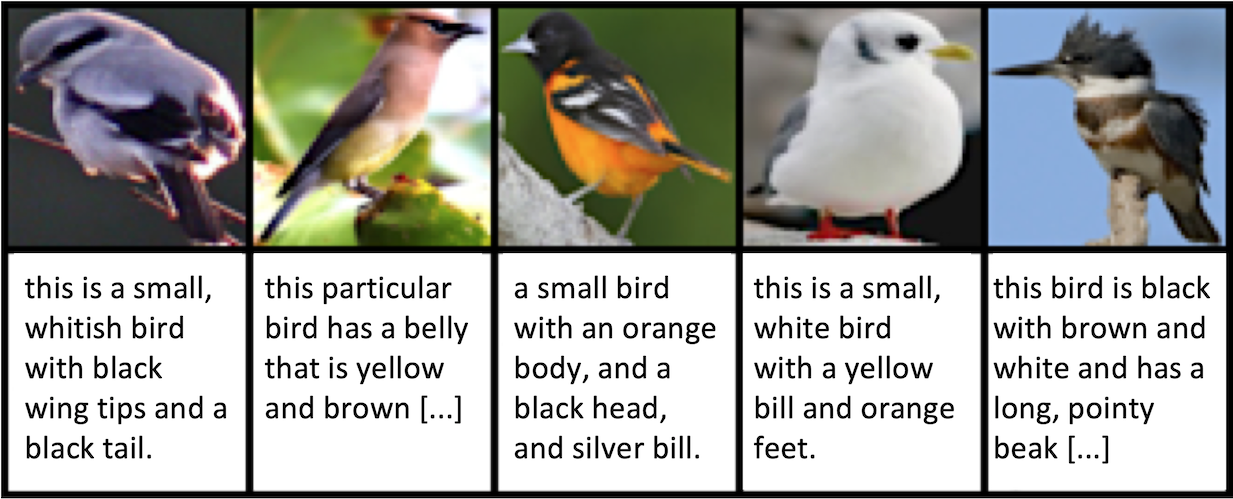}
  \caption{Caltech Birds (CUB)\\(2 modalities)}
\label{fig:cub_example}
\end{subfigure}
\caption{%
The three considered datasets. Each subplot shows samples from the respective
dataset. The two PolyMNIST datasets are conceptually similar in that the digit
label is shared between five synthetic modalities. The Caltech Birds (CUB)
dataset provides a more realistic application for which there is no annotation
on what is shared between paired images and captions.}
\label{fig:dataset_examples}
\end{center}
\end{figure}

\subsection{The generative quality gap}
\label{subsec:generative_quality_gap}

We assume that an increase in the generative discrepancy $\Delta(\X,
\mathcal{S})$ is associated with a drop of generative quality. However, we want
to point out that there can also be an inductive bias that favors modality
sub-sampling despite the approximation error that it incurs. In fact, our
experiments reveal a fundamental tradeoff between generative quality and
generative coherence when shared information can be predicted in expectation
across modalities. 

We measure generative quality in terms of Fr\'echet inception
distance \citep[FID;][]{Heusel2017}, a standard metric for evaluating the
quality of generated images. Lower FID represents better generative quality
and the values typically correlate well with human perception \citep{Borji2019}. 
In addition, in \Cref{app:experiments} we provide log-likelihood values, as
well as qualitative results for all modalities including captions, for which
FID cannot be computed.

\Cref{fig:fids} presents the generative quality across a range of
$\beta$ values.%
\footnote{%
  The regularization coefficient $\beta$ weights the KL-divergence term of the
  multimodal ELBO (\Cref{def:multimodal_elbo,def:mixture_based}) and it is
  arguably the most impactful hyperparameter in VAEs \citep[e.g.,
  see][]{Higgins2017}.
}
To relate different methods, we compare models with the \textit{best} FID
respectively, because different methods can reach their optima at different
$\beta$ values. As described by \Cref{thm:irreducible_error}, mixture-based
multimodal VAEs that sub-sample modalities (MMVAE and MoPoE-VAE) exhibit a
pronounced generative quality gap compared to unimodal VAEs.  When we compare
the best models, we observe a gap of more than 60 points on both PolyMNIST and
Translated-PolyMNIST, and about 30 points on CUB images. Qualitative results
(\Cref{fig:qualitative_unconditional} in
\Cref{app:additional_experimental_results}) confirm that this gap is clearly
visible in the generated samples and that it applies not only to image
modalities, but also to captions.  In contrast, the MVAE (without ELBO
sub-sampling) reaches the generative quality of unimodal VAEs, which is in line
with our theoretical result from \Cref{cor:no_modality_subsampling}. For
completeness, in \Cref{app:additional_experimental_results}, we also report
joint log-likelihoods, latent classification performance, as well as additional
FIDs for all modalities.

\Cref{fig:num_mod_ablation} examines how the generative quality is affected
when we vary the number of modalities. Notably, for the MMVAE and MoPoE-VAE,
the generative quality deteriorates almost continuously with the number of
modalities, which is in line with our theoretical result from
\Cref{cor:more_modalities}.  Interestingly, for the MVAE, the generative
quality on {$\text{Translated-PolyMNIST}$} also decreases slightly, but the
change is comparatively small. \Cref{fig:num_mod_ablation_repeated_modality} in
\Cref{app:additional_experimental_results}, shows a similar trend even when we
control for modality-specific differences by generating PolyMNIST using the
\textit{same} background image for all modalities.

In summary, the results from \Cref{fig:fids} and \Cref{fig:num_mod_ablation}
provide empirical support for the existence of a generative quality gap between
unimodal and mixture-based multimodal VAEs that sub-sample modalities. The
results verify that the approximation of the joint distribution improves
for models without sub-sampling, which manifests in better generative
quality. In contrast, the gap increases disproportionally with each additional
modality for both the MMVAE and MoPoE-VAE\@. Hence, the presented results
support all of the theoretical statements from
\Cref{subsec:formalization,subsec:implications}.

\begin{figure}[t]
\begin{center}
\begin{subfigure}[t]{.33\linewidth}
  \centering
  \includegraphics[width=1.0\linewidth]{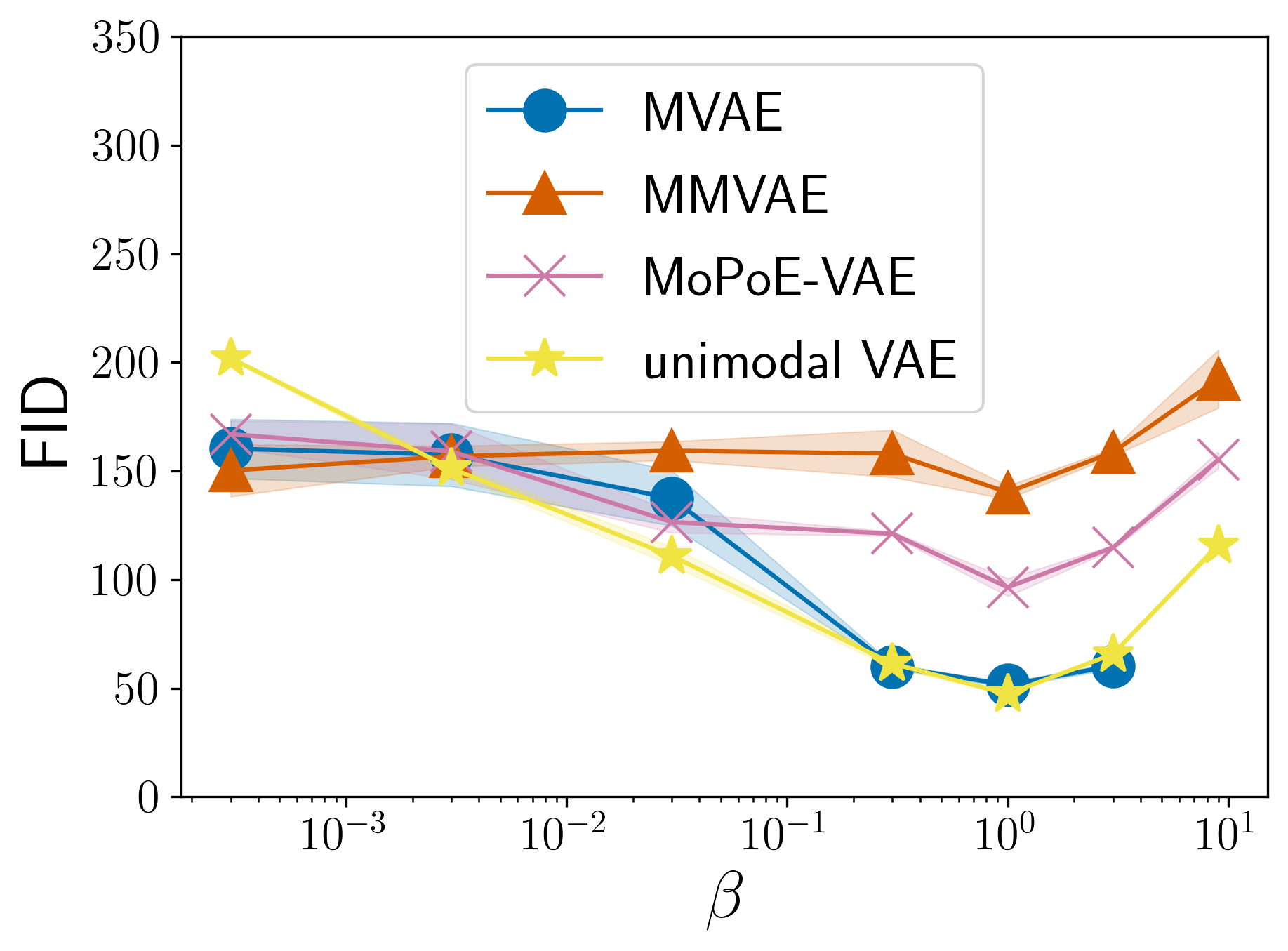}
  \caption{PolyMNIST}
\end{subfigure}%
\begin{subfigure}[t]{.33\linewidth}
  \centering
  \includegraphics[width=1.0\linewidth]{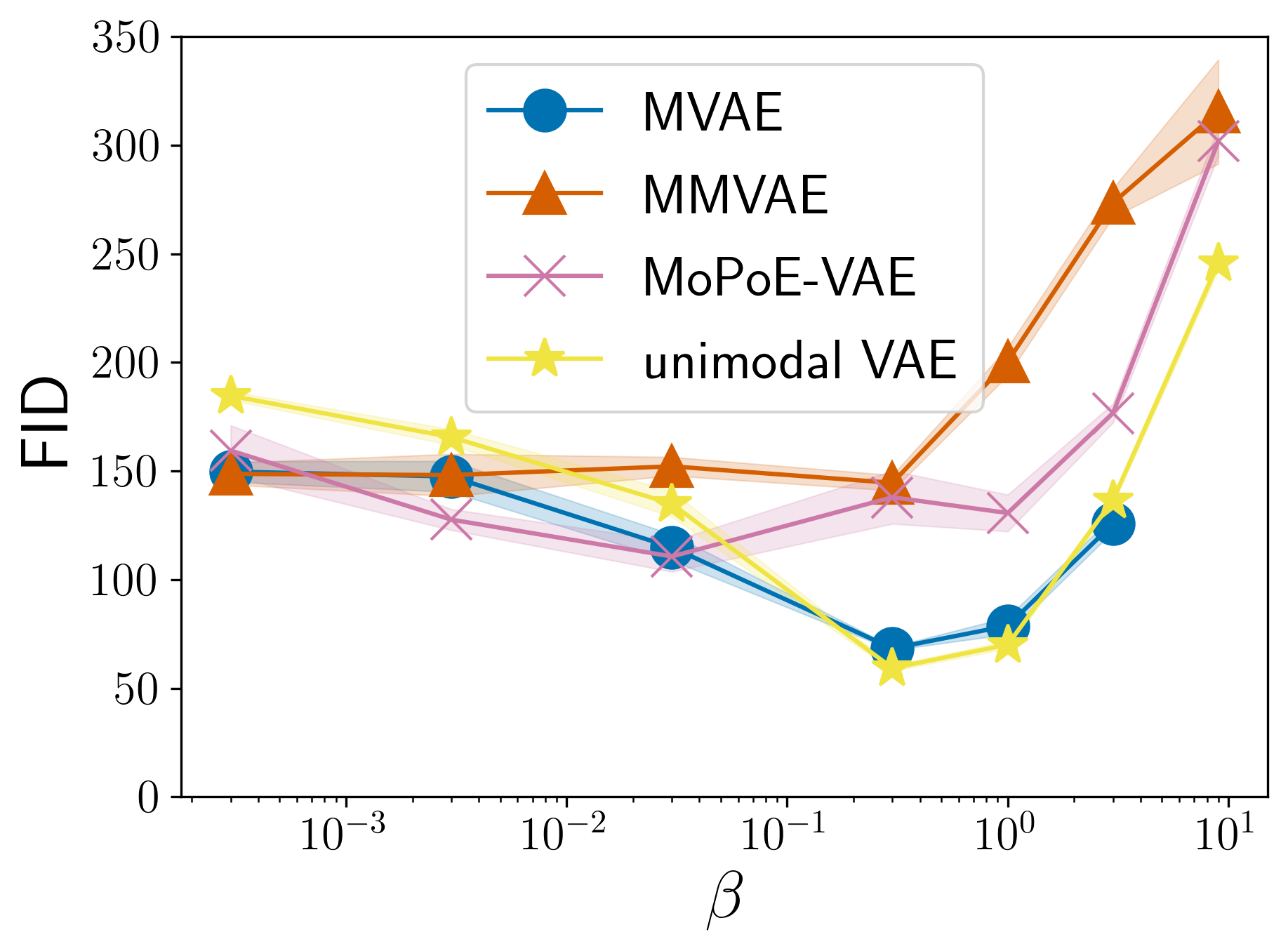}
  \caption{Translated-PolyMNIST}
\end{subfigure}%
\begin{subfigure}[t]{.33\linewidth}
  \centering
  \includegraphics[width=1.0\linewidth]{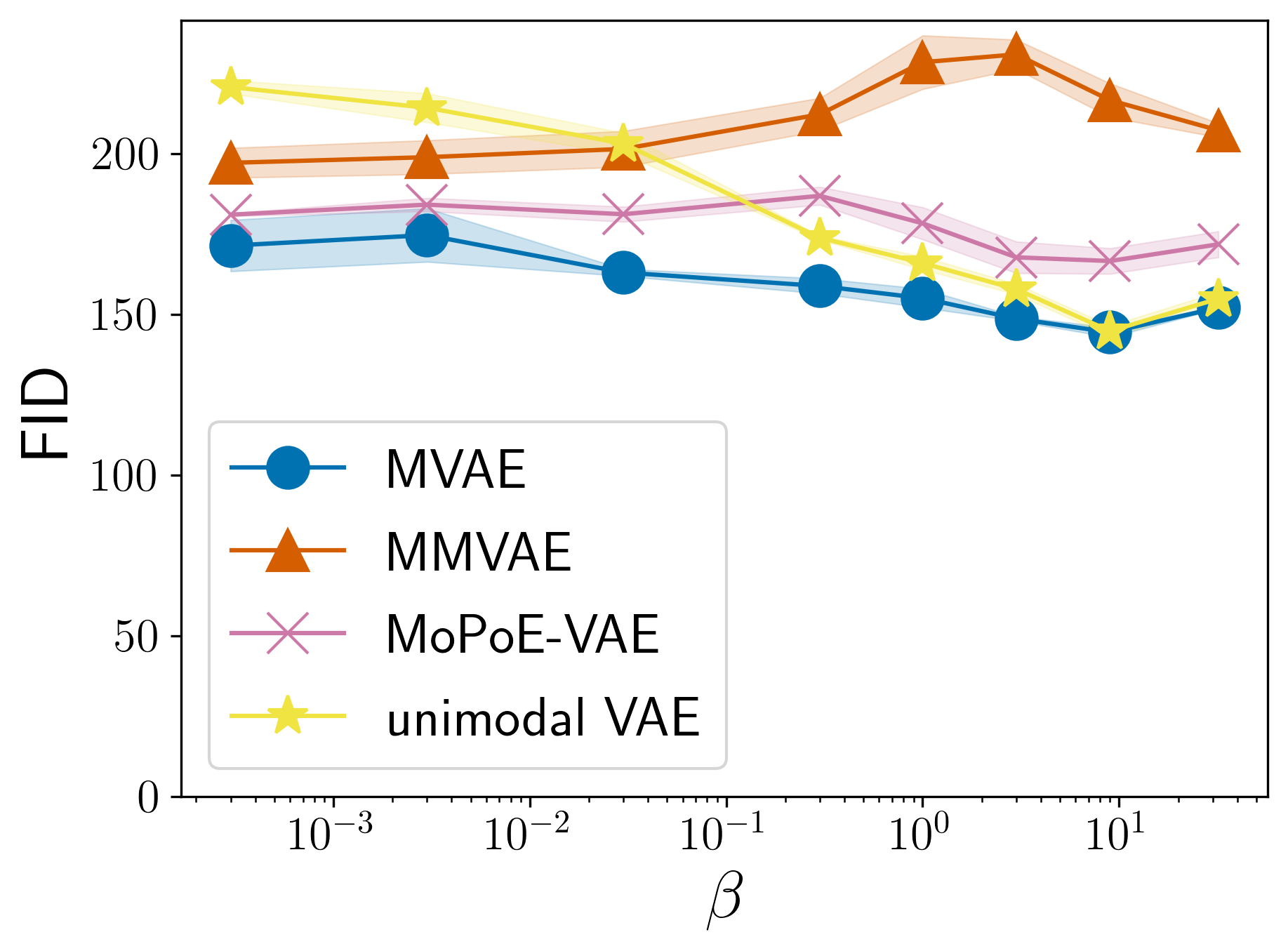}
  \caption{Caltech Birds (CUB)}
\end{subfigure}
\caption{%
  Generative quality for one output modality over a range of $\beta$ values. Points
  denote the FID averaged over three seeds and bands show one standard
  deviation respectively. Due to numerical instabilities, the MVAE could
  not be trained with larger $\beta$ values.
}
\label{fig:fids}
\end{center}
\end{figure}

\begin{figure}[t]
\begin{center}
\begin{subfigure}[t]{.35\linewidth}
  \centering
  \includegraphics[width=1.0\linewidth]{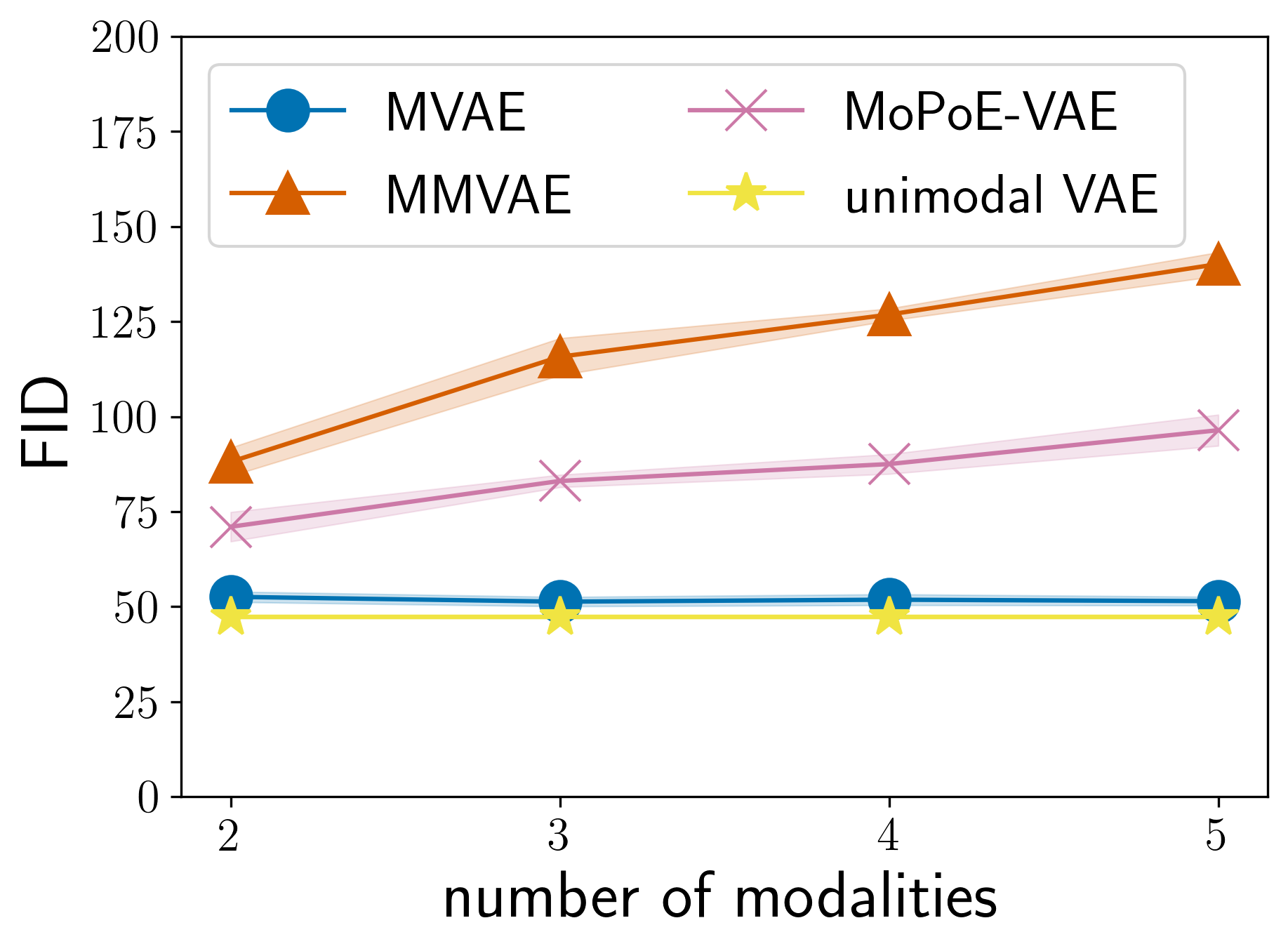}
  \caption{PolyMNIST}
\end{subfigure}
\hskip +0.135in
\begin{subfigure}[t]{.35\linewidth}
  \centering
  \includegraphics[width=1.0\linewidth]{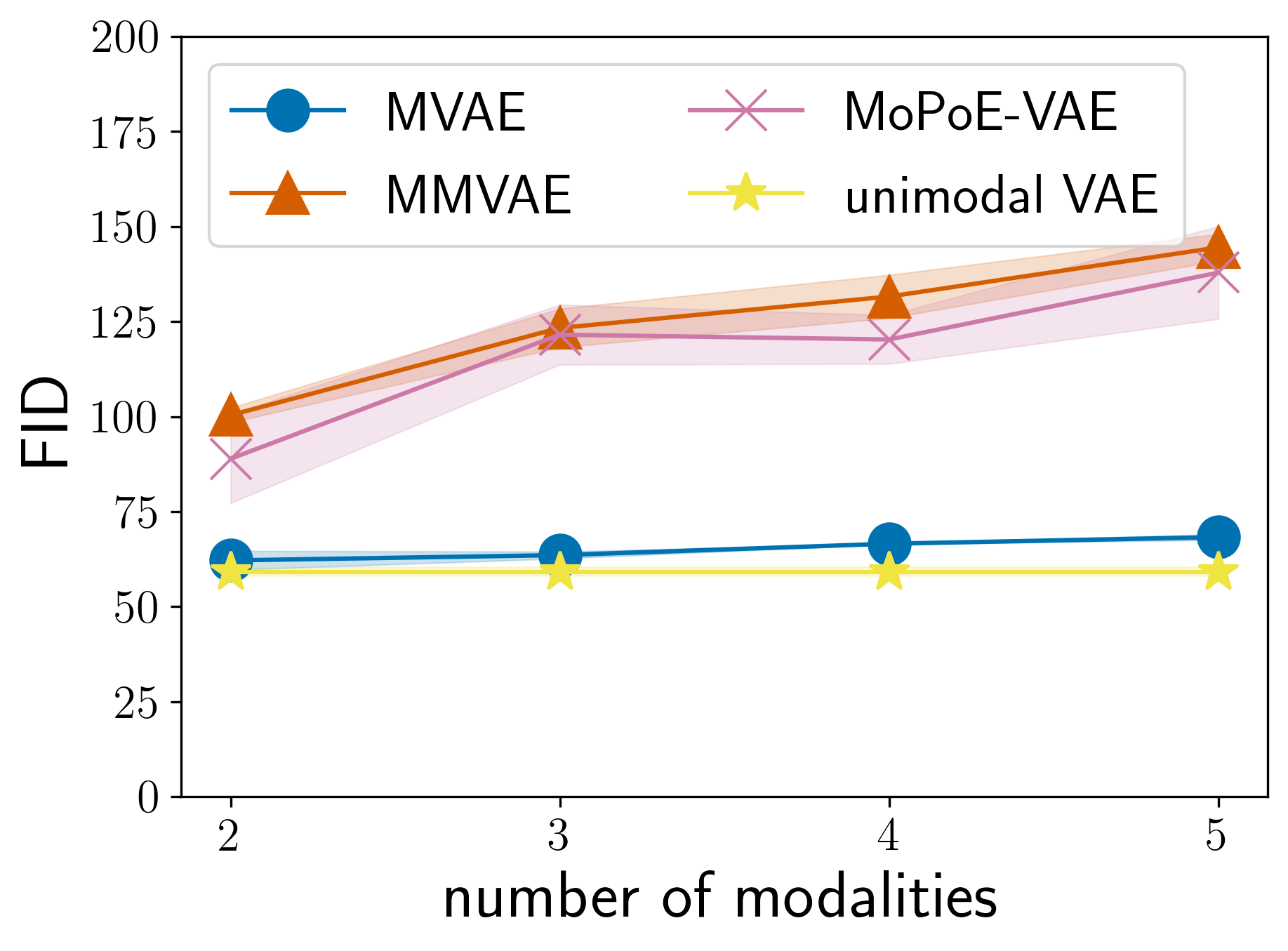}
  \caption{Translated-PolyMNIST}
\end{subfigure}%
\caption{%
  Generative quality as a function of the number of modalities. The results show
  the FID of the same modality and therefore all values on the same scale. All
  models are trained with $\beta=1$ on PolyMNIST and $\beta=0.3$ on
  Translated-PolyMNIST\@. The results are averaged over three seeds and the
  bands show one standard deviation respectively.  For the unimodal VAE, which
  uses only a single modality, the average and standard deviation are plotted
  as a constant.
}
\label{fig:num_mod_ablation}
\end{center}
\end{figure}

\subsection{Lack of generative coherence on more complex data}
\label{subsec:lack_of_coherence}

Apart from generative quality, another desired criterion
\citep{Shi2019,Sutter2020} for an effective multimodal generative model is
\textit{generative coherence}, which measures a model's ability to generate
semantically related samples across modalities. 
To be consistent with \citet{Sutter2021}, we compute the leave-one-out
coherence (see \Cref{app:implementation_details}), which means that the input
to each model consists of all modalities except the one that is being
conditionally generated. On CUB, we resort to a qualitative evaluation of
coherence, because there is no ground truth annotation of shared factors and
the proxies used in \citet{Shi2019} and \citet{Shi2021} do not yield meaningful
estimates when applied to the conditionally generated images from models that
were trained on \textit{real} images.%
\footnote{%
  Please note that previous work \citep{Shi2019,Shi2021} used a simplified
  version of the CUB dataset, where images were replaced by precomputed
  ResNet-features.
}

In terms of generative coherence, \Cref{fig:coherences} reveals that the
positive results from previous work do not translate to more complex datasets.
As a baseline, for PolyMNIST (\Cref{fig:coherences:poly}) we replicate the
coherence results from \citet{Sutter2021} for a range of $\beta$ values.
Consistent with previous work \citep{Shi2019, Shi2021, Sutter2020, Sutter2021},
we find that the MMVAE and MoPoE-VAE exhibit superior coherence compared to the
MVAE\@.  Though, it was not apparent from previous work that MVAE's coherence
can improve significantly with increasing $\beta$ values, which can be of
independent interest for future work. On Translated-PolyMNIST
(\Cref{fig:coherences:tpoly}), the stark decline of all models makes it evident
that coherence cannot be guaranteed when shared information cannot be predicted
in expectation across modalities. Our qualitative results
(\Cref{fig:qualitative_conditional} in
\Cref{app:additional_experimental_results}) confirm that not a single
multimodal VAE is able to conditionally generate coherent examples and, for the
most part, not any digits at all.  To verify that the lack of coherence is not
an artifact of our implementation, we have checked that the encoders and
decoders have sufficient capacity such that digits show up in most
self-reconstructions.  On CUB (\Cref{fig:coherences:cub}), for which coherence
cannot be computed, the qualitative results for conditional generation verify
that none of the existing approaches generates images that are both of
sufficiently high quality and coherent with respect to the given caption.
Overall, the negative results on Translated-PolyMNIST and CUB showcase the
limitations of existing approaches when applied to more complex datasets than
those used in previous benchmarks.

\begin{figure}[t]
\begin{center}
\begin{subfigure}[b]{.37\linewidth}
  \centering
  \includegraphics[width=1.0\linewidth]{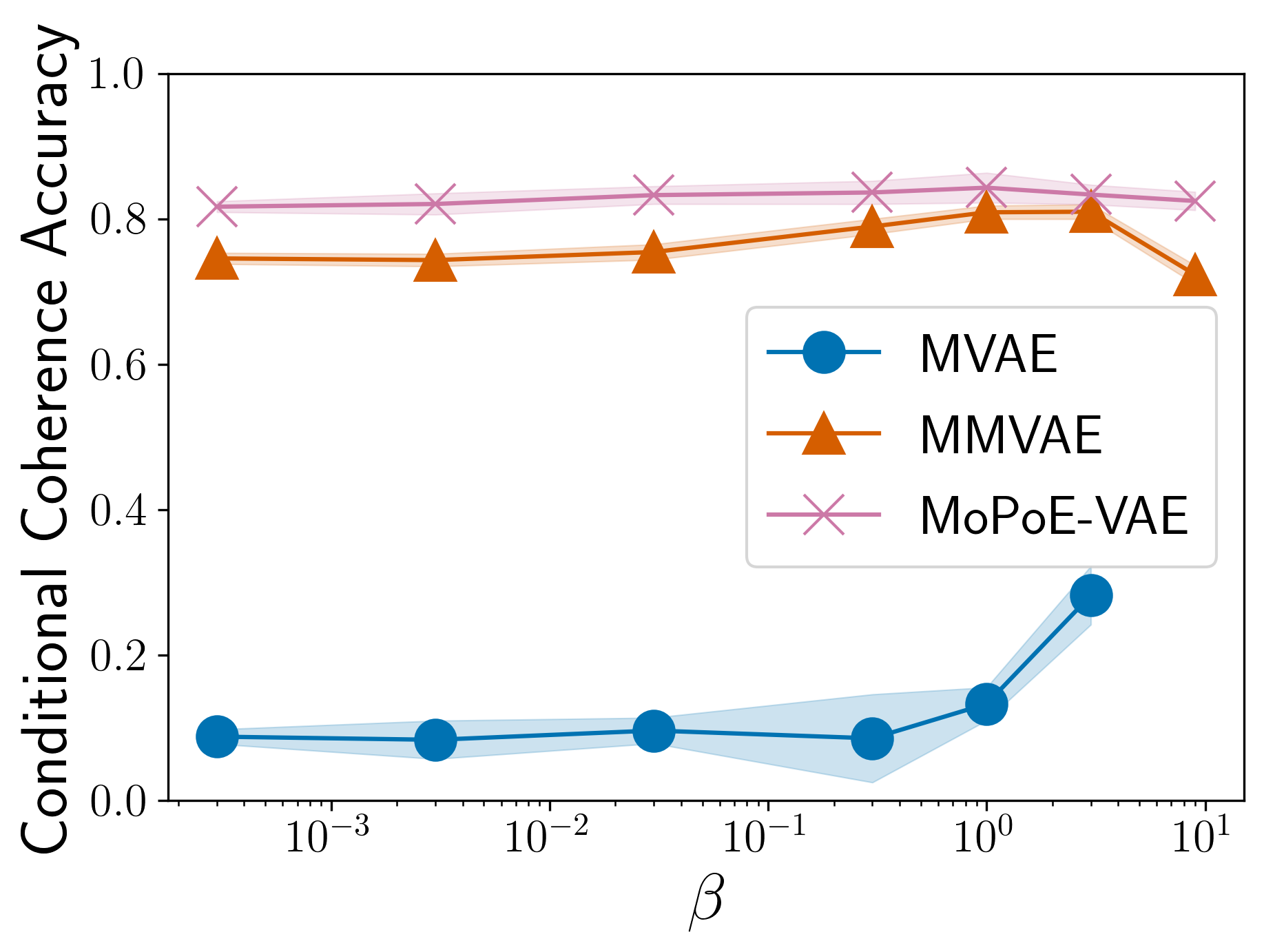}
  \vskip -0.15cm
  \caption{PolyMNIST}
\label{fig:coherences:poly}
  \centering
  \includegraphics[width=1.0\linewidth]{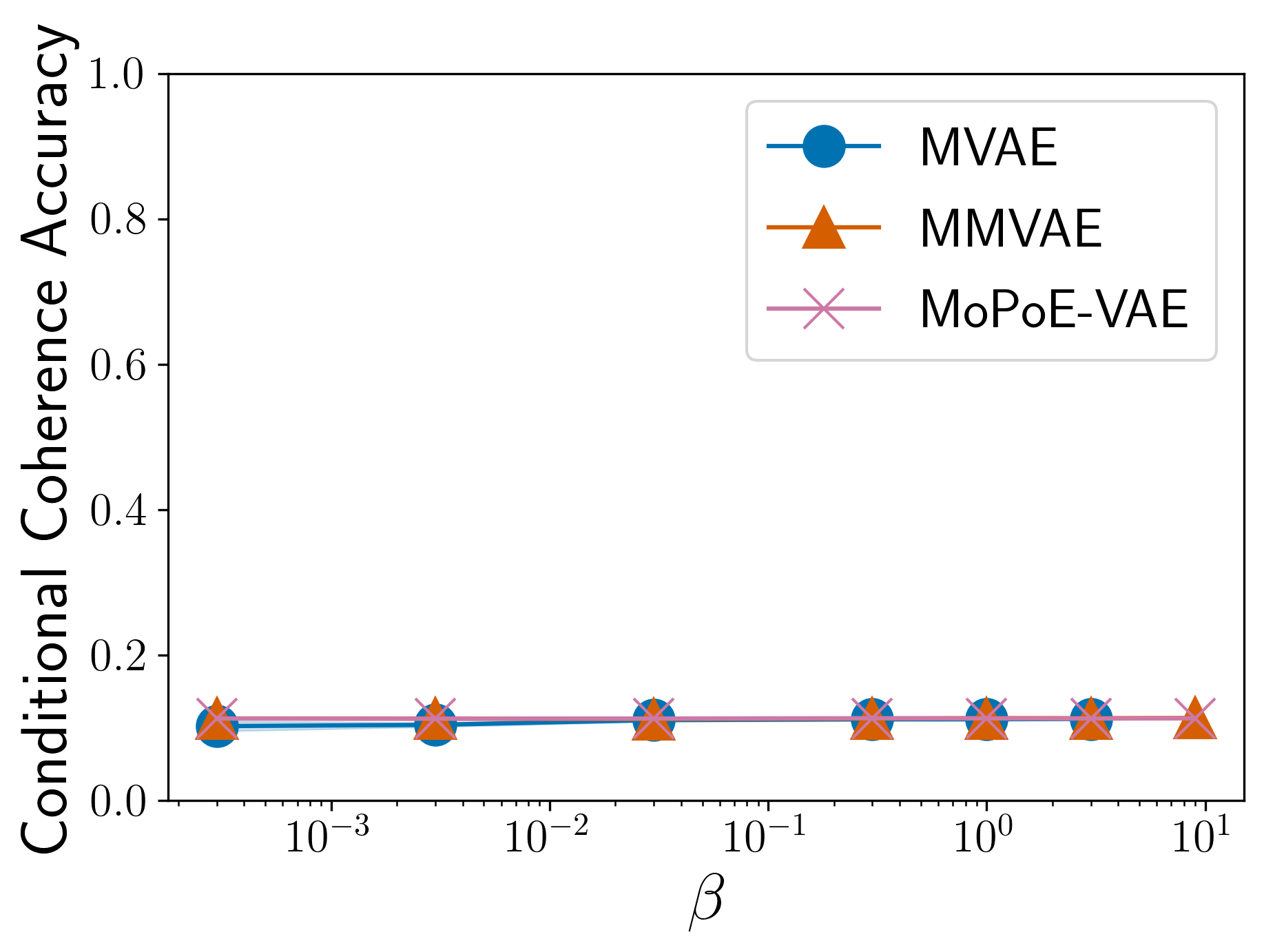}
  \vskip -0.15cm
  \caption{Translated-PolyMNIST}
\label{fig:coherences:tpoly}
\end{subfigure}
\hskip +0.38in
\begin{subfigure}[b]{.37\linewidth}
  \centering
  \includegraphics[width=1.0\linewidth]{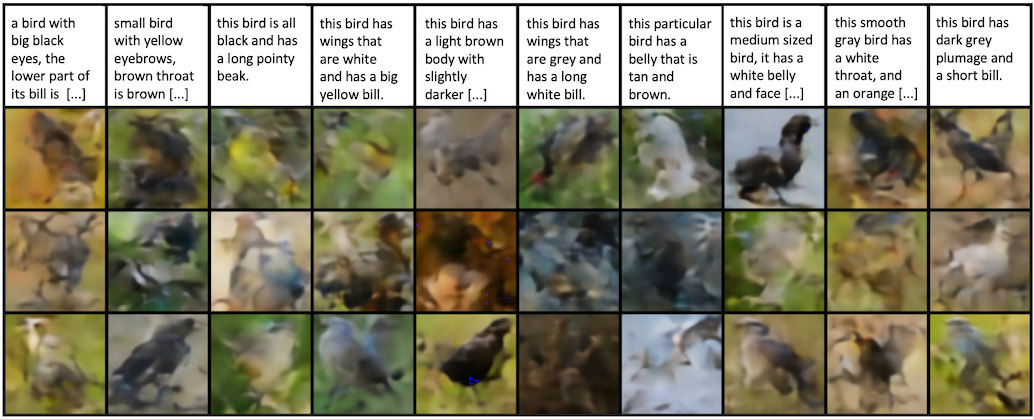}
  \vskip -0.15cm
  \caption*{MVAE, $\beta=9$}
  \vskip 0.10cm
  \includegraphics[width=1.0\linewidth]{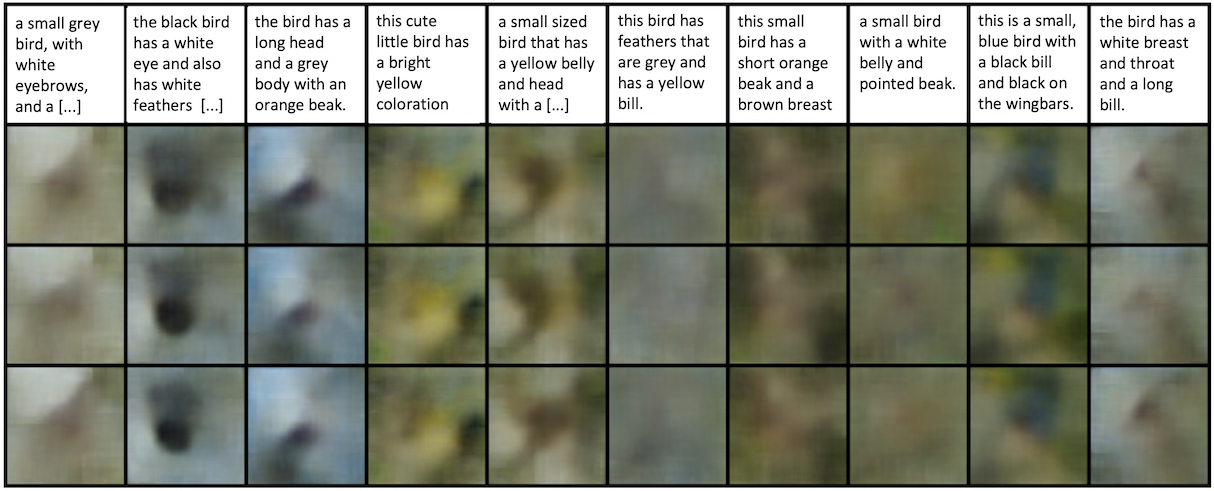}
  \vskip -0.15cm
  \caption*{MMVAE, $\beta=9$}
  \vskip 0.10cm
  \includegraphics[width=1.0\linewidth]{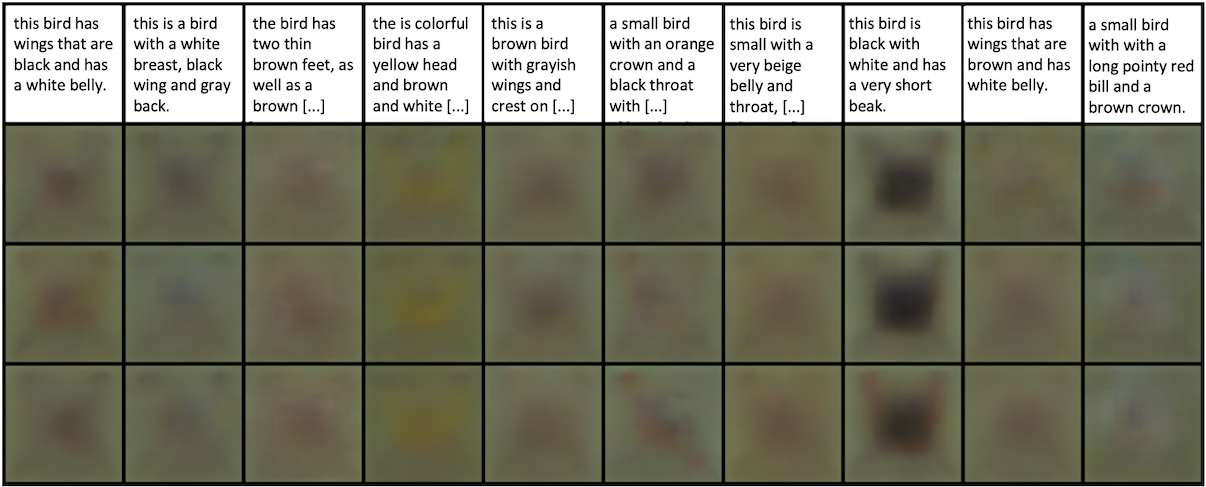}
  \vskip -0.15cm
  \caption*{MoPoE-VAE, $\beta=9$}
  \caption{Caltech Birds (CUB)}
\label{fig:coherences:cub}
\end{subfigure}
\caption{%
  Generative coherence for the conditional generation across modalities. For
  PolyMNIST (\Cref{fig:coherences:poly,fig:coherences:tpoly}), we plot the
  average leave-one-out coherence. Due to numerical instabilities, the MVAE
  could not be trained with larger $\beta$ values.  For CUB
  (\Cref{fig:coherences:cub}), we show qualitative results for the conditional
  generation of images given captions.  Best viewed zoomed and in color.
}
\label{fig:coherences}
\end{center}
\end{figure}

\section{Discussion}
\label{sec:discussion}

\paragraph{Implications and scope}
Our experiments lend empirical support to the proposed theoretical limitations
of mixture-based multimodal VAEs. On both synthetic and real data, our results
showcase the generative limitations of multimodal VAEs that sub-sample
modalities. However, our results also reveal that none of the existing
approaches (including those without sub-sampling) fulfill all desired criteria
\citep{Shi2019,Sutter2020} of an effective multimodal generative model.  More
broadly, our results showcase the limitations of existing VAE-based approaches
for modeling weakly-supervised data in the presence of modality-specific
information, and in particular when shared information cannot be predicted in
expectation across modalities. The {$\text{Translated-PolyMNIST}$} dataset
demonstrates this problem in a simple setting, while the results on CUB confirm
that similar issues can be expected on more realistic datasets. 
For future work, it would be interesting to generate simulated data where the
discrepancy $\Delta(\X, \mathcal{S})$ can be measured exactly and where it is
gradually increased by an adaptation of the dataset in a way that increases
only the modality-specific variation. Furthermore, it is worth noting that
\Cref{thm:irreducible_error} applies to all multimodal VAEs that optimize
\Cref{eq:objective_s}, which is a lower bound on the multimodal ELBO for models
that sub-sample modalities. Our theory predicts the same discrepancy for models
that optimize a tighter bound (e.g., via
\Cref{eq:mopoe_theoretical_objective}), because the discrepancy $\Delta(\X,
\mathcal{S})$ derives from the likelihood term, which is equal for
\Cref{eq:objective_s,eq:mopoe_theoretical_objective}. In
\Cref{app:additional_experimental_results} we verify that the discrepancy can
also be observed for the MMVAE with the original implementation from
\cite{Shi2019} that uses a tighter bound. Further analysis of the different
bounds can be an interesting direction for future work.  

\paragraph{Model selection and generalization} 
Our results raise fundamental questions regarding model selection and
generalization, as generative quality and generative coherence do not
necessarily go hand in hand.  In particular, our experiments demonstrate that
FIDs and log-likelihoods do not reflect the problem of lacking coherence and
without access to ground truth labels (on what is shared between modalities)
coherence metrics cannot be computed.  As a consequence, it can be difficult to
perform model selection on more realistic multimodal datasets, especially for
less interpretable types of modalities, such as DNA sequences.  Hence, for
future work it would be interesting to design alternative metrics for
generative coherence that can be applied when shared information is not
annotated.  
For the related topic of generalization, it can be illuminating to consider
what would happen, if one could arbitrarily ``scale things up''. In the limit of
infinite i.i.d.\ data, perfect generative coherence could be achieved by a
model that memorizes the pairwise relations between training examples from
different modalities. However, would this yield a model that generalizes out of
distribution (e.g., under distribution shift)? We believe that for future work
it would be worthwhile to consider out-of-distribution generalization
performance \citep[e.g.,][]{Montero2021} in addition to generative quality and
coherence.

\paragraph{Limitations}
In general, the limitations and tradeoffs presented in this work apply to a
large family of multimodal VAEs, but not necessarily to other types of
generative models, such as generative adversarial networks
\citep{Goodfellow2014}. Where current VAEs are limited by the reconstruction of
modality-specific information, other types of generative models might offer
less restrictive objectives.  Similar to previous work, we have only considered
models with simple priors, such as Gauss and Laplace distributions with
independent dimensions. Further, we have not considered models with
modality-specific latent spaces, which seem to yield better empirical results
\citep{Hsu2018,Sutter2020,Daunhawer2020}, but currently lack theoretical
grounding. Modality-specific latent spaces offer a potential solution to the
problem of cross-modal prediction by providing modality-specific context from
the target modalities to each decoder. However, more work is required to
establish \textit{guarantees} for the identifiability and disentanglement of
shared and modality-specific factors, which might only be possible for VAEs
under relatively strong assumptions
\citep{Locatello2019,Locatello2020,Gresele2019,Kuegelgen2021}.

\section{Conclusion}

In this work, we have identified, formalized, and demonstrated several
limitations of multimodal VAEs. Across different datasets, this work revealed a
significant gap in generative quality between unimodal and mixture-based
multimodal VAEs. We showed that this apparent paradox can be explained by the
sub-sampling of modalities, which enforces an undesirable upper bound on the
multimodal ELBO and therefore limits the generative quality of the respective
models. While the sub-sampling of modalities allows these models to learn the
inference networks for different subsets of modalities efficiently, there is a
notable tradeoff in terms of generative quality.  Finally, we studied two
failure cases---Translated-PolyMNIST and CUB---that demonstrate the limitations
of multimodal VAEs when applied to more complex datasets than those used in
previous benchmarks.

For future work, we believe that it is crucial to be aware of the limitations
of existing methods as a first step towards developing new methods that achieve
more than incremental improvements for multimodal learning. We conjecture that
there are at least two potential strategies to circumvent the theoretical
limitations of multimodal VAEs. First, the sub-sampling of modalities can be
combined with modality-specific context from the target modalities. Second,
cross-modal reconstruction terms can be replaced with less restrictive
objectives that do not require an exact prediction of modality-specific
information.  Finally, we urge future research to design more challenging
benchmarks and to compare multimodal generative models in terms of both
generative quality and coherence across a range of hyperparameter values, to
present the tradeoff between these metrics more transparently.

\newpage

\section*{Acknowledgements}
ID and KC were supported by the SNSF grant \texttt{\#200021\_188466}. Special
thanks to Alexander Marx, Nicolò Ruggeri, Maxim Samarin, Yuge Shi, and Mario
Wieser for helpful discussions and/or feedback on the manuscript.

\section*{Reproducibility Statement}
For all theoretical statements, we provide detailed derivations and state the
necessary assumptions.  For our main theoretical results, we present empirical
support on both synthetic and real data. To ensure empirical reproducibility,
the results of each experiment and every ablation were averaged over multiple
seeds and are reported with standard deviations.  All of the used datasets are
either public or can be generated from publicly available resources using the
code that we provide in the supplementary material. Information about
implementation details, hyperparameter settings, and evaluation metrics are
included in \Cref{app:experiments}.

\bibliographystyle{apalike}
\bibliography{references}

\newpage
\appendix

\section{Definitions}
\label{app:infotheory}

Let $\mathcal{X}$, $\mathcal{Y}$, and $\mathcal{Z}$ denote the support sets of
three discrete random vectors $\X$, $\Y$, and $\Z$ respectively. Let
$p_{\X}(\x)$, $p_{\Y}(\y)$, and $p_{\Z}(\z)$ denote the respective marginal
distributions and note that we will leave out the subscripts (e.g., $p(\x)$ instead of $p_{\mathcal{X}}(\x)$)
when it is clear from context which distribution we are referring to.
Analogously, we write shorthand $p(\y \given \x)$ for the
conditional distribution of $\Y$ given $\X$ and $p(\x, \y)$ for the joint
distribution of $\X$ and $\Y$.

The entropy of $\X$ is defined as
\begin{equation}
  H(\X) = - \sum_{\x \in \mathcal{X}} p(\x) \log p(\x)\;.
\end{equation}

The conditional entropy of $\X$ given $\Y$ is defined as
\begin{equation}
  H(\X \given \Y) = - \sum_{\x \in \mathcal{X}, \y \in \mathcal{Y}} p(\x, \y) \log p(\x \given \y)\;.
\end{equation}

The joint entropy of $\X$ and $\Y$ is defined as
\begin{equation}
  H(\X, \Y) = - \sum_{\x \in \mathcal{X}, \y \in \mathcal{Y}} p(\x, \y) \log p(\x, \y)\;.
\end{equation}

The Kullback-Leibler divergence of the discrete probability distribution $P$ from the
discrete probability distribution $Q$ is defined as
\begin{equation}
  D_{\text{KL}}(P \divfrom Q) = \sum_{\x \in \mathcal{X}} P(\x) \log \frac{P(\x)}{Q(\x)}
\end{equation}
assuming that $P$ and $Q$ are defined on the same support set $\mathcal{X}$.

The cross-entropy of the discrete probability distribution $Q$ from the
discrete probability distribution $P$ is defined as
\begin{equation}
  CE(P, Q) = - \sum_{\x \in \mathcal{X}} P(\x) \log Q(\x)
\end{equation}
assuming that $P$ and $Q$ are defined on the same support set $\mathcal{X}$.

The mutual information of $\X$ and $\Y$ is defined as
\begin{equation}
  I(\X; \Y) = D_{\text{KL}}(p(\x, \y) \divfrom p(\x) p(\y))\;.
\end{equation}

The conditional mutual information of $\X$ and $\Y$ given $\Z$ is defined as
\begin{equation}
  I(\X; \Y \given \Z) = 
  \sum_{\z \in \mathcal{Z}} p(\z) D_{\text{KL}}(p(\x, \y \given \z) \divfrom p(\x \given \z) p(\y \given \z))\;.
\end{equation}

Recall that we assume discrete random vectors (e.g., pixel values) and
therefore can assume non-negative entropy, conditional entropy and conditional
mutual information terms \citep{Cover2012}. For continuous random variables, all
of the above sums can be replaced with integrals. The only
information-theoretic quantities for which in this work we use continuous
random vectors are the KL-divergence and mutual information, both of which are
always non-negative.

\newpage
\section{Proofs}
\label{app:proofs}

\subsection{Information-theoretic derivation of the multimodal ELBO}
\label{app:elbo_derivation}

\Cref{prop:joint_entropy} relates the multimodal ELBO
(\Cref{def:multimodal_elbo}) to the expected log-evidence, the quantity that is
being approximated by all likelihood-based generative models including VAEs.
The derivation is based on a straightforward extension of the variational
information bottleneck \citep[VIB;][]{Alemi2017}. We include the result mainly
for the purpose of illustration---to clarify the notation, as well as the
relation between the multimodal ELBO and the underlying information-theoretic
quantities of interest: the entropy, conditional entropy, and mutual
information.

\paragraph{Notation} 
Readers who are familiar with latent variable models, but may be less
familiar with the information-theoretic perspective on VAEs, please keep in
mind the following notational differences. In contrast to the latent variable
model perspective, which defines a variational posterior (typically denoted by
the letter $q$) and a stochastic decoder (typically denoted by the letter $p$),
the VIB defines a stochastic encoder $p_{\theta}(\z \given \x)$ and variational
decoder $q_{\phi}(\x \given \z) $. Moreover, the VIB makes no assumptions about
the true posterior. Also note that latent variable models tend to write the
ELBO with respect to the log-evidence $\log p(\x)$, but information-theoretic
approaches write the ELBO with respect to the \textit{expected} log-evidence
$\E_{p(\x)}[ \log p(\x)]$; though, it is still assumed that the estimation of
the ELBO is based on a finite sample from $p(\x)$.

\begin{restatable}{proposition}{propjointentropy}
\label{prop:joint_entropy}
The multimodal ELBO forms a variational lower bound on the expected log-evidence:
\begin{equation}
\E_{p(\x)}[\log p(\x)] \geq \mathcal{L}(\x; \theta, \phi)\;.
\end{equation}
\end{restatable}

\begin{proof}

First, notice that the expected log-evidence is equal to the negative
entropy ${- H(\X) = \E_{p(\x)} [ \log p_{}(\x) ]}$. Given any random variable
$Z$, the entropy can be decomposed into conditional entropy and mutual
information terms: $H(\X) = H(\X \given Z) + I(\X; Z)$. 

The expected log-evidence relates to the multimodal ELBO as follows:
\begin{align}
  \E_{p(\x)} [ \log p(\x)]
  &= -H(\X \given Z) - I(\X; Z)  \label{eq:tmp}\\
  &\geq \E_{p(\x)p_{\theta}(\z \given \x)} [ \log q_{\phi}(\x \given \z) ] - \E_{p(\x)} [D_{\text{KL}}(p_{\theta}(\z \given \x) \divfrom q(\z))] \\
  &= \mathcal{L}_{}(\x; \theta, \phi)
\end{align}
where the inequality follows from the variational approximations of the
respective terms. As in \citet{Alemi2017}, we can use the following variational
bounds.

For the conditional entropy, we have
\begin{align}
  - H(\X \given Z) 
  &= \E_{p(\x)p_{\theta}(\z \given \x)} \left[ \log p_{}(\x \given \z) \right] \\
  &= \E_{p(\x)p_{\theta}(\z \given \x)} \left[ \log q_{\phi}(\x \given \z) \right] +
     \E_{p(\z)} \left[ D_{\text{KL}}(p_{}(\x \given \z) \divfrom q_{\phi}(\x \given \z)) \right] \\
  &\geq \E_{p(\x)p_{\theta}(\z \given \x)} \left[ \log q_{\phi}(\x \given \z) \right]
\end{align}
where $q_{\phi}(\x \given \z)$ is a variational decoder that is parameterized by $\phi$.

For the mutual information, we have
\begin{align}
  - I(\X; Z) 
  &= - \E_{p(\x)} \left[ D_{\text{KL}}(p_{\theta}(\z \given \x) \divfrom p(\z)) \right] \\
  &= - \E_{p(\x)} \left[ D_{\text{KL}}(p_{\theta}(\z \given \x) \divfrom q(\z)) \right] + 
       D_{\text{KL}}(p(\z) \divfrom q(\z)) \\
  &\geq - \E_{p(\x)} \left[ D_{\text{KL}}(p_{\theta}(\z \given \x) \divfrom q(\z)) \right]
\end{align}
where $q(\z)$ is a prior.

Hence, the multimodal ELBO forms a variational lower bound on the expected log-evidence:
\begin{align}
  \E_{p(\x)}[\log p(\x)] &= \mathcal{L}_{\text{}}(\x; \theta, \phi) + \Delta_{\text{VA}}(\x, \phi) \\
        &\geq \mathcal{L}_{}(\x; \theta, \phi)
\end{align}
where
\begin{equation}
  \Delta_{\text{VA}}(\x, \phi) =
  \E_{p(\z)} \left[ D_{\text{KL}}(p_{}(\x \given \z) \divfrom q_{\phi}(\x \given \z)) \right] +
    D_{\text{KL}}(p(\z) \divfrom q(\z))
\end{equation}

denotes the (non-negative) variational approximation gap.

\end{proof}

\newpage
\subsection{Relation between the different objectives}%
\label{app:elbo_s_derivation}

\Cref{prop:elbo_s_derivation} relates the multimodal ELBO $\mathcal{L}$ from \Cref{def:multimodal_elbo} to the objective
$\mathcal{L}_{\mathcal{S}}$, which is a general formulation
of the objective maximized by all mixture-based multimodal VAEs. 
Compared to previous mixture-based formulations \citep{Shi2019,Sutter2020}, our
formulation is more general in that it allows for arbitrary subsets with
non-uniform mixture coefficients. Further, the derivation \textit{quantifies}
the approximation gap between $\mathcal{L}$ and
$\mathcal{L}_{\mathcal{S}}$, where the latter corresponds to
the objectives that are actually being optimized in the implementations of the
MMVAE, MoPoE-VAE, and MVAE without sub-sampling.

\begin{proposition}
\label{prop:elbo_s_derivation}
  For every stochastic encoder $p^{\mathcal{S}}_{\theta}(\z \given \x)$ that is
  consistent with \Cref{def:mixture}, the following inequality holds:
  \begin{equation}
  \mathcal{L}(\x; \theta, \phi) \geq \mathcal{L}_{\mathcal{S}}(\x; \theta, \phi) \;.
  \end{equation}
\end{proposition}

\begin{proof}
Recall the multimodal ELBO from \Cref{def:multimodal_elbo}:
\begin{align}
\mathcal{L}_{}(\x; \theta, \phi) = \E_{p(\x)p_{\theta}(\z \given
\x)} [ \log q_{\phi}(\x \given \z) ] - \E_{p(\x)}
[D_{\text{KL}}(p_{\theta}(\z \given \x) \divfrom q(\z))]\;.
\end{align}
For the encoder $p_{\theta}(\z \given \x)$, plug in the mixture-based encoder 
${p^{\mathcal{S}}_{\theta}(\z \given \x) = \sum_{A \in \mathcal{S}} \omega_A \, p_{\theta}(\z \given \x_A)}$
from \Cref{def:mixture} and re-write as follows:
\begin{align}
  \label{eq:mopoe_theoretical_objective}
  &\E_{p(\x)p^{\mathcal{S}}_\theta(\z \given \x)} [ \log q_{\phi}(\x \given \z) ] 
  - \E_{p(\x)} [D_{\text{KL}}(p^{\mathcal{S}}_\theta(\z \given \x) \divfrom q(\z))] \\[5pt]
  &= \vphantom{\bigg(} \E_{p(\x)\sum_{A \in \mathcal{S}} \omega_A \, p_{\theta}(\z \given \x_A)} [ \log q_{\phi}(\x \given \z) ]\, -  \\[-7pt]\notag
  &\hskip1.4em\relax \E_{p(\x)\sum_{A \in \mathcal{S}} \omega_A \, p_{\theta}(\z \given \x_A)} [ \log p^{\mathcal{S}}_\theta(\z \given \x) - \log q(\z)] \vphantom{\bigg)} \\[5pt]
  &= \sum_{A \in \mathcal{S}} \omega_A \big\{ \E_{p(\x)p_{\theta}(\z \given \x_A)} [ \log q_{\phi}(\x \given \z) ] 
  - \E_{p(\x)p_{\theta}(\z \given \x_A)} [\log p^{\mathcal{S}}_\theta(\z \given \x)] \,+ \\[-7pt]\notag
  &\hskip5.2em\relax \E_{p(\x)p_{\theta}(\z \given \x_A)} [\log q(\z)] \big\}\\[5pt]\label{eq:prop:elbo_s_derivation:tmp0}
  &= \sum_{A \in \mathcal{S}} \omega_A \big\{ \E_{p(\x)p_{\theta}(\z \given \x_A)} [ \log q_{\phi}(\x \given \z) ] 
  + \E_{p(\x)} [ CE(p_{\theta}(\z \given \x_A), p^{\mathcal{S}}_\theta(\z \given \x)) ] \,-  \\[-7pt]\notag
&\hskip5.2em\relax \E_{p(\x)} [ CE(p_{\theta}(\z \given \x_A), q(\z)) ] \big\} \\[5pt]\label{eq:prop:elbo_s_derivation:tmp1}
  &= \sum_{A \in \mathcal{S}} \omega_A \big\{ \E_{p(\x)p_{\theta}(\z \given \x_A)} [ \log q_{\phi}(\x \given \z) ]
  + \E_{p(\x)} [ D_{\text{KL}}(p_{\theta}(\z \given \x_A) \divfrom p^{\mathcal{S}}_\theta(\z \given \x)) ] \,- \\[-7pt]\notag
  &\hskip5.2em\relax \E_{p(\x)} [ D_{\text{KL}}(p_{\theta}(\z \given \x_A) \divfrom q(\z)) ]
    \big\} \\[5pt]\label{eq:prop:elbo_s_derivation:tmp2}
  &\geq \sum_{A \in \mathcal{S}} \omega_A \big\{ \E_{p(\x)p_{\theta}(\z \given \x_A)} [ \log q_{\phi}(\x \given \z) ] 
    - \E_{p(\x)} [ D_{\text{KL}}(p_{\theta}(\z \given \x_A) \divfrom q(\z)) 
      ]
    \big\} \\[5pt] \label{eq:prop:elbo_s_derivation:tmp3}
  &= \mathcal{L}_{\mathcal{S}}(\x; \theta, \phi)
\end{align}
In \Cref{eq:prop:elbo_s_derivation:tmp0}, $CE(p, q)$ denotes the cross-entropy
between distributions $p$ and $q$. For \Cref{eq:prop:elbo_s_derivation:tmp1},
decompose both cross-entropy terms using $CE(p, q) = H(p) + D_{\text{KL}}(p
\divfrom q)$ and notice that the respective entropy terms cancel out. The
inequality (\Cref{eq:prop:elbo_s_derivation:tmp2}) follows from the
non-negativity of the KL-divergence. This concludes the proof that
$\mathcal{L}_{\mathcal{S}}(\x; \theta, \phi)$ forms a lower bound on
$\mathcal{L}(\x; \theta, \phi)$.

\end{proof}

\paragraph{Objectives of individual models} 
\citet{Sutter2021} already showed that \Cref{eq:mopoe_theoretical_objective}
subsumes the objectives of the MMVAE, MoPoE-VAE, and MVAE without ELBO
sub-sampling. However, in their actual implementation, all of these methods
take the sum out of the KL-divergence term \citep[e.g., see][Equation
3]{Shi2019}, which corresponds to the objective $\mathcal{L}_{\mathcal{S}}$. To
see how $\mathcal{L}_{\mathcal{S}}$ recovers the objectives of the individual
models, simply plug in the model-specific definition of $\mathcal{S}$ into
\Cref{eq:prop:elbo_s_derivation:tmp2} and use uniform mixture coefficients
$\,\omega_A = 1 / \vert \mathcal{S} \vert$ for all subsets. For the MVAE
without ELBO sub-sampling, $\,\mathcal{S}$ is comprised of only one subset, the
complete set of modalities $\{\x_1, \ldots, \x_M\}$.  For the MMVAE,
$\,\mathcal{S}$ is comprised of the set of unimodal subsets $\{\{\x_1\},
\ldots, \{\x_M\}\}$.  For the MoPoE-VAE, $\,\mathcal{S}$ is comprised of the
powerset $\mathcal{P}(M) \setminus \{\emptyset\}$.  Further implementation
details, such as importance sampling and ELBO sub-sampling, are discussed in
\Cref{app:additional_experimental_results}.

\newpage
\subsection{Objective \texorpdfstring{$\mathcal{L}_{\mathcal{S}}$}\ \ is a special case of the VIB}
\label{app:ib_objective}

\begin{lemma}
\label{lemma:vib}
$\mathcal{L}_{\mathcal{S}}(\x; \theta, \phi)$
is a special case of the variational information bottleneck (VIB) objective
\begin{equation}
  \label{eq:vib_mixture}
\min_{\psi} \sum_{A \in \mathcal{S}} \omega_A \left\{ H_{\psi}(\X \given Z_A) + I_{\psi}(\X_A; Z_A) \right\}\,,
\end{equation}
where the encoding $Z_A = f_{\psi}(\X_A)$ is a function of a subset $\X_A$, the
terms $H_\psi(\X \given Z_A)$ and $I_\psi(\X_A; Z_A)$ denote variational upper
bounds of $H(\X \given Z_A)$ and $I(\X_A; Z_A)$ respectively, and $\psi$
summarizes the parameters of these variational estimators. 
\end{lemma}

\begin{proof}
We start from $\mathcal{L}_{\mathcal{S}}$, the objective optimized by all
mixture-based multimodal VAEs. Recall from \Cref{def:mixture_based}:
\begin{align}
\mathcal{L}_{\mathcal{S}}(\x; \theta, \phi) 
  &= \sum_{A \in \mathcal{S}} \omega_A \Big\{ 
  \underbrace{\E_{p(\x)p_{\theta}(\z \given \x_A)} [ \log q_{\phi}(\x \given \z) ]}_{(i)}
  - \underbrace{\E_{p(\x)} \left[ D_{\text{KL}}\left(p_{\theta}(\z \given \x_A) \divfrom q(\z)\right) \right]}_{(ii)}
  \Big\} \;.
\end{align}
Each term within the sum is comprised of two terms: $(i)$ the log-likelihood
estimation based on a variational decoder $q_{\phi}(\x \given \z)$; $(ii)$ the
regularization of the stochastic encoder $p_{\theta}(\z \given \x_A)$ with
respect to a variational prior $q(\z)$. The sampled encoding $\z \sim
p_{\theta}(\z \given \x_A)$ can be viewed as the output of a function $Z_A =
f_{\theta}(\X_A)$ of a subset of modalities.  

To see the relation to the underlying information terms $H(\X \given Z_A)$ and
$I(\X_A; Z_A)$, we undo the variational approximation for $(i)$ and
$(ii)$ by re-introducing the unobserved ground truth decoder $p(\x \given \z)$
and the ground truth prior $p(\z)$.%

For $(i)$, we have
\begin{align}
  \E_{p(\x)p_{\theta}(\z \given \x_A)} \left[ \log q_{\phi}(\x \given \z) \right]
  &\leq \E_{p(\x)p_{\theta}(\z \given \x_A)} \left[ \log q_{\phi}(\x \given \z) \right]\, + \\\notag
  &\hskip1.19em\relax \E_{p(\z)} \left[ D_{\text{KL}}(p_{}(\x \given \z) \divfrom q_{\phi}(\x \given \z)) \right] \\
  &= \E_{p(\x)p_{\theta}(\z \given \x_A)} \left[ \log p(\x \given \z) \right] \\
  &= -H(\X \given Z_A) 
\end{align}
For $(ii)$, we have
\begin{align}
  \E_{p(\x)} \left[ D_{\text{KL}}(p_{\theta}(\z \given \x_A) \divfrom q(\z)) \right]
  &\geq \E_{p(\x)} \left[ D_{\text{KL}}(p_{\theta}(\z \given \x_A) \divfrom q(\z)) \right]
  - D_{\text{KL}}(p(\z) \divfrom q(\z)) \\
  &= \E_{p(\x)} \left[ D_{\text{KL}}(p_{\theta}(\z \given \x_A) \divfrom p(\z)) \right] \\
  &= I(\X_A; Z_A) 
\end{align}

Since $\mathcal{L}_{\mathcal{S}}(\x; \theta, \phi)$ is being
maximized, $(i)$ is being maximized, while $(ii)$ is being minimized. The
maximization of $(i)$ is equal to the minimization of a variational upper bound
on $H(\X \given Z_A)$.  Similarly, the minimization of $(ii)$ is equal to the
minimization of a variational upper bound on $I(\X_A; Z_A)$. Hence, we have
established that $\mathcal{L}_{\mathcal{S}}(\x; \theta, \phi)$ is
a special case of the more general VIB objective (\Cref{eq:vib_mixture}) where
the information terms are estimated with a mixture-based multimodal
VAE that is parameterized by $\psi = \{\theta, \phi\}$.

\end{proof}

\subsection{Decomposition of the conditional entropy for subsets of modalities}%
\label{app:explanation_decomposition}

\begin{lemma}%
\label{lemma:decomposition}
Let $\X_A \subseteq \X$ be some subset of modalites. If $Z_A = f(\X_A)$, where
$f$ is some function of the subset $\X_A$, then the following equality holds:
\begin{equation}
  H(\X \given Z_A)
  = H(\X_{\{1, \ldots, M\} \setminus A} \given \X_A) + H(\X_A \given Z_A)\;.
\end{equation}
\end{lemma}

\begin{proof}
When $Z_A$ is a function of a subset $\X_A \subseteq \X$, we have the Markov chain
${Z_A \leftarrow \X_A \relbar\mkern-9mu\relbar \X_{\{1, \ldots, M\} \setminus A}}$,
since $Z_A$ is a function of the (observed) subset of modalities and depends on
the remaining (unobserved) modalities only through $\X_A$.

We can re-write $H(\X \given Z_A)$ as follows:
\begin{align}
  H(\X \given Z_A) &= H(\X \given Z_A, \X_A) + I(\X; \X_A \given Z_A) \label{eq:tmp1} \\
  &= H(\X \given \X_A) + I(\X; \X_A \given Z_A) \label{eq:tmp2} \\
  &= H(\X_{\{1, \ldots, M\} \setminus A} \given \X_A) + I(\X; \X_A \given Z_A) \label{eq:tmp3}\\
  &= H(\X_{\{1, \ldots, M\} \setminus A} \given \X_A) + H(\X_A \given Z_A) \label{eq:tmp4}
\end{align}
\Cref{eq:tmp1} applies the definition of the conditional mutual
information. \Cref{eq:tmp2} is based on the conditional independence ${\X \,
\indep\, Z_A \given \X_A}$ implied by the Markov chain. \Cref{eq:tmp3} removes
the ``known'' information that we condition on. Finally, \Cref{eq:tmp4} follows
from $\X_A \subseteq \X$, which implies that $I(\X; \X_A) = H(\X_A)$ and $I(\X;
\X_A \given Z_A) = H(\X_A \given Z_A)$.

\end{proof}

\subsection{Proof of Theorem~\ref{thm:irreducible_error}}
\label{app:proof_of_thm}

\theoremirreducible*

\begin{proof}

\Cref{lemma:vib} shows that all mixture-based multimodal VAEs
approximate the expected log-evidence via the more general VIB objective
\begin{equation}
  \label{eq:ib_mixture}
  \min_{\psi} \sum_{A \in \mathcal{S}} \omega_A \left\{ H_{\psi}(\X \given Z_A) + I_{\psi}(\X_A; Z_A) \right\}
\end{equation}
where the encoding $Z_A = f_{\psi}(\X_A)$ is a function of a subset $\X_A \subseteq \X$. 

The fact that $Z_A$ is a function of a \textit{subset}, permits the following
decomposition of the conditional entropy (see \Cref{lemma:decomposition}):
\begin{align}
  \label{eq:explanation_decomposition}
  H(\X \given Z_A)
  &= H(\X_{\{1, \ldots, M\} \setminus A} \given \X_A) + H(\X_A \given Z_A)\;.
\end{align}
In particular, \Cref{eq:explanation_decomposition} holds for every $Z_A =
f_{\psi}(\X_A)$ and thus for every value $\psi$. Further, notice that
$H(\X_{\{1, \ldots, M\} \setminus A} \given \X_A)$ is independent of the
learned encoding $Z_A$ and thus remains constant during the optimization with
respect to $\psi$.

Hence, for every value $\psi$, the following inequality holds:
\begin{align}
  H_{\psi}(\X \given Z_A) 
  &\geq H(\X \given Z_A) \\
  &\geq H(\X_{\{1, \ldots, M\} \setminus A} \given \X_A)
\end{align}
which means that the minimization of $H_{\psi}(\X \given Z_A)$
is lower-bound by $H(\X_{\{1, \ldots, M\} \setminus A} \given \X_A)$, even if
$H_{\psi}(\X \given Z_A)$ is a tight estimator of $H(\X \given Z_A)$.

Analogously, for the optimization of the VIB objective (\Cref{eq:ib_mixture}),
for every value $\psi$, the following inequality holds:
\begin{align}
  &\sum_{A \in \mathcal{S}} \omega_A \left\{ H_{\psi}(\X \given Z_A) + I_{\psi}(\X_A; Z_A) \right\} \\
  &\geq \sum_{A \in \mathcal{S}} \omega_A \left\{ H(\X \given Z_A) + I_{\psi}(\X_A; Z_A) \right\} \\
  &= \sum_{A \in \mathcal{S}} \omega_A \left\{ H(\X_A \given Z_A) + I_{\psi}(\X_A; Z_A) \right\}
  + \underbrace{\sum_{A \in \mathcal{S}} \omega_A \, H(\X_{\{1, \ldots, M\} \setminus A} \given \X_A) }_{\hspace{0.78cm}\Delta(\X, \mathcal{S})}
 \label{eq:ib_mixture_decomposed}
\end{align}
where $\Delta(\X, \mathcal{S})$ is independent of $\psi$ and thus remains
constant during the optimization. Consequently, $\Delta(\X, \mathcal{S})$
represents an irreducible error for the optimization of the VIB objective.

For mixture-based multimodal VAEs, \Cref{lemma:vib} shows that
$\mathcal{L}_{\mathcal{S}}(\x; \theta, \phi)$ is a special case of the VIB
objective with $\psi = (\theta, \phi)$. Hence, for every value of $\theta$ and
$\phi$, the following inequality holds:
\begin{align}
  \E_{p(\x)}[\log p(\x)] \geq \mathcal{L}_{\mathcal{S}}(\x; \theta, \phi) + \Delta(\X, \mathcal{S}) \;.
\end{align}
The exact value of $\Delta(\X, \mathcal{S})$ depends on the definition of the
mixture distribution $\mathcal{S}$, as well as on the amount of
modality-specific variation in the data. In particular, $\Delta(\X,
\mathcal{S}) > 0$, if there is any subset $A \in \mathcal{S}$ with $\omega_A >
0$ for which $H(\X_{\{1, \ldots, M\} \setminus A} \given \X_A) > 0$.

\end{proof}

\subsection{Proof of Corollary~\ref{cor:no_modality_subsampling}}
\label{app:cor:no_modality_subsampling}

\corollarynomodalitysubsampling*
\begin{proof}
  Without modality sub-sampling, $\mathcal{S}$ is comprised of only one subset,
  the complete set of modalities $\{1, \ldots, M\}$, and therefore $\X_A = \X$
  and ${\X_{\{1, \ldots, M\}\setminus A} = \emptyset}$.  It follows that
  $\Delta(\X, \mathcal{S}) = H(\X_{\{1, \ldots, M\} \setminus A} \given \X_A) =
  H(\emptyset \given \X) = 0$, since the conditional entropy of the empty set
  is zero.

\end{proof}

\subsection{Proof of Corollary~\ref{cor:more_modalities}}
\label{app:cor:more_modalities}

\begin{customcorollary}{2}
  For the MMVAE and MoPoE-VAE, the generative discrepancy increases given an
  additional modality $X_{M+1}$, if the new modality is sufficiently diverse in
  the following sense:
  \begin{align}
  \left(\frac{1}{\vert \mathcal{S}^+ \vert} - \frac{1}{\vert \mathcal{S} \vert}\right) \sum_{A \in \mathcal{S}} I(\X_{\{1, \ldots, M\} \setminus A}; X_{M+1} \given \X_A)\, <\, \; 
  &\frac{1}{\vert \mathcal{S}^+ \vert \vert \mathcal{S} \vert} \sum_{A \in \mathcal{S}} H(\X_A \given X_{M+1}) \;+ \\
  &\,\frac{1}{\vert \mathcal{S}^+ \vert} \sum_{A \in \mathcal{S}} H(X_{M+1} \given \X)
  \end{align}
  where $\mathcal{S}$ denotes the model-specific mixture distribution over the set
  of subsets of modalities given modalities $X_1,\ldots,X_M$ and
  $\mathcal{S}^+$ is the respective mixture distribution over the extended set
  of subsets of modalities given $X_1, \ldots, X_{M+1}$.
\end{customcorollary}

\begin{proof}
Let $X_{M+1}$ be the new modality, let ${\X^{+} \coloneqq \{X_1, \ldots,
X_{M+1}\}}$ denote the extended set of modalities, and let $\mathcal{S}^+$
denote the new mixture distribution over subsets given $\X^+$. Note that all
subsets from $\mathcal{S}$ are still contained in $\mathcal{S}^+$, but that 
$\mathcal{S}^+$ contains new subsets in addition to those in $\mathcal{S}$.
Further, due to the re-weighting of mixture coefficients, $\mathcal{S}^+$ can
have different mixture coefficients for the subsets it shares with
$\mathcal{S}$. We denote by ${S^- \coloneqq \{ (A, \omega_A^+) \in
\mathcal{S}^+ : A \not \in \mathcal{S} \}}$ the set of new subsets and let
$\omega^+_A$ denote the new mixture coefficients, where typically $\omega_A \not
= \omega_A^+$ due to the re-weighting.

We are interested in the change of the generative discrepancy, when we add modality $X_{M+1}$:
\begin{align}
  &\Delta(\X^+, \mathcal{S}^+) - \Delta(\X, \mathcal{S}) \\
  &= \sum_{B \in \mathcal{S}^+} \omega_B^+\,H(\X_{\{1, \ldots, M+1\} \setminus B} \given \X_B)
    - \sum_{A \in \mathcal{S}} \omega_A\,H(\X_{\{1, \ldots, M\} \setminus A} \given \X_A)\;.
\label{eq:more_modalities:tmp0}
\end{align}
Re-write the right hand side in terms of subsets that are contained in
both $\mathcal{S}$ and $\mathcal{S}^+$ and subsets that are only contained
in $\mathcal{S}^+$. For this, we decompose the first term as follows
\begin{align}
  &\sum_{B \in \mathcal{S}^+} \omega_B^+\,H(\X_{\{1, \ldots, M+1\} \setminus B} \given \X_B) \\
  &= \sum_{A \in \mathcal{S}} \omega_A^+\,H(\X_{\{1, \ldots, M+1\} \setminus A} \given \X_A)
  + \sum_{B \in \mathcal{S}^-} \omega_B^+\,H(\X_{\{1, \ldots, M+1\} \setminus B} \given \X_B) \\
  &= \sum_{A \in \mathcal{S}} \omega_A^+\,H(\X_{\{1, \ldots, M\} \setminus A} \given \X_A)
  + \sum_{A \in \mathcal{S}} \omega_A^+\,H(X_{M+1} \given \X) \; + \\
  &\hskip0.98em\relax \sum_{B \in \mathcal{S}^-} \omega_B^+\,H(\X_{\{1, \ldots, M+1\} \setminus B} \given \X_B)
  \label{eq:more_modalities:tmp1}
\end{align}
where the last equation follows from 
\begin{align}
  H(\X_{\{1, \ldots, M+1\} \setminus A} \given \X_A) 
  &= H(\X_{\{1, \ldots, M\} \setminus A} \given \X_A) + H(X_{M+1} \given \X_A, \X_{\{1, \ldots, M\}\setminus A}) \\
  &= H(\X_{\{1, \ldots, M\} \setminus A} \given \X_A) + H(X_{M+1} \given \X)\;.
\end{align}
We can use the decomposition from
\Cref{eq:more_modalities:tmp1} to re-write the right hand side of
\Cref{eq:more_modalities:tmp0} by collecting the corresponding terms for 
$H(\X_{\{1, \ldots, M\} \setminus A} \given \X_A)$:
\begin{equation}
  \begin{split}
  &\sum_{A \in \mathcal{S}} (\omega_A^+ - \omega_A)H(\X_{\{1, \ldots, M\} \setminus A} \given \X_A)
    + \sum_{A \in \mathcal{S}} \omega_A^+\,H(X_{M+1} \given \X) \; + \\
  &\sum_{B \in \mathcal{S}^-} \omega_B^+\,H(\X_{\{1, \ldots, M+1\} \setminus B} \given \X_B)\;.\label{eq:more_modalities:tmp2}
  \end{split}
\end{equation}
Notice that in \Cref{eq:more_modalities:tmp2} only the first term can be
negative, due to the re-weighting of mixture coefficients for terms that do not
contain $X_{M+1}$. Hence, in the general case, the generative discrepancy can only
decrease, if the mixture coefficients change in such a way that the first term
in \Cref{eq:more_modalities:tmp2} dominates the other two terms. 

For the relevant special case of uniform mixture weights, which applies to both
the MMVAE and MoPoE-VAE, we can further decompose
\Cref{eq:more_modalities:tmp2} into $(i)$ information shared between $\X$ and
$X_{M+1}$, and $(ii)$ information that is specific to $\X$ or $X_{M+1}$. 

Using uniform mixture coefficients $\omega_A = \frac{1}{\vert \mathcal{S}
\vert}$ and $\omega_A^+ = \frac{1}{\vert \mathcal{S}^+ \vert}$ for all subsets,
we can factor out the coefficients and re-write
\Cref{eq:more_modalities:tmp2} as follows:
\begin{equation}
\label{eq:more_modalities:tmp3}
  \begin{split}
  \left(\frac{1}{\vert \mathcal{S}^+ \vert} - \frac{1}{\vert \mathcal{S} \vert}\right) \sum_{A \in \mathcal{S}} H(\X_{\{1, \ldots, M\} \setminus A} \given \X_A)
  + \frac{1}{\vert \mathcal{S}^+ \vert} \sum_{A \in \mathcal{S}} H(X_{M+1} \given \X) \; +\\
  \frac{1}{\vert \mathcal{S}^+ \vert} \sum_{B \in \mathcal{S}^-} H(\X_{\{1, \ldots, M+1\} \setminus B} \given \X_B)
  \end{split}
\end{equation}
where the second term already denotes information that is specific to
$X_{M+1}$. Hence, we decompose the first and last terms corresponding to
$(i)$ and $(ii)$.

For the first term from \Cref{eq:more_modalities:tmp3}, we have 
\begin{align}
  &\left(\frac{1}{\vert \mathcal{S}^+ \vert} - \frac{1}{\vert \mathcal{S} \vert}\right) \sum_{A \in \mathcal{S}} H(\X_{\{1, \ldots, M\} \setminus A} \given \X_A) \\
  &= \left(\frac{1}{\vert \mathcal{S}^+ \vert} - \frac{1}{\vert \mathcal{S} \vert}\right) \sum_{A \in \mathcal{S}} \Big\{ H(\X_{\{1, \ldots, M\} \setminus A} \given \X_A, X_{M+1}) + I(\X_{\{1, \ldots, M\} \setminus A}; X_{M+1} \given \X_A) \Big\}.
\label{eq:more_modalities:tmp4}
\end{align}

For the last term from \Cref{eq:more_modalities:tmp3}, we have 
\begin{align}
&\frac{1}{\vert \mathcal{S}^+ \vert} \sum_{B \in \mathcal{S}^-} H(\X_{\{1, \ldots, M+1\} \setminus B} \given \X_B) \\
&=\frac{1}{\vert \mathcal{S}^+ \vert} \Big\{ H(\X \given X_{M+1}) + \sum_{A \in \mathcal{S}} \mathbf{1}_{\left\{(A \cup \left\{ M+1 \right\}) \in \mathcal{S}^-\right\}} H(\X_{\{1, \ldots, M\} \setminus A} \given \X_A, X_{M+1}) \Big\}
\label{eq:more_modalities:tmp5}
\end{align}
where we can further decompose 
\begin{align}
\frac{1}{\vert \mathcal{S}^+ \vert} H(\X \given X_{M+1}) 
&= \frac{1}{\vert \mathcal{S}^+ \vert} \Big\{ H(\X \given \X_A, X_{M+1}) + I(\X; \X_A \given X_{M+1}) \Big\} \\
&= \frac{1}{\vert \mathcal{S}^+ \vert} \Big\{ H(\X \given \X_A, X_{M+1}) + H(\X_A \given X_{M+1}) \Big\} \\
&= \frac{1}{\vert \mathcal{S}^+ \vert \vert \mathcal{S} \vert} \sum_{A \in \mathcal{S}} \Big\{ H(\X_{\{1, \ldots, M\} \setminus A} \given \X_A, X_{M+1}) + H(\X_A \given X_{M+1}) \Big\}.
\label{eq:more_modalities:tmp6}
\end{align}

Collecting all corresponding terms from
\Cref{eq:more_modalities:tmp4,eq:more_modalities:tmp5,eq:more_modalities:tmp6},
we can re-write \Cref{eq:more_modalities:tmp3} as follows:
\begin{align}
  \left(\frac{1}{\vert \mathcal{S}^+ \vert} - \frac{1}{\vert \mathcal{S} \vert} + \frac{1}{\vert \mathcal{S}^+ \vert \vert \mathcal{S} \vert} \right) \sum_{A \in \mathcal{S}} H(\X_{\{1, \ldots, M\} \setminus A} \given \X_A, X_{M+1}) \;+ \\
  \left(\frac{1}{\vert \mathcal{S}^+ \vert} - \frac{1}{\vert \mathcal{S} \vert} \right) \sum_{A \in \mathcal{S}} I(\X_{\{1, \ldots, M\} \setminus A}; X_{M+1} \given \X_A) \;+ \\
  \frac{1}{\vert \mathcal{S}^+ \vert} \sum_{A \in \mathcal{S}} \mathbf{1}_{\left\{(A \cup \left\{ M+1 \right\}) \in \mathcal{S}^-\right\}} H(\X_{\{1, \ldots, M\} \setminus A} \given \X_A, X_{M+1}) \;+ \\
  \frac{1}{\vert \mathcal{S}^+ \vert \vert \mathcal{S} \vert} \sum_{A \in \mathcal{S}} H(\X_A \given X_{M+1}) \;+ \\
  \frac{1}{\vert \mathcal{S}^+ \vert} \sum_{A \in \mathcal{S}} H(X_{M+1} \given \X).
\end{align}

For both the MMVAE and MoPoE, the first and last terms cancel out, which can
see by plugging in the respective definitions of $\mathcal{S}$ into the above
equation. Recall that for the MMVAE, $\mathcal{S}$ is comprised of the set of
unimodal subsets $\{\{\x_1\}, \ldots, \{\x_M\}\}$ and thus $\mathcal{S}^+$ is
comprised of $\{\{\x_1\}, \ldots, \{\x_{M+1}\}\}$. For the MoPoE-VAE,
$\,\mathcal{S}$ is comprised of the powerset $\mathcal{P}(M) \setminus
\{\emptyset\}$ and thus $\mathcal{S}^+$ is comprised of the powerset
$\mathcal{P}(M+1) \setminus \{\emptyset\}$. Hence, for the MMVAE and MoPoE-VAE,
we have shown that $\Delta(\X^+, \mathcal{S}^+) - \Delta(\X, \mathcal{S})$ is
equal to the following expression:
\begin{align}
  \left(\frac{1}{\vert \mathcal{S}^+ \vert} - \frac{1}{\vert \mathcal{S} \vert}\right) \sum_{A \in \mathcal{S}} I(\X_{\{1, \ldots, M\} \setminus A}; X_{M+1} \given \X_A) \; +  \label{eq:more_modalities:tmp8} \\
\frac{1}{\vert \mathcal{S}^+ \vert \vert \mathcal{S} \vert} \sum_{A \in \mathcal{S}} H(\X_A \given X_{M+1})
  + \frac{1}{\vert \mathcal{S}^+ \vert} \sum_{A \in \mathcal{S}} H(X_{M+1} \given \X) \label{eq:more_modalities:tmp7}
\end{align}
where the information is decomposed into:
\begin{enumerate}[$(i)$]
  \item information shared between $\X$ and $X_{M+1}$ (term~(\ref{eq:more_modalities:tmp8})), and 
  \item information that is specific to $\X$ or $X_{M+1}$ (the first and second terms in~(\ref{eq:more_modalities:tmp7}) respectively),
\end{enumerate}
and where only $(i)$ can be negative since $\vert \mathcal{S}^+ \vert > \vert
\mathcal{S} \vert$.  This concludes the proof of \Cref{cor:more_modalities},
showing that $\Delta(\X^+, \mathcal{S}^+) - \Delta(\X, \mathcal{S}) > 0$, if
$X_{M+1}$ is sufficiently diverse in the sense that ${(ii) > (i)}$.

\end{proof}

\newpage
\section{Experiments}
\label{app:experiments}

\subsection{Description of the datasets}
\label{app:description_of_data}

\paragraph{PolyMNIST}%
The PolyMNIST dataset, introduced in \citet{Sutter2021}, combines the MNIST dataset
\citep{LeCun1998} with crops from five different background images to create
five synthetic image modalities. Each sample from the data is a set of five
MNIST images (with digits of the same class) overlayed on $28 \times 28$
crops from five different background images.
\Cref{fig:vanilla_polymnist_example} shows 10 samples from the PolyMNIST
dataset; each column represents one sample and each row represents one
modality. The dataset provides a convenient testbed for the evaluation of
generative coherence, because by design only the digit information is shared
between modalities.

\paragraph{Translated-PolyMNIST}%
This new dataset is conceptually similar to PolyMNIST in that a digit label is
shared between five synthetic image modalities. The difference is that in the
creation of the dataset, we change the size and position of the digit, as shown
in \Cref{fig:translated_polymnist_example}. Technically, instead of overlaying
a full-sized $28 \times 28$ MNIST digit on a patch from the respective
background image, we downsample the MNIST digit by a factor of two and place it
at a random $(x, y)$-coordinate within the $28 \times 28$ background patch.
Conceptually, these transformations leave the shared information  between
modalities (i.e., the digit label) unaffected and only serve to make it more
difficult to predict the shared information across modalities on expectation.

\paragraph{Caltech Birds (CUB)}%
The extended CUB dataset from \citet{Shi2019} is comprised of two modalities,
images and captions. Each image from Caltech-Birds
\citep[CUB-200-2011][]{Wah2011} is coupled with 10 crowdsourced descriptions of
the respective bird. \Cref{fig:cub_example} shows five samples from the
dataset. It is important to note that we use the CUB dataset with
\textit{real images}, instead of the simplified version based on precomputed
ResNet-features that was used in \citet{Shi2019,Shi2021}.

\subsection{Implementation details}
\label{app:implementation_details}

Our experiments are based on the publicly available code from
\citet{Sutter2021}, which already provides an implementation of PolyMNIST\@. A
notable difference in our implementation is that we employ ResNet
architectures, because we found that the previously used convolutional neural
networks did not have sufficient capacity for the more complex datasets we use.
For internal consistency, we use ResNets for PolyMNIST as well. We have
verified that there is no significant difference compared to the results from
\citet{Sutter2021} when we change to ResNets.

\paragraph{Hyperparameters}
All models were trained using the Adam optimizer \citep{Kingma2015} with
learning rate 5e-4 and a batch size of 256.  For image modalities we estimate
likelihoods using Laplace distributions and for captions we employ one-hot
categorical distributions.  Models were trained for 500, 1000, and 150 epochs
on PolyMNIST, Translated-PolyMNIST, and CUB respectively.  Similar to previous
work, we use Gaussian priors and a latent space with 512 dimensions for
PolyMNIST and 64 dimensions for CUB\@. For a fair comparison, we reduce the
latent dimensionality of unimodal VAEs proportionally (wrt.\ the number of
modalities) to control for capacity. For the $\beta$-ablations, we use $\beta
\in \{$3e-4, 3e-3, 3e-1, 1, 3, 9$\}$ and, in addition, 32 for CUB.

\paragraph{Evaluation metrics}
For the evaluation of \textit{generative quality}, we use the Fr\'echet
inception distance \citep[FID;][]{Heusel2017}, a standard metric for evaluating
the quality of generated images. In \Cref{app:additional_experimental_results},
we also provide log-likelihoods and qualitative results for both images and
captions.  To compute \textit{generative coherence}, we adopt the definitions
from previous works \citep{Shi2019,Sutter2021}. Generative coherence requires
annotation on what is shared between modalities; for example, in both PolyMNIST
and Translated-PolyMNIST the digit label is shared by design. For a single
generated example $\hat{\x}_m \sim q_{\phi}(\x_m \given \z)$ from modality $m$,
the generative coherence is computed as the following indicator: 
\begin{equation}
  \text{Coherence}(\hat{\x}_m, y, g_m) = \bm{1}_{\{g_m(\hat{\x}_m)\,=\,y\}}
\end{equation}
where $y$ is a ground-truth class label and $g_m$ is a pretrained classifier
(learned on the training data from modality $m$) that outputs a predicted class
label. To compute the \textit{conditional coherence accuracy}, we average the
coherence values over a set of $N$ conditionally generated examples, where $N$
is typically the size of the test set. In particular, when $\hat{\x}_m \sim
q_{\phi}(\x_m \given \z)$ is conditionally generated from $\z \sim
p_{\theta}(\z \given \x_A)$ such that $A = \{1, \ldots, M\} \setminus m$, the
metric is specified as the \textit{leave-one-out conditional coherence
accuracy}, because the input consists of all modalities except the one that is
being generated. When it is clear from context which metric is used, we refer
to the (leave-one-out) conditional coherence accuracy simply as generative
coherence. For PolyMNIST, we use the pretrained digit classifiers that are
provided in the publicly available code from \citet{Sutter2021} and for
Translated-PolyMNIST we train the classifiers from scratch with the same
architectures that are used for the VAE encoders. Notably, the new pretrained
digit classifiers have a classification accuracy between 93.5--96.9\% on the
test set of the respective modality, which means that it is possible to predict
the digits fairly well with the given architectures.

\subsection{Additional experimental results}
\label{app:additional_experimental_results}

\paragraph{Linear classification}
\citet{Shi2019} propose linear classification as a measure of latent
factorization, to judge the quality of learned representations and to assess
how well the information decomposes into shared and modality-specific features.
\Cref{fig:clfs} shows the linear classification accuracy on the learned
representations.  The results suggest that not only does the generative
coherence decline when we switch from PolyMNIST to Translated-PolyMNIST, but
also the quality of the learned representations. While a low
classification accuracy does not imply that there is no digit information
encoded in the latent representation (after all, digits show up in most
self-reconstructions), the result demonstrates that a \textit{linear}
classifier cannot extract the digit information.

\paragraph{Log-likelihoods and qualitative results}
\Cref{fig:lliks} shows the generative quality in terms of joint
log-likelihoods. We observe a similar ranking of models as with FID, but we
notice that the gap between MVAE and MoPoE-VAE appears less pronounced.  The
reason for this discrepancy is that, to be consistent with \citet{Sutter2021},
we estimate joint log-likelihoods given \textit{all} modalities---a procedure
that resembles reconstruction more than it does unconditional generation. It
can be of independent interest that log-likelihoods might overestimate the
generative quality for unconditional generation for certain types of models.
Qualitative results for unconditional generation
(\Cref{fig:qualitative_unconditional}) support the hypothesis that the
presented log-likelihoods do not reflect the visible lack of generative quality
for the MoPoE-VAE\@. Further, qualitative results for conditional generation
(\Cref{fig:qualitative_conditional}) indicate a lack of diversity for both the
MMVAE and MoPoE-VAE\@: even though we draw different samples from the
posterior, the respective conditionally generated samples (i.e., the ten
samples along each column) show little diversity in terms of backgrounds or
writing styles.

\paragraph{Repeated modalities}
\begin{wrapfigure}{r}{0.27\textwidth}
  \begin{center}
    \includegraphics[width=0.27\textwidth]{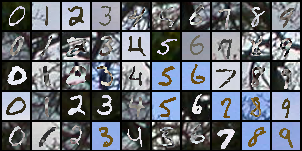}
\caption{PolyMNIST with five ``repeated'' modalities.}%
\label{fig:repeated_modality_example}
\end{center}
\end{wrapfigure}
To check if the generative quality gap is also present when modalities have
\textit{similar} modality-specific variation, we use PolyMNIST with
``repeated'' modalities generated from the same background image (illustrated
in \Cref{fig:repeated_modality_example}). We vary the number of modalities from
2 to 5, but in contrast to the results from \Cref{fig:num_mod_ablation}, we now
use repeated modalities.  \Cref{fig:num_mod_ablation_repeated_modality}
confirms that the generative quality of both the MVAE and MoPoE-VAE
deteriorates with each additional modality, even in this simplified setting
with repeated modalities. In comparison, the generative quality of the MVAE is
much closer to the unimodal VAE for any number of modalities. These results
lend further support to the theoretical statements from
\Cref{cor:no_modality_subsampling,cor:more_modalities}.

\paragraph{MMVAE with the official implementation} The empirical results of the
MMVAE in \Cref{sec:experiments} are based on a simplified version of the model
that was proposed by \citet{Shi2019}. In particular, we use the
re-implementation from \citet{Sutter2021}, which optimizes the standard ELBO
and not the doubly reparameterized ELBO gradient estimator
\citep[DReG,][]{Tucker2019} with importance sampling that is used in the
official implementation from \citet{Shi2019}.  Further, the re-implementation
does not parameterize the prior but uses a fixed, standard normal prior
instead.

To verify that these implementation differences do not affect the core
results---the generative quality gap and the lack of coherence---we conducted
experiments using the MMVAE with the official implementation from
\citet{Shi2019}. \Cref{fig:vanilla_polymnist_beta_ablation_mmvae_official}
shows the $\beta$-ablation for PolyMNIST and it confirms that there is still a
clear gap in generative quality between the unimodal VAE and the MMVAE when we
use the official implementation. For Translated-PolyMNIST (not shown) the
results are similar; in particular, we have verified that generative coherence
for cross generation is random, even if we limit the dataset to two modalities.

\paragraph{MVAE with ELBO sub-sampling}
For the MVAE, \citet{Wu2018} introduce ELBO sub-sampling as an additional
training strategy to learn the inference networks for different subsets of
modalities. In our notation, ELBO sub-sampling can be described by the
following objective:
\begin{equation}
  \mathcal{L}_{}(\x; \theta, \phi) 
  + \sum_{A \in \mathcal{S}} \mathcal{L}_{}(\x_A; \theta, \phi)
\end{equation}
where $\mathcal{S}$ denotes some set of subsets of modalities. \citet{Wu2018}
experiment with different choices for $\mathcal{S}$, but throughout all of
their experiments they use at least the set of unimodal subsets $\{\{\x_1\},
\ldots, \{\x_M\}\}$, which yields the following objective:
\begin{equation}
  \mathcal{L}_{}(\x; \theta, \phi) 
  + \sum_{i = 1}^{M} \mathcal{L}_{}(\x_i; \theta, \phi)\;.
\label{eq:MVAEplus_objective}
\end{equation}
It is important to note that the above objective differs from the objective
optimized by all mixture-based multimodal VAEs (\Cref{def:mixture_based}) in
that there are no cross-modal reconstructions in \Cref{eq:MVAEplus_objective}.
As a consequence, ELBO sub-sampling puts more weight on the approximation of
the marginal distributions compared to the conditionals and therefore does not
optimize a proper bound on the joint distribution \citep{Wu2020}.

\Cref{fig:vanilla_polymnist_beta_ablation_mvaeplus} shows the PolyMNIST
$\beta$-ablation comparing MVAE with and without ELBO sub-sampling.
MVAE$^{{+}}$ denotes the model with ELBO sub-sampling using
objective~(\ref{eq:MVAEplus_objective}). Notably, MVAE$^{{+}}$ achieves
significantly better generative coherence, while both models perform similarly
in terms of generative quality (both in terms of FID and joint log-likelihood).
Hence, even though the MVAE$^{{+}}$ optimizes an incorrect bound on the joint
distribution \citep{Wu2020}, our results suggest that the learned models behave
quite similar in practice, which can be of independent interest for future
work.

\begin{figure}[!b]
\begin{center}
\begin{subfigure}[t]{.41\linewidth}
  \centering
  \includegraphics[width=1.0\linewidth]{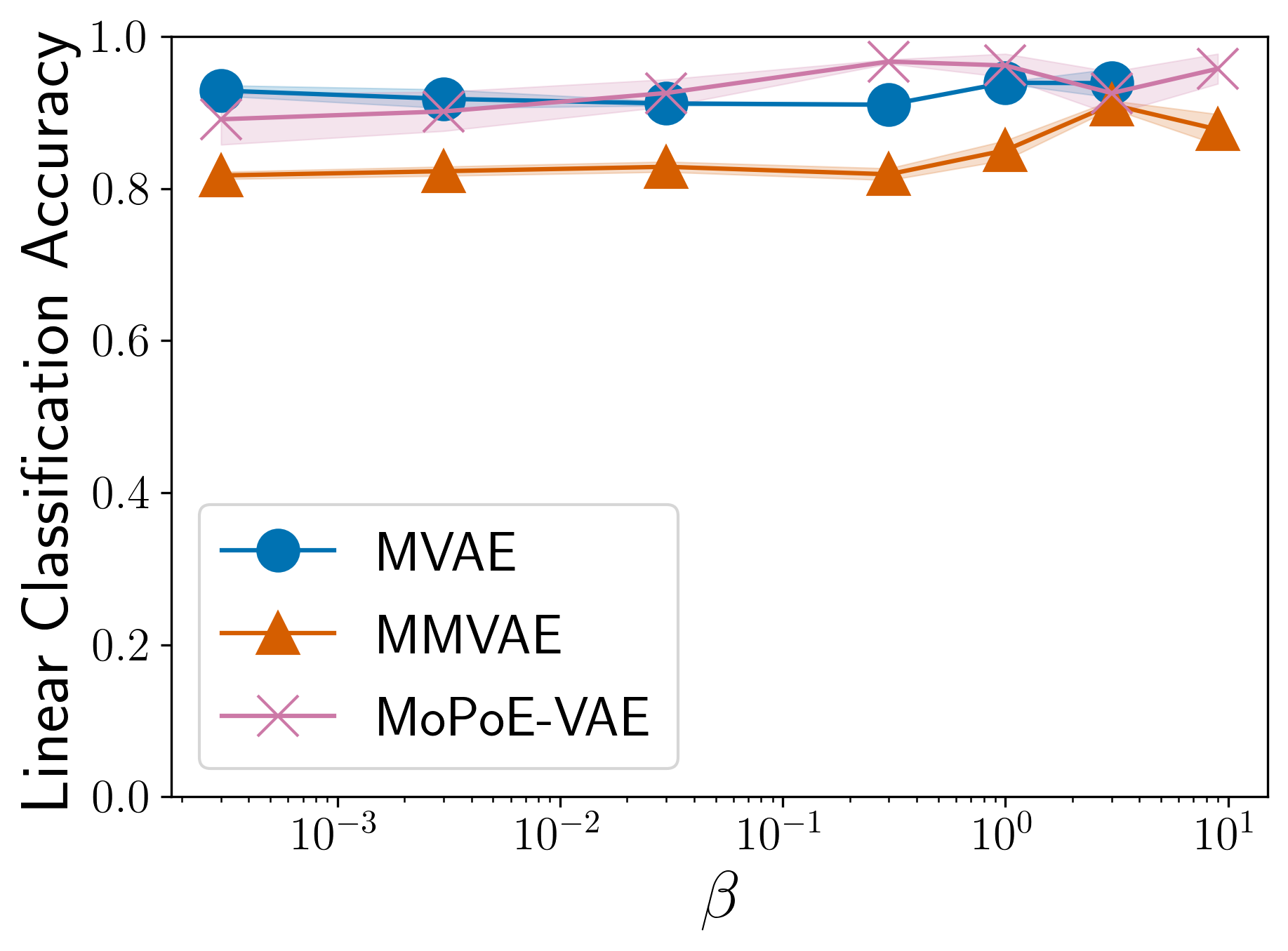}
  \caption{PolyMNIST}
\end{subfigure}%
\hskip +0.335in
\begin{subfigure}[t]{.41\linewidth}
  \centering
  \includegraphics[width=1.0\linewidth]{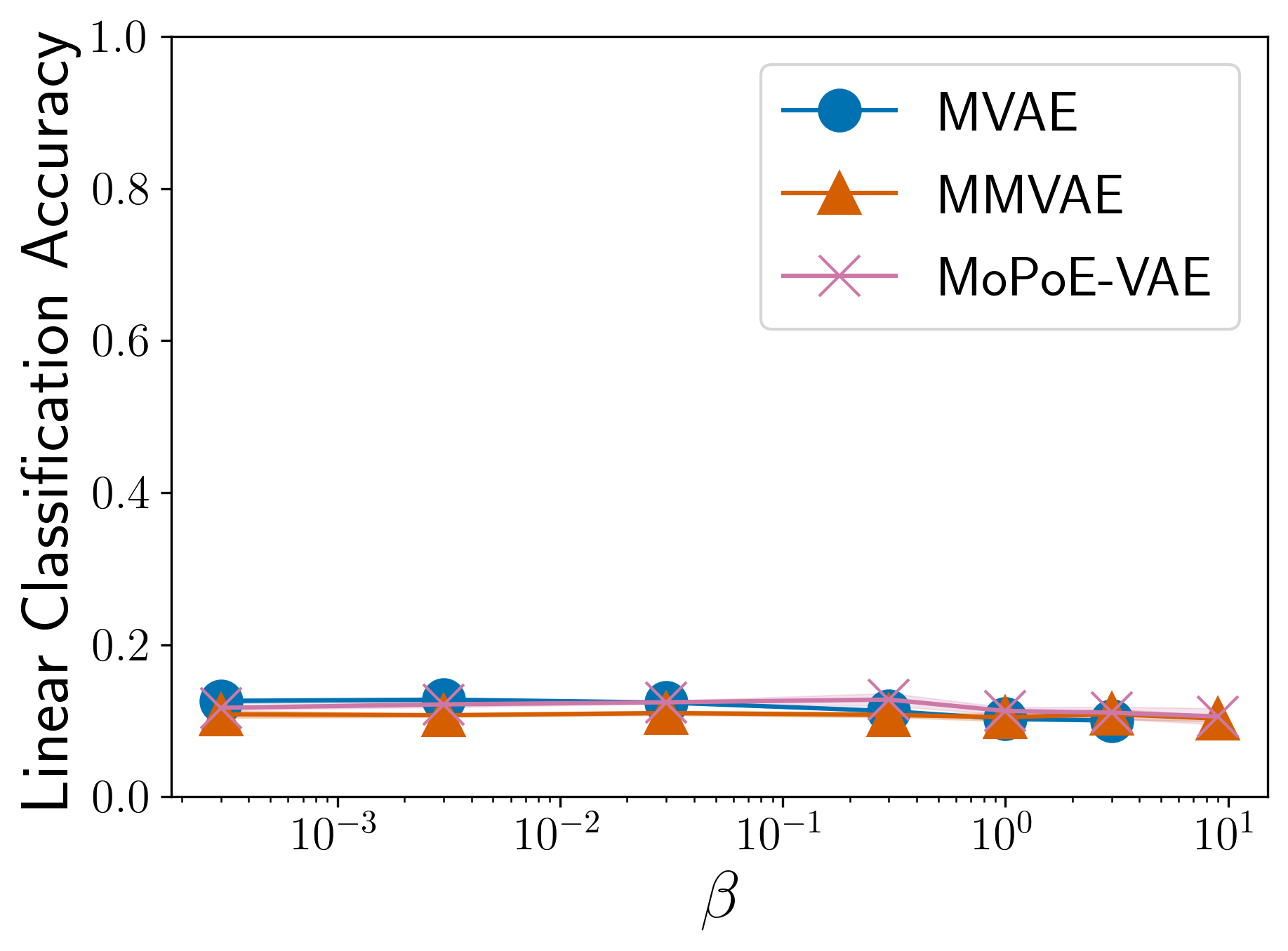}
  \caption{Translated-PolyMNIST}
\end{subfigure}%
\caption{%
  Linear classification of latent representations. For each model, linear
  classifiers were trained on the joint embeddings from 500 randomly sampled
  training examples.  Points denote the average digit classification accuracy
  of the respective classifiers. The results are averaged over three seeds
  and the bands show one standard deviation respectively.  Due to numerical
  instabilities, the MVAE could not be trained with larger $\beta$ values. For
  CUB, classification performance cannot be computed, because shared factors
  are not annotated.
}
\label{fig:clfs}
\end{center}
\end{figure}

\newpage
\begin{figure}[t]
\begin{center}
\begin{subfigure}[t]{.32\linewidth}
  \centering
  \includegraphics[width=1.0\linewidth]{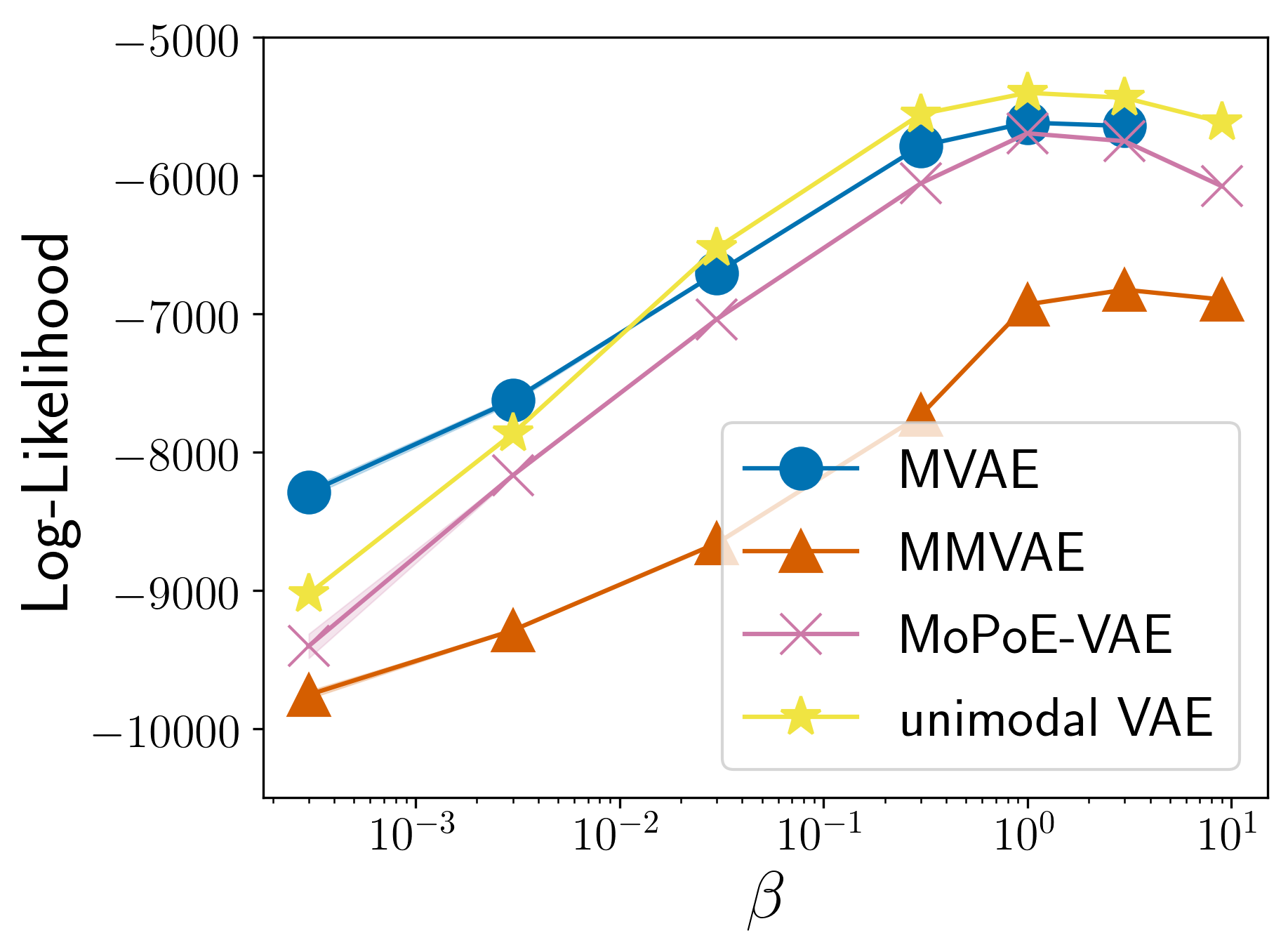}
  \caption{PolyMNIST}
\end{subfigure}%
\begin{subfigure}[t]{.32\linewidth}
  \centering
  \includegraphics[width=1.0\linewidth]{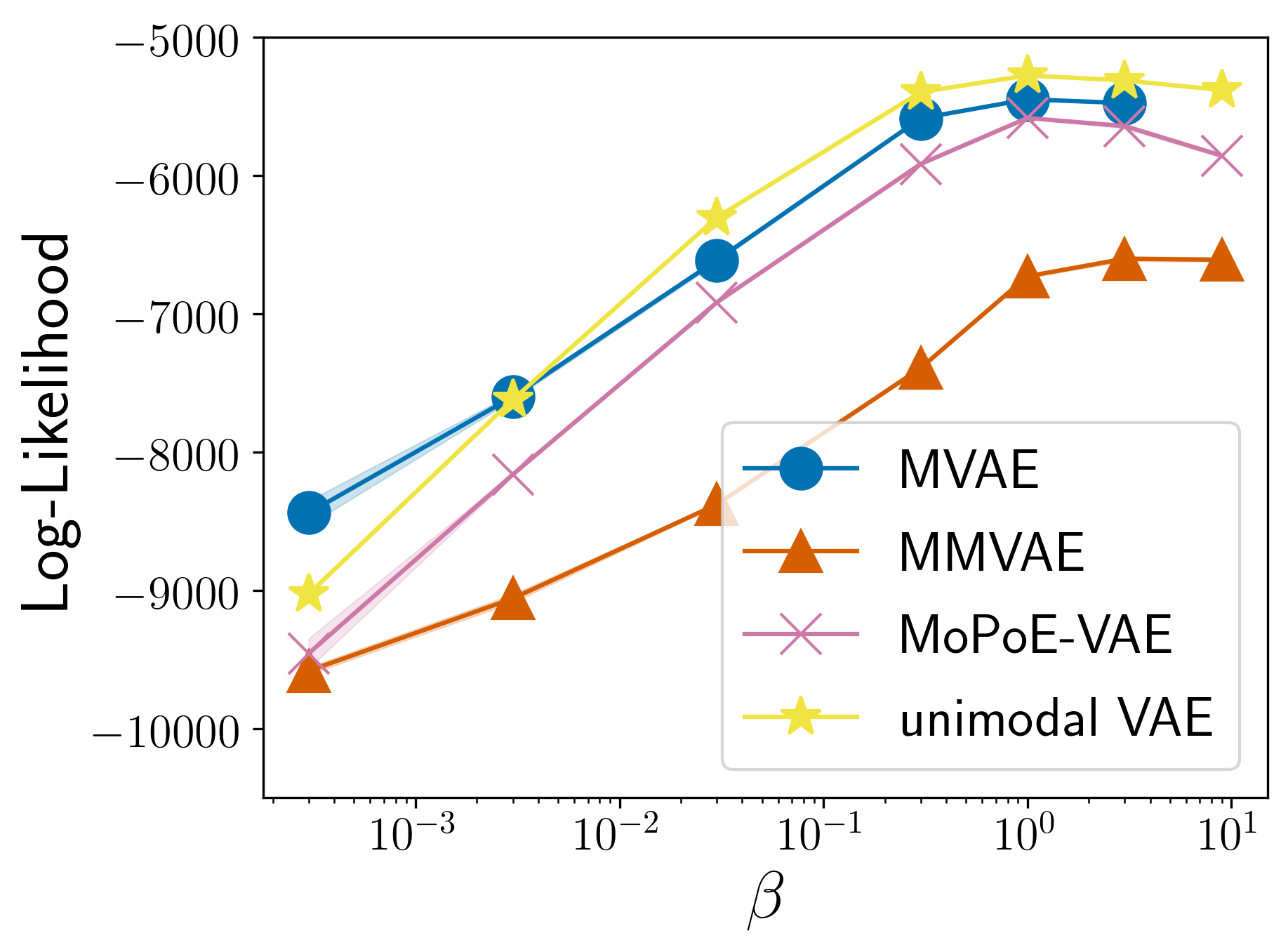}
  \caption{Translated-PolyMNIST}
\end{subfigure}%
\begin{subfigure}[t]{.32\linewidth}
  \centering
  \includegraphics[width=1.0\linewidth]{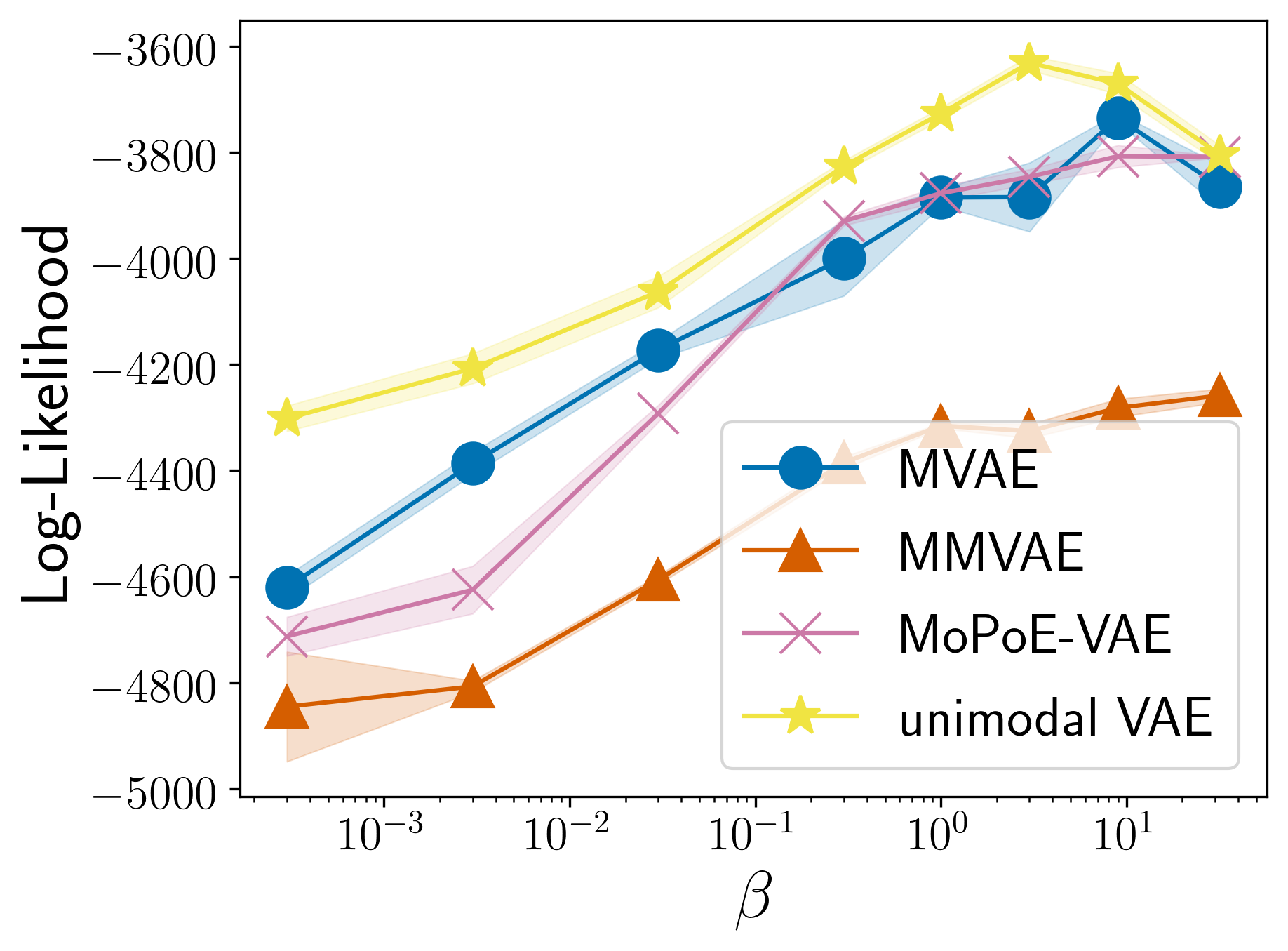}
  \caption{Caltech Birds (CUB)}
\end{subfigure}
\caption{%
  Joint log-likelihoods over a range of $\beta$ values. Each point
  denotes the estimated joint log-likelihood averaged over three different
  seeds and the bands show one standard deviation respectively. Due
  to numerical instabilities, the  MVAE could not be trained with larger $\beta$
  values.
}
\label{fig:lliks}
\end{center}
\end{figure}

\begin{figure}[t]
\begin{center}
\begin{subfigure}[t]{.195\linewidth}
  \centering
  \includegraphics[width=1.0\linewidth]{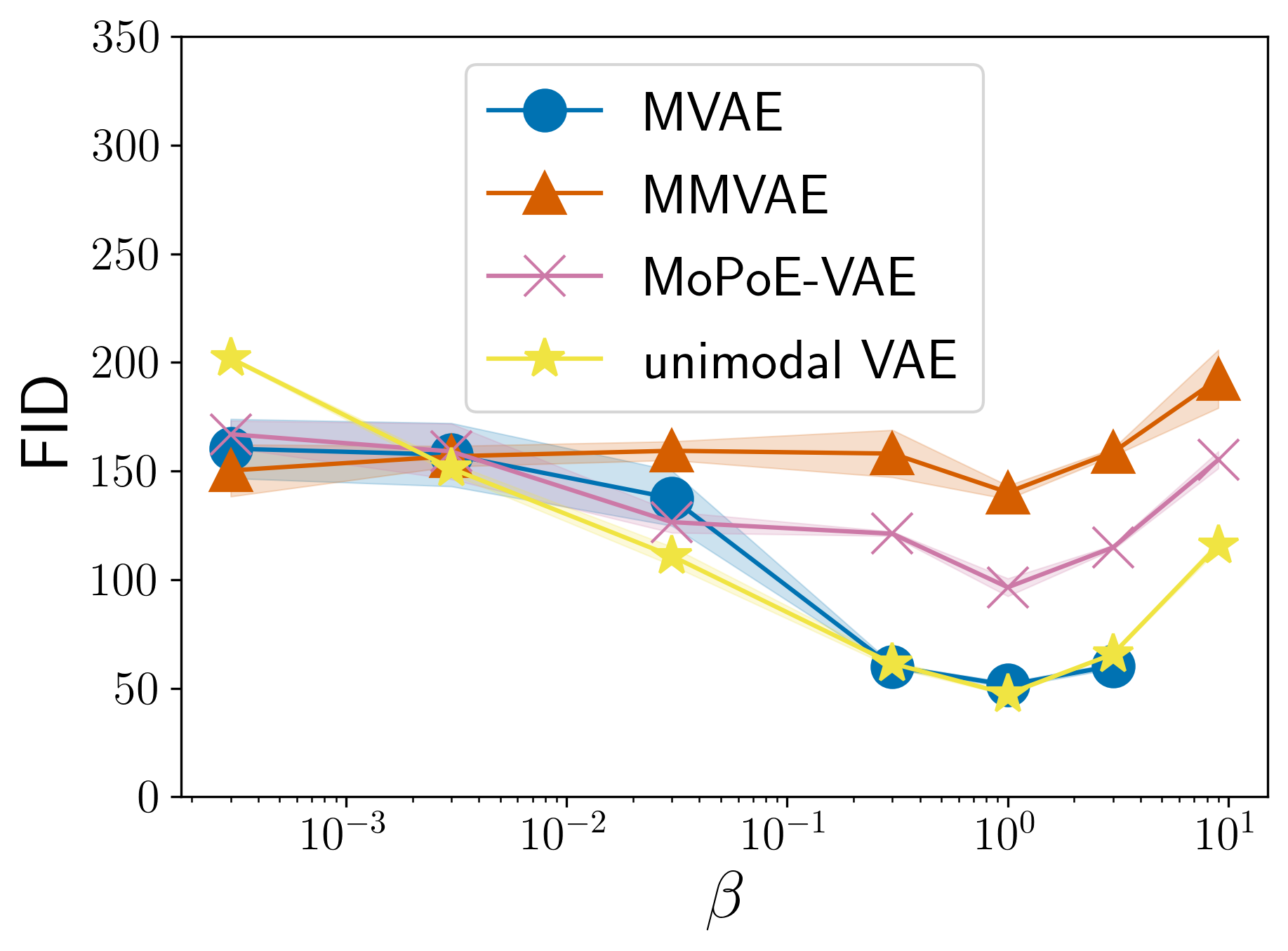}
\end{subfigure}%
\begin{subfigure}[t]{.195\linewidth}
  \centering
  \includegraphics[width=1.0\linewidth]{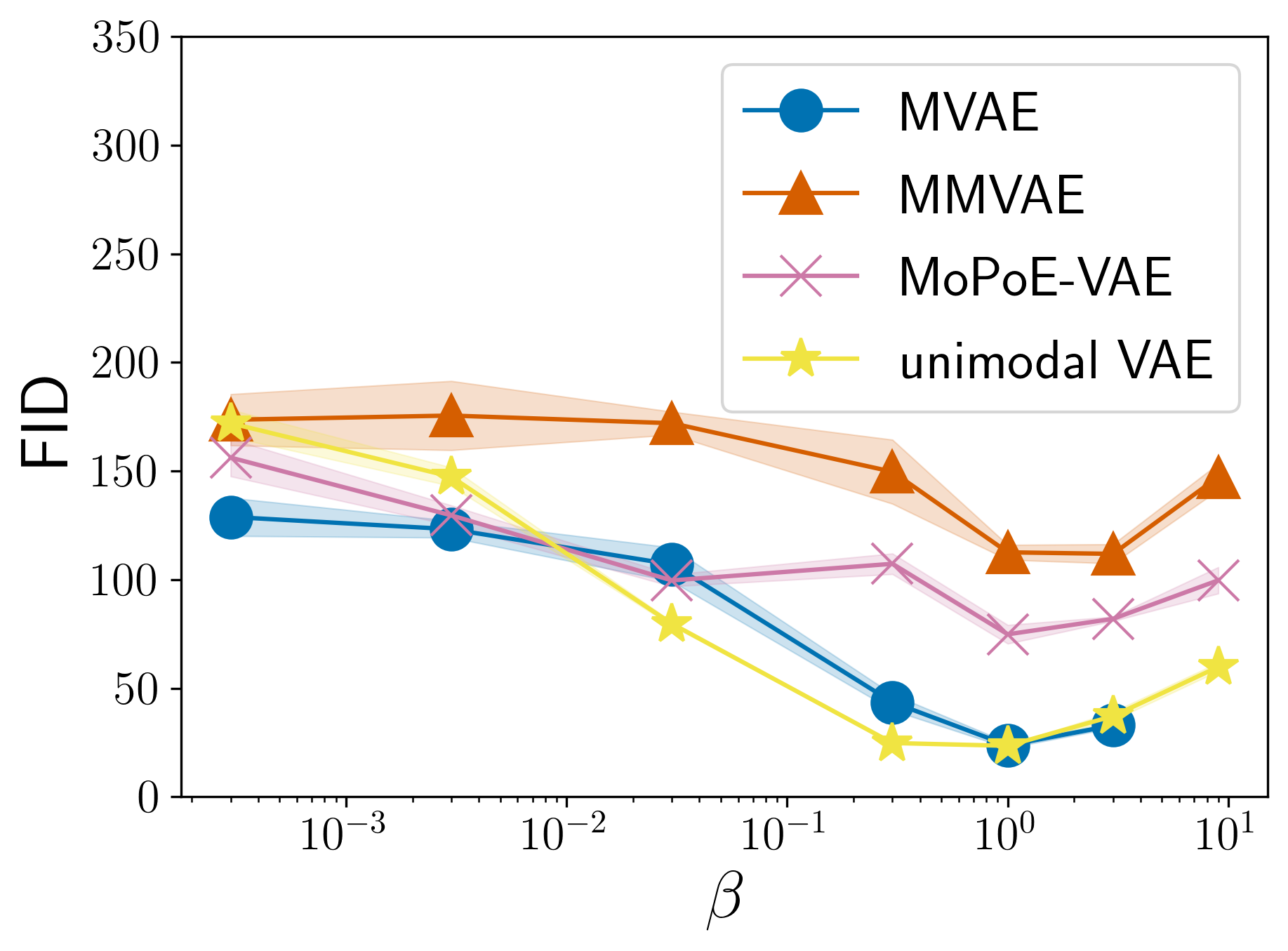}
\end{subfigure}%
\begin{subfigure}[t]{.195\linewidth}
  \centering
  \includegraphics[width=1.0\linewidth]{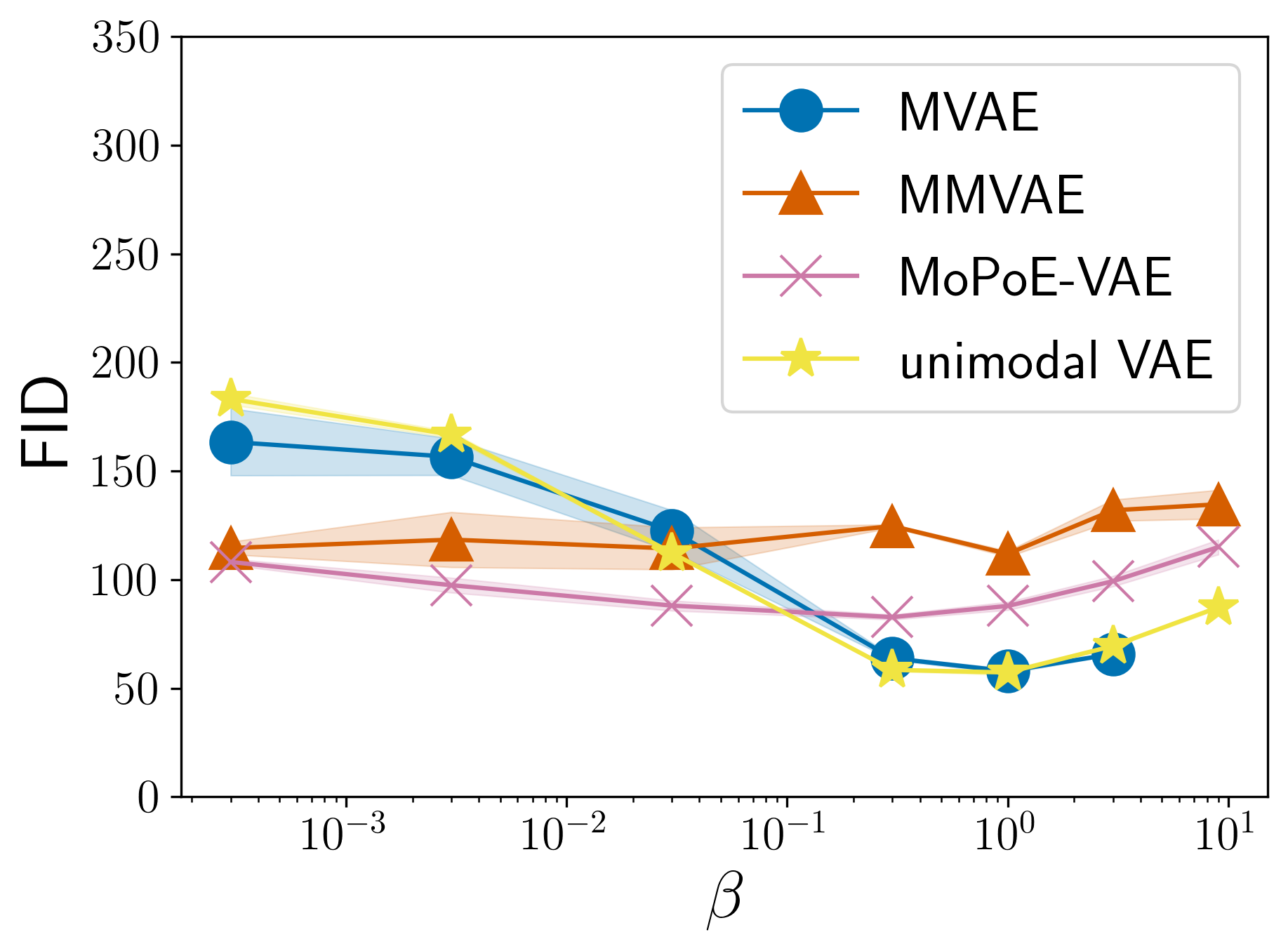}
\end{subfigure}
\begin{subfigure}[t]{.195\linewidth}
  \centering
  \includegraphics[width=1.0\linewidth]{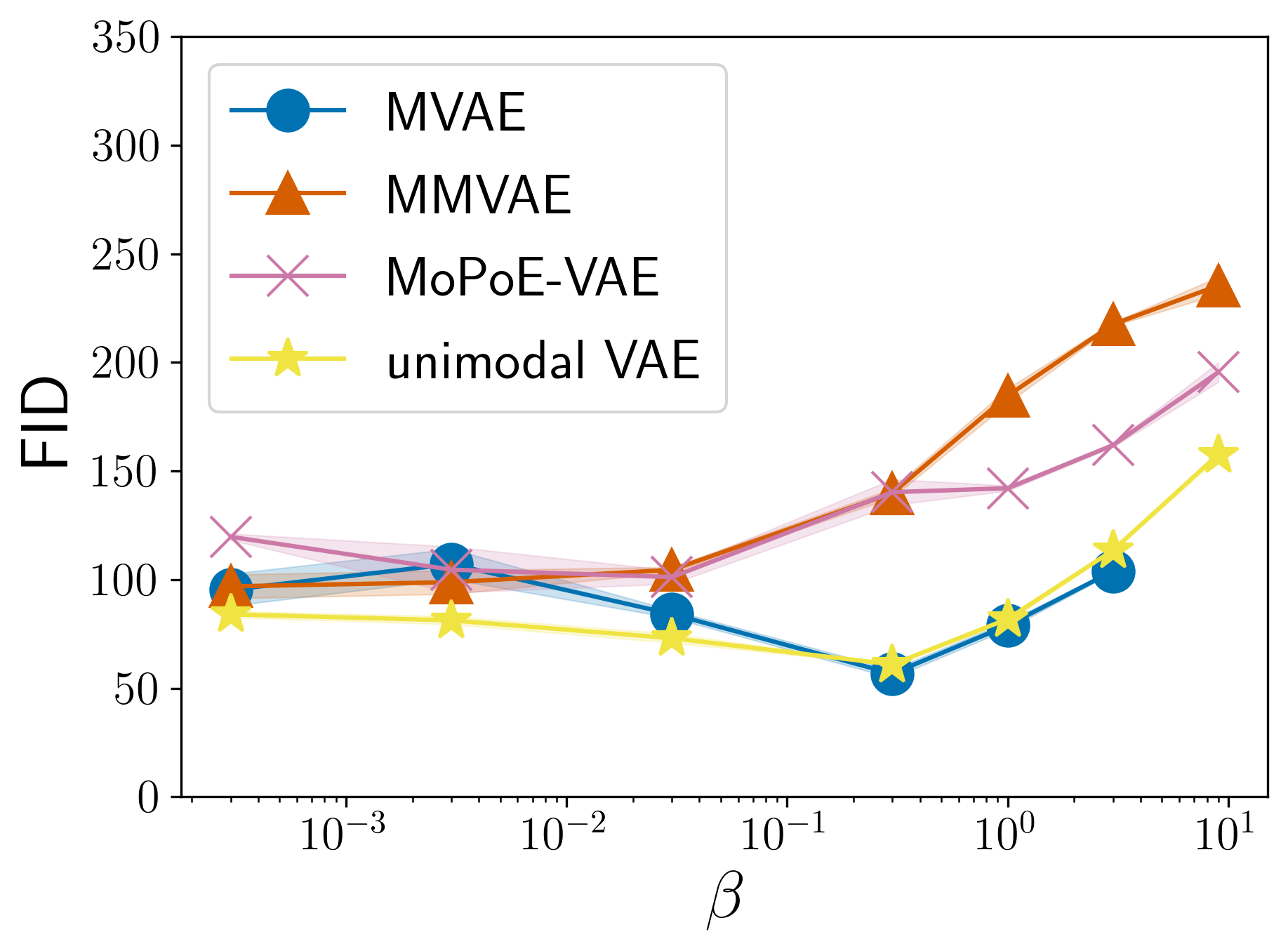}
\end{subfigure}
\begin{subfigure}[t]{.195\linewidth}
  \centering
  \includegraphics[width=1.0\linewidth]{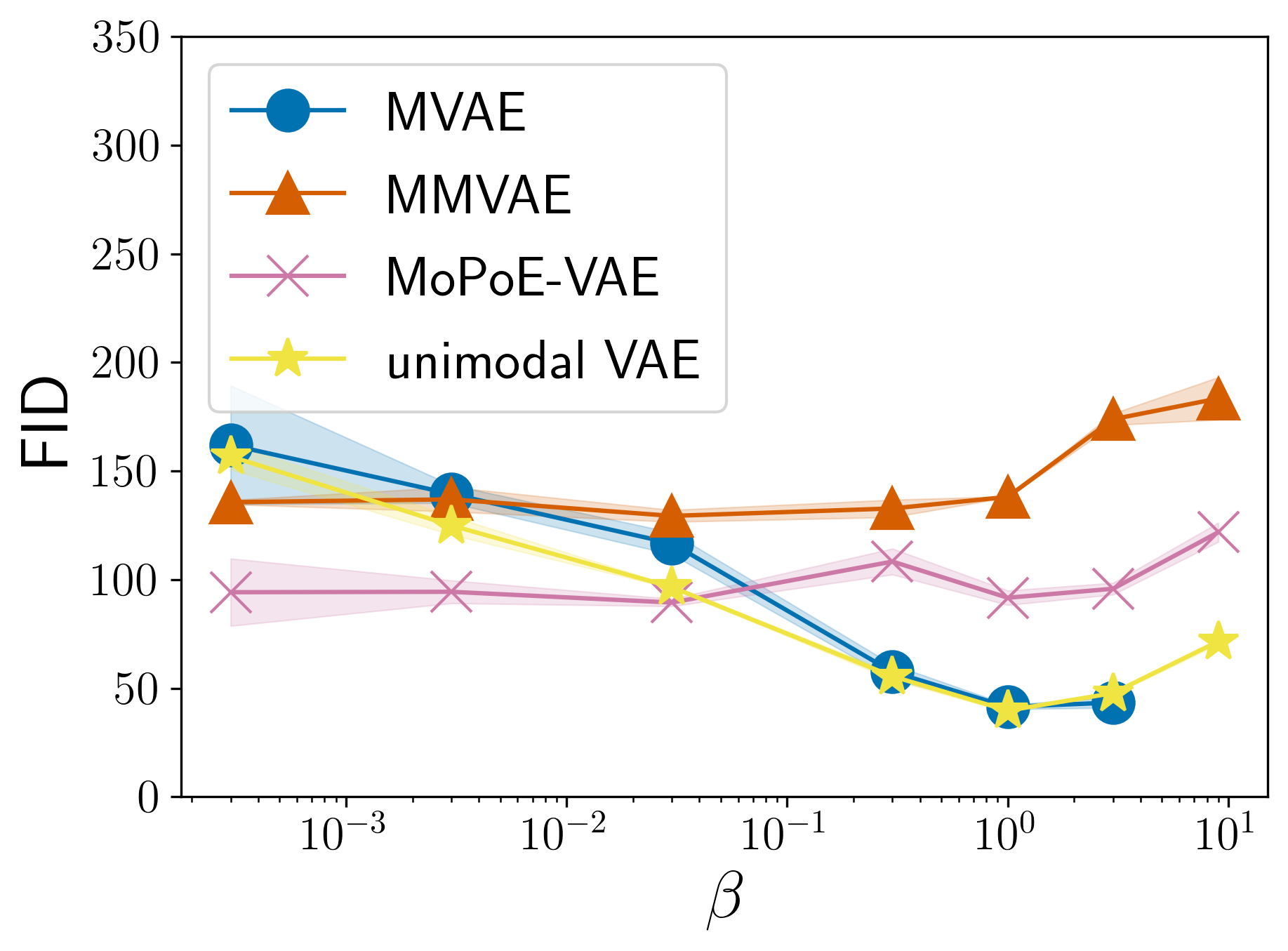}
\end{subfigure}
\begin{subfigure}[t]{.195\linewidth}
  \centering
  \includegraphics[width=1.0\linewidth]{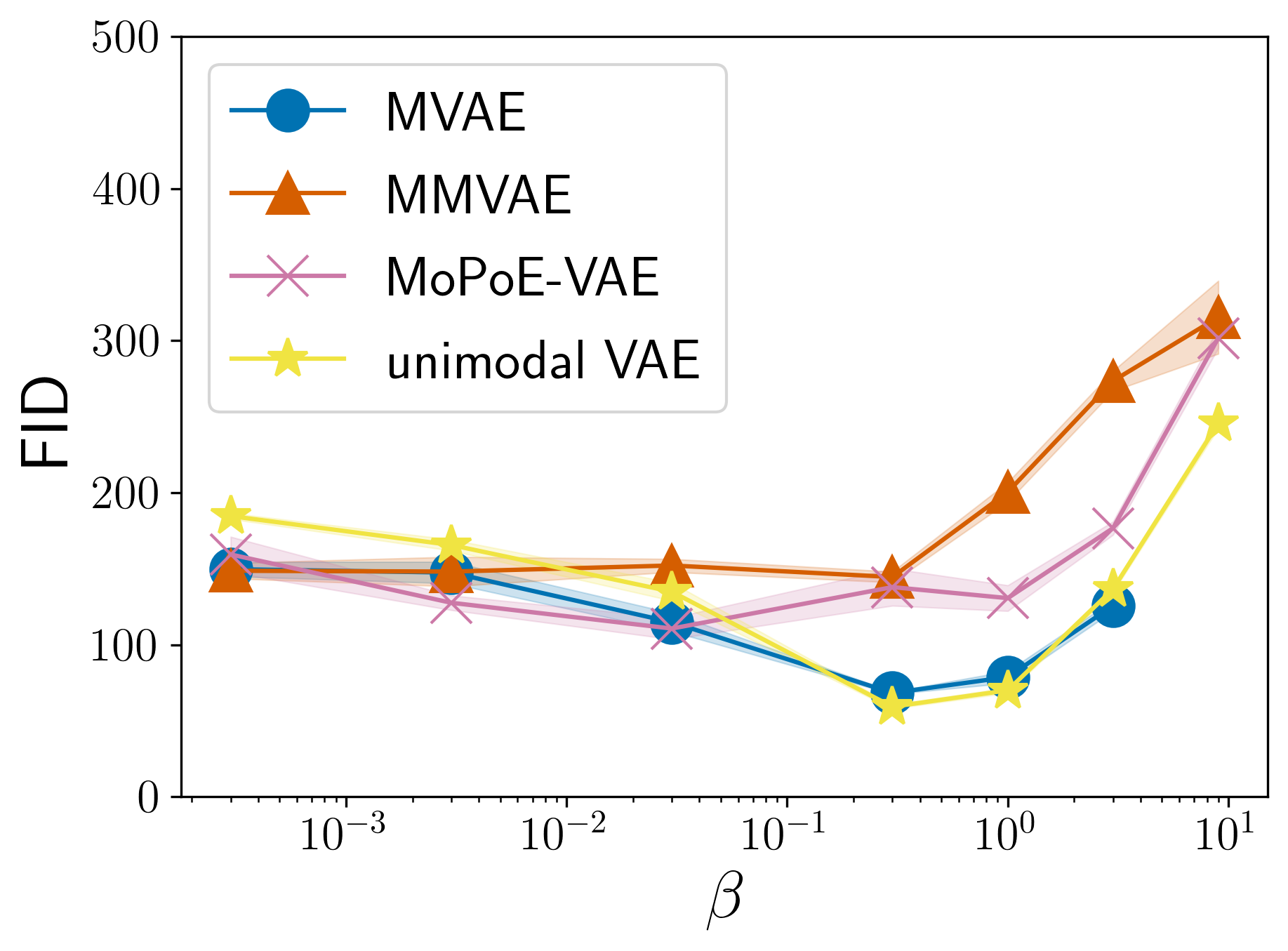}
  \caption{$X_1$}
\end{subfigure}%
\begin{subfigure}[t]{.195\linewidth}
  \centering
  \includegraphics[width=1.0\linewidth]{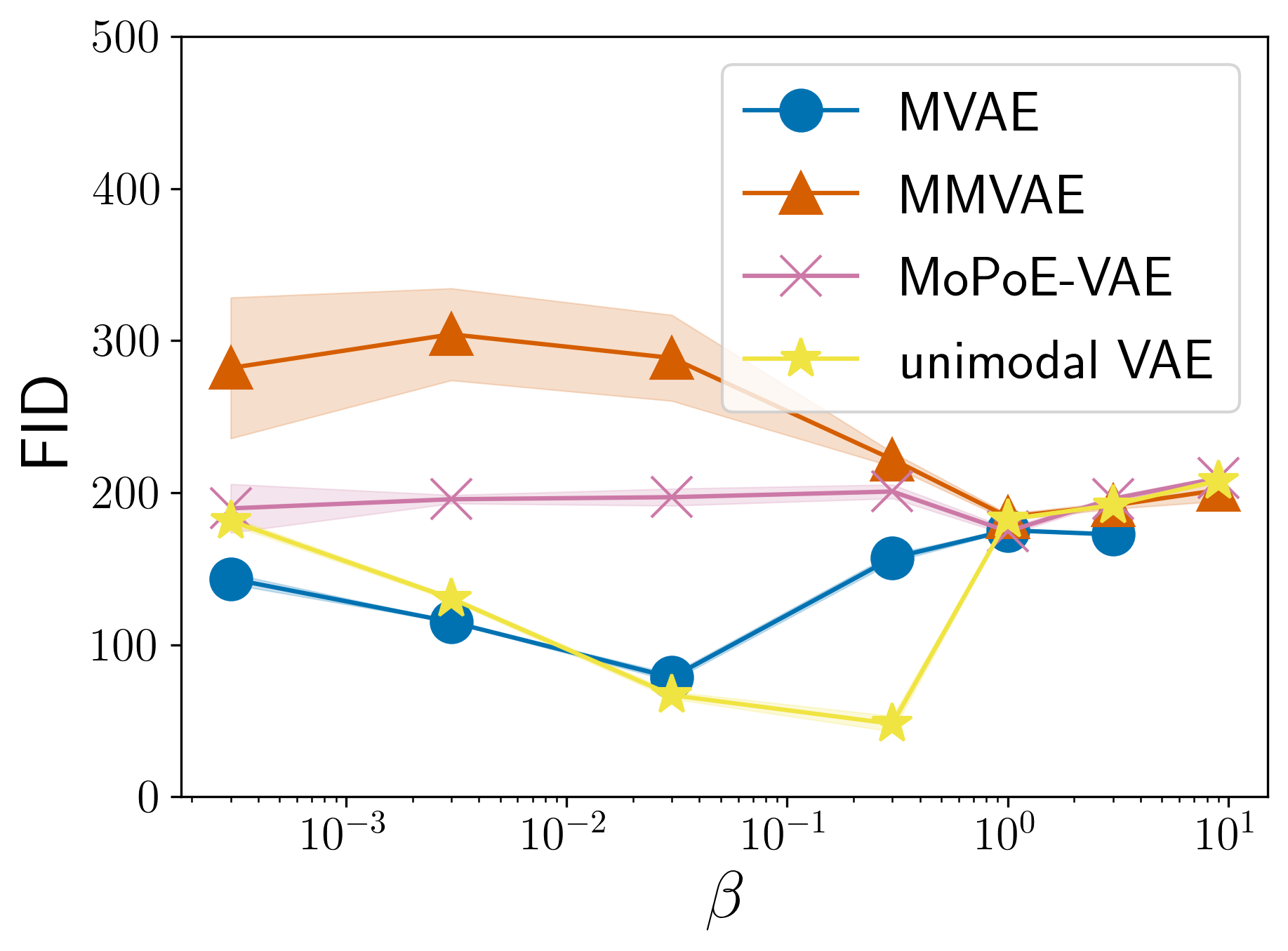}
  \caption{$X_2$}
\end{subfigure}%
\begin{subfigure}[t]{.195\linewidth}
  \centering
  \includegraphics[width=1.0\linewidth]{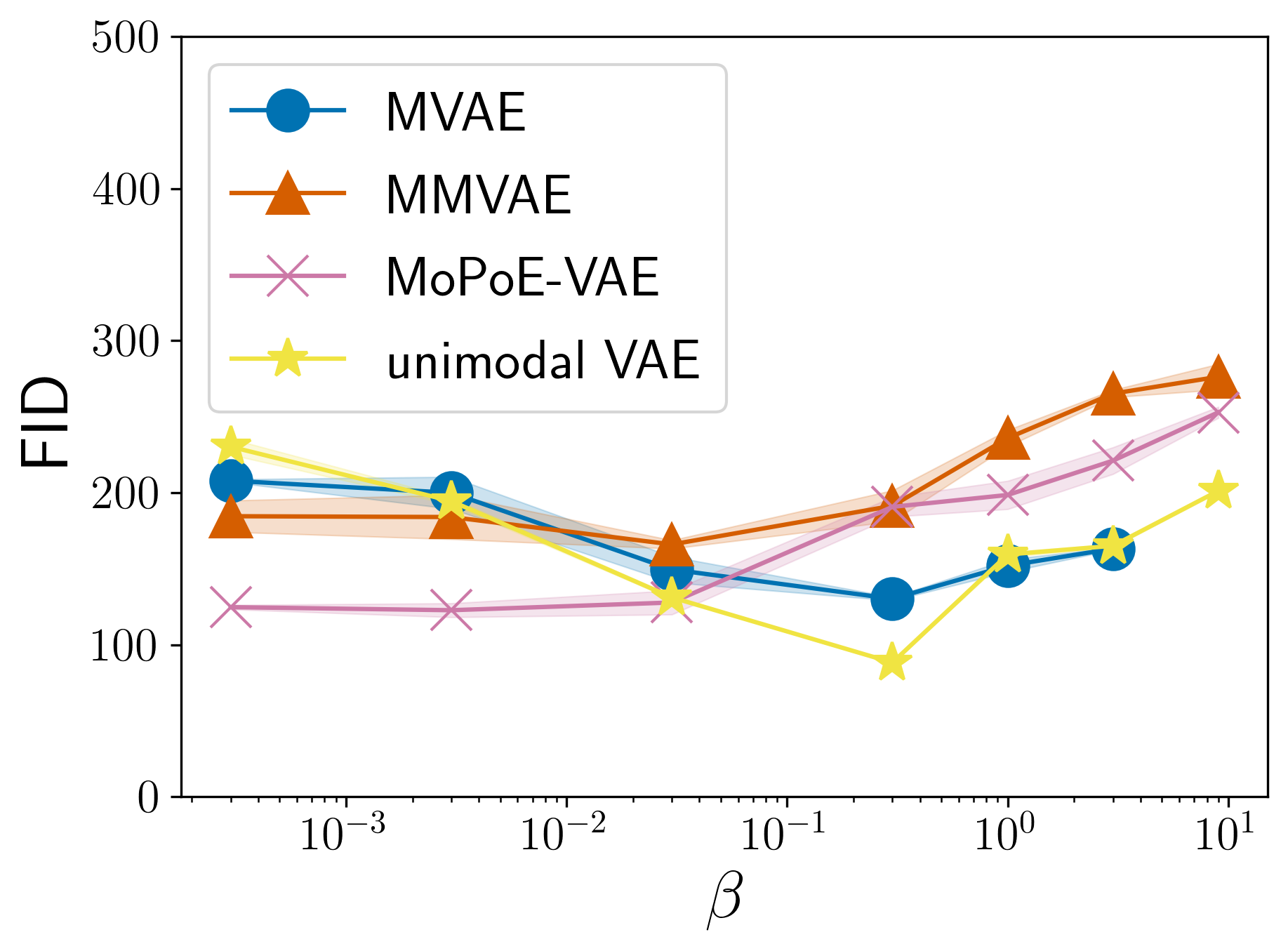}
  \caption{$X_3$}
\end{subfigure}
\begin{subfigure}[t]{.195\linewidth}
  \centering
  \includegraphics[width=1.0\linewidth]{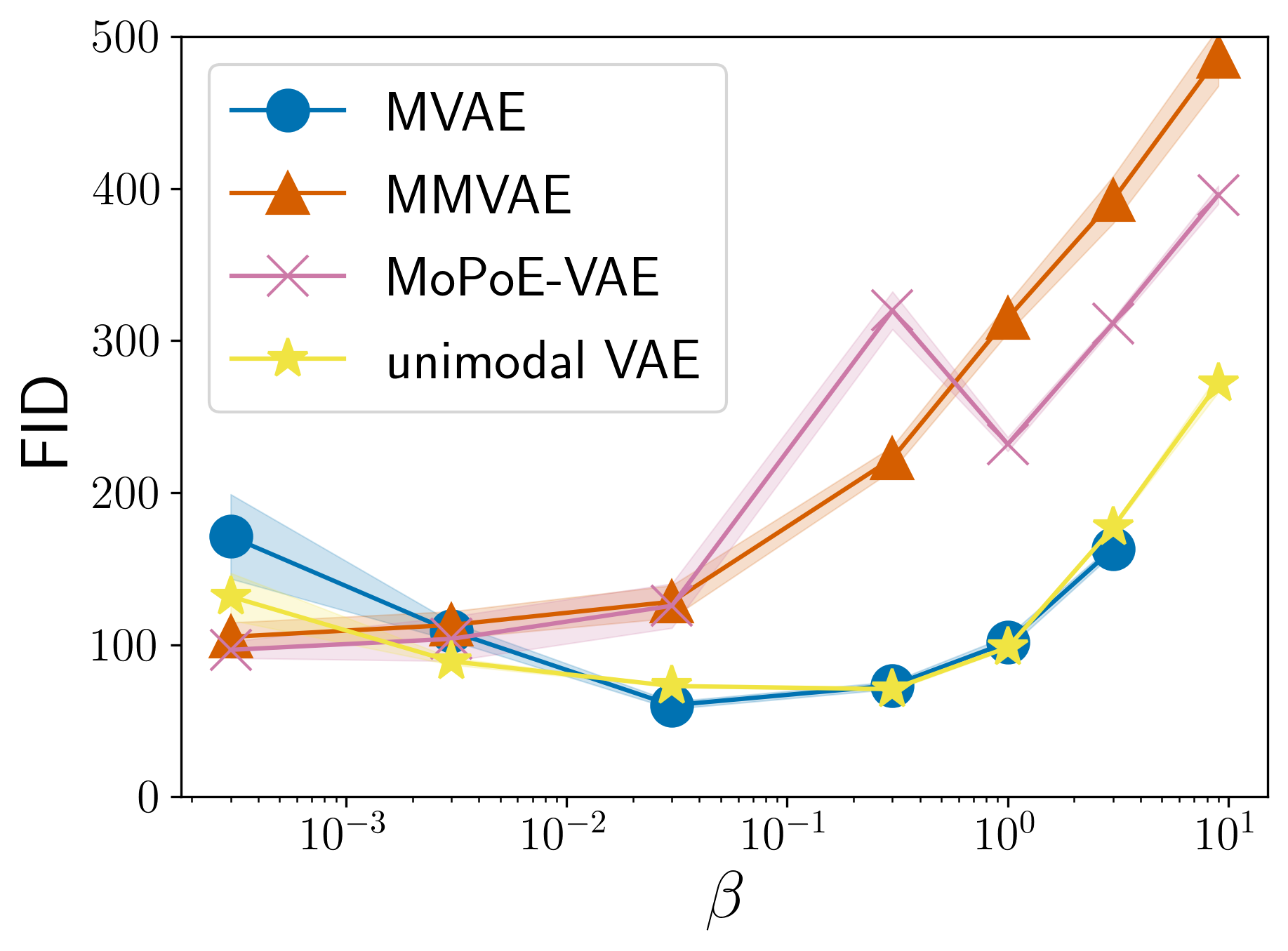}
  \caption{$X_4$}
\end{subfigure}
\begin{subfigure}[t]{.195\linewidth}
  \centering
  \includegraphics[width=1.0\linewidth]{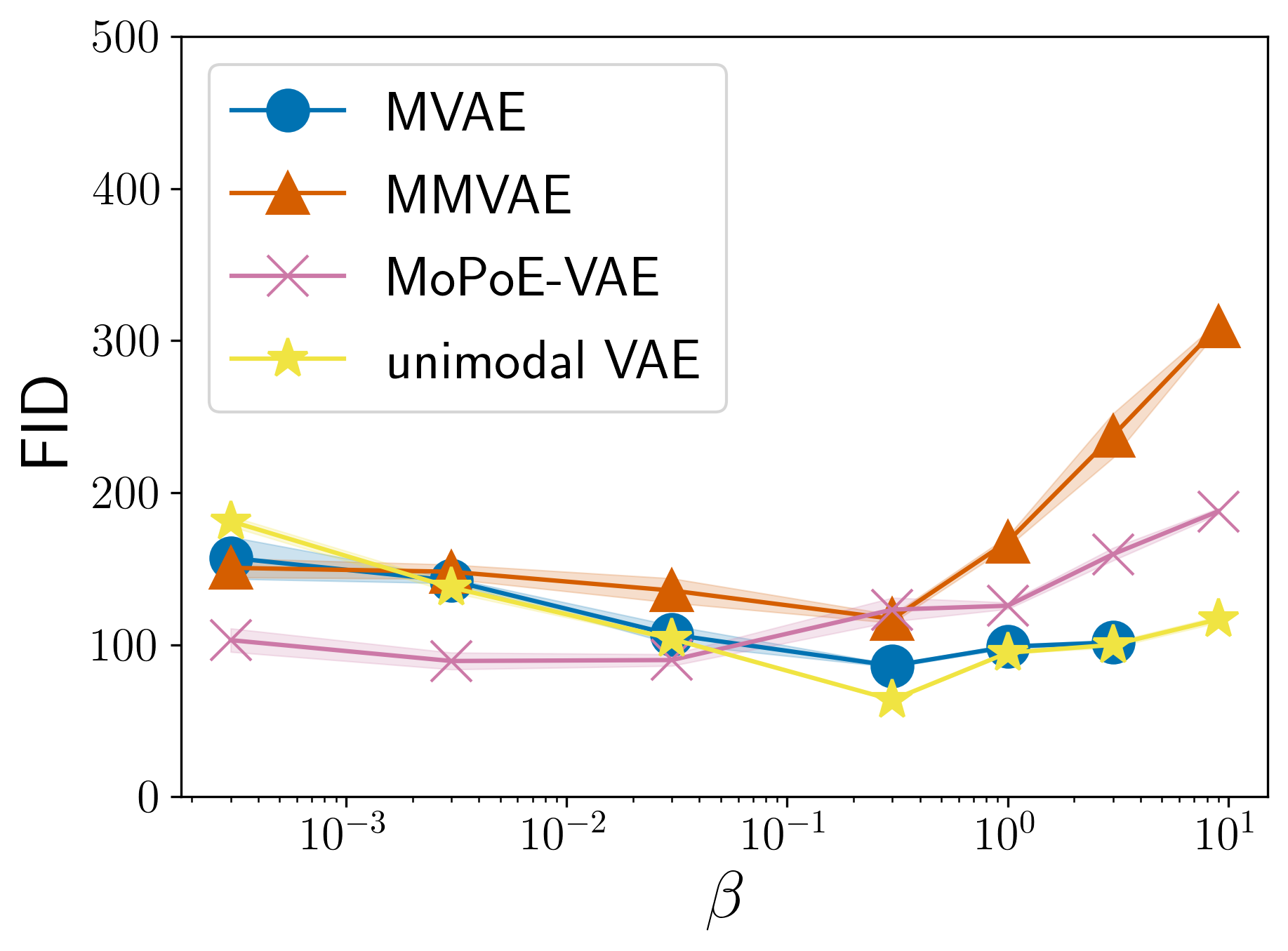}
  \caption{$X_5$}
\end{subfigure}
\caption{%
  FID for modalities $X_1, \ldots, X_5$.  The top row shows all FIDs for PolyMNIST
  and the bottom row for Translated-PolyMNIST respectively. Points denote the FID
  averaged over three seeds and bands show one standard deviation respectively.
  Due to numerical instabilities, the MVAE could not be trained with larger
  $\beta$ values.
}
\label{fig:beta_ablation_fids_in_all_directions}
\end{center}
\end{figure}

\begin{figure}[t]
\begin{center}
\captionsetup[subfigure]{font=scriptsize,labelfont=scriptsize}
\begin{subfigure}[t]{.24\linewidth}
  \centering
  \includegraphics[width=1.0\linewidth]{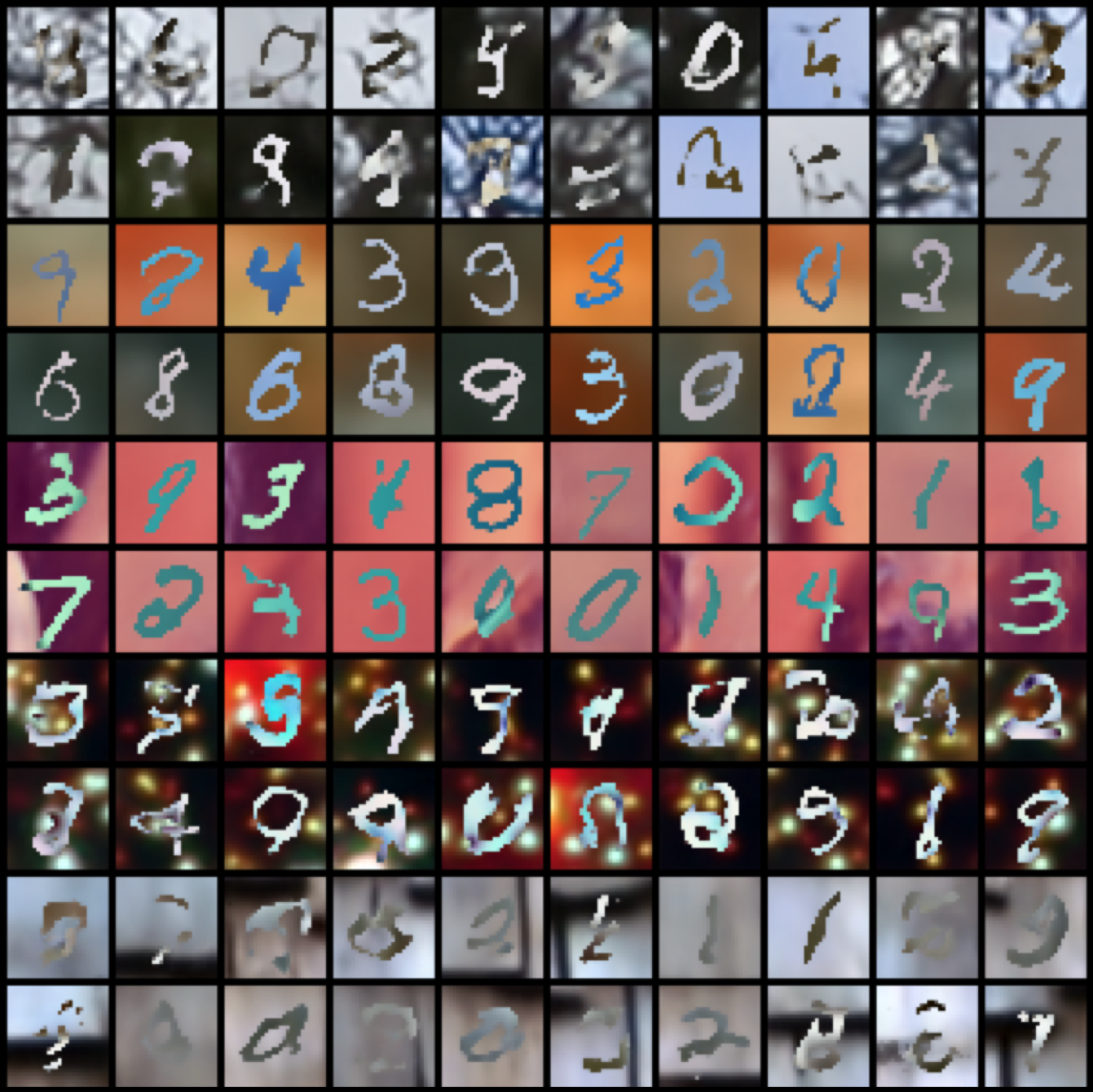}
  \caption{unimodal VAE, $\beta=1$}
\end{subfigure}
\begin{subfigure}[t]{.24\linewidth}
  \centering
  \includegraphics[width=1.0\linewidth]{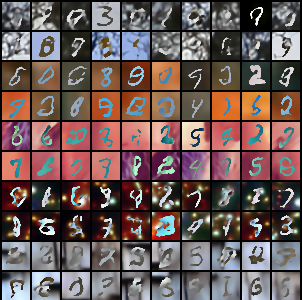}
  \caption{MVAE, $\beta=1$}
\end{subfigure}
\begin{subfigure}[t]{.24\linewidth}
  \centering
  \includegraphics[width=1.0\linewidth]{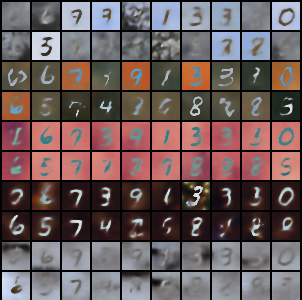}
  \caption{MMVAE, $\beta=1$}
\end{subfigure}
\begin{subfigure}[t]{.24\linewidth}
  \centering
  \includegraphics[width=1.0\linewidth]{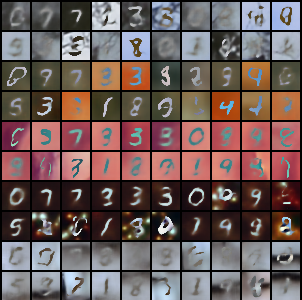}
  \caption{MoPoE-VAE, $\beta=1$}
\end{subfigure}
\vskip 0.30cm
\begin{subfigure}[t]{.24\linewidth}
  \centering
  \includegraphics[width=1.0\linewidth]{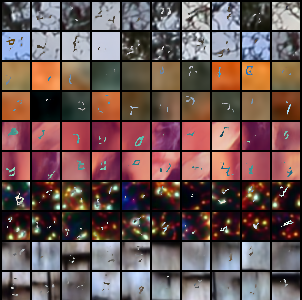}
  \caption{unimodal VAE, $\beta=0.3$}
\end{subfigure}
\begin{subfigure}[t]{.24\linewidth}
  \centering
  \includegraphics[width=1.0\linewidth]{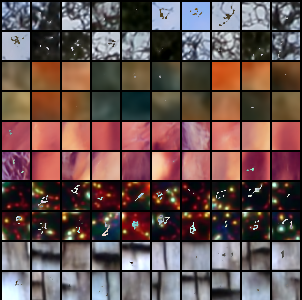}
  \caption{MVAE, $\beta=0.3$}
\end{subfigure}
\begin{subfigure}[t]{.24\linewidth}
  \centering
  \includegraphics[width=1.0\linewidth]{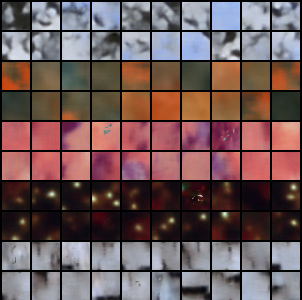}
  \caption{MMVAE, $\beta=0.3$}
\end{subfigure}
\begin{subfigure}[t]{.24\linewidth}
  \centering
  \includegraphics[width=1.0\linewidth]{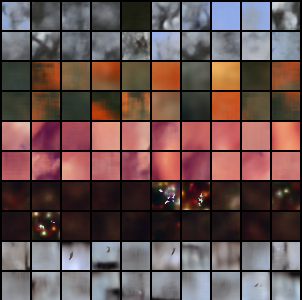}
  \caption{MoPoE-VAE, $\beta=0.3$}
\end{subfigure}
\vskip 0.30cm
\begin{subfigure}[t]{.24\linewidth}
  \centering
  \includegraphics[width=1.0\linewidth]{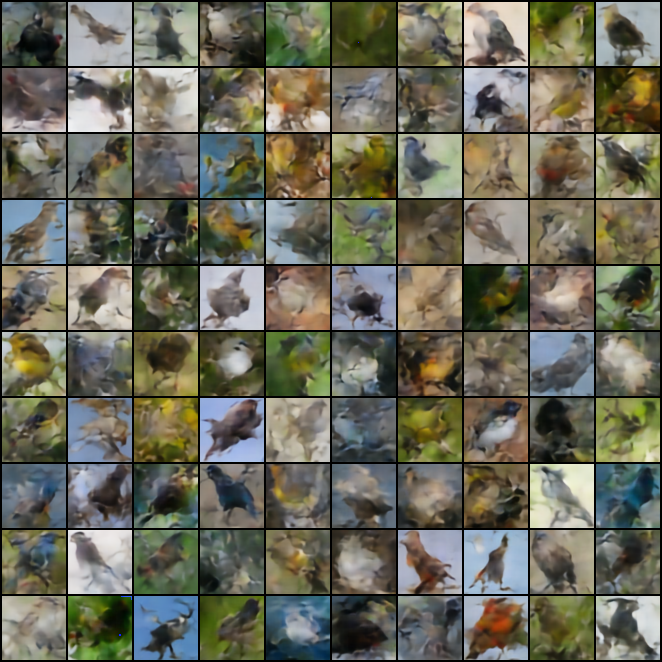}
  \caption{unimodal VAE, $\beta=9$}
\end{subfigure}
\begin{subfigure}[t]{.24\linewidth}
  \centering
  \includegraphics[width=1.0\linewidth]{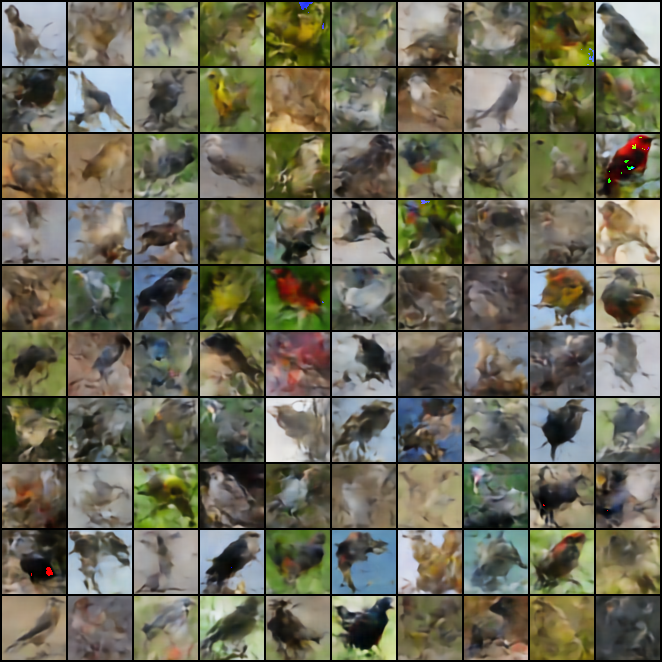}
  \caption{MVAE, $\beta=9$}
\end{subfigure}
\begin{subfigure}[t]{.24\linewidth}
  \centering
  \includegraphics[width=1.0\linewidth]{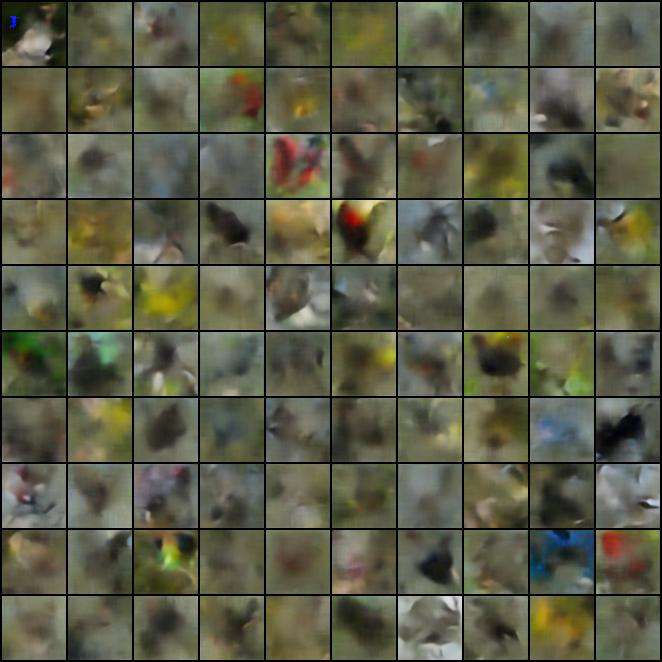}
  \caption{MMVAE, $\beta=9$}
\end{subfigure}
\begin{subfigure}[t]{.24\linewidth}
  \centering
  \includegraphics[width=1.0\linewidth]{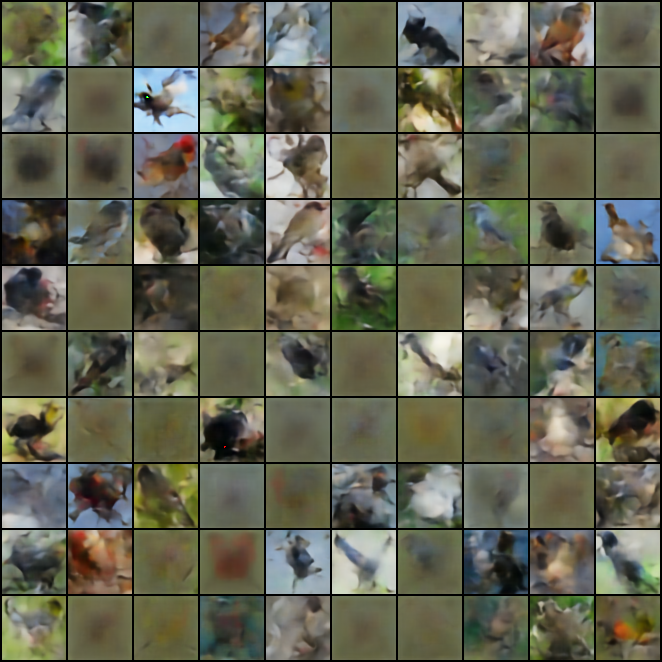}
  \caption{MoPoE-VAE, $\beta=9$}
\end{subfigure}
\vskip 0.30cm
\begin{subfigure}[t]{.24\linewidth}
  \centering
  \includegraphics[width=1.0\linewidth]{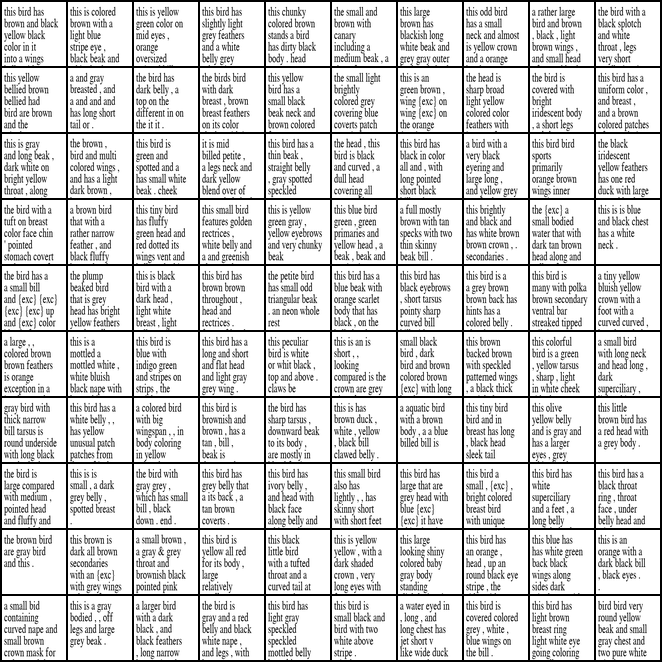}
  \caption{unimodal VAE, $\beta=9$}
\end{subfigure}
\begin{subfigure}[t]{.24\linewidth}
  \centering
  \includegraphics[width=1.0\linewidth]{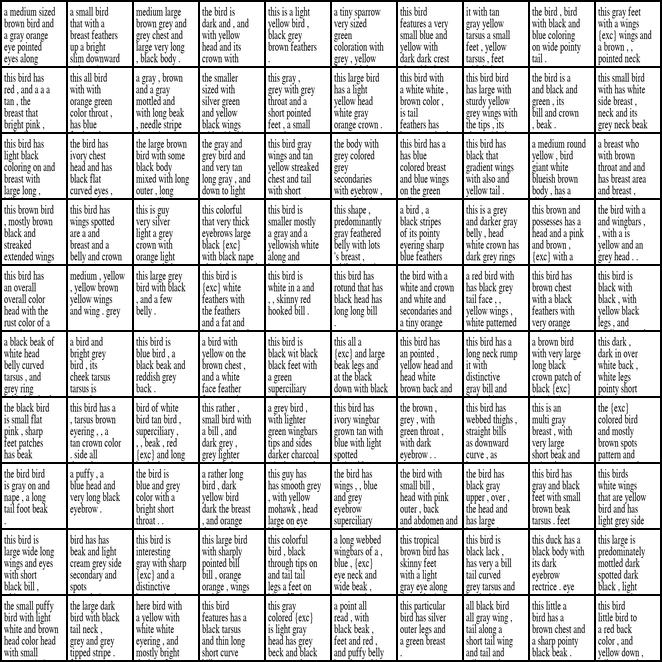}
  \caption{MVAE, $\beta=9$}
\end{subfigure}
\begin{subfigure}[t]{.24\linewidth}
  \centering
  \includegraphics[width=1.0\linewidth]{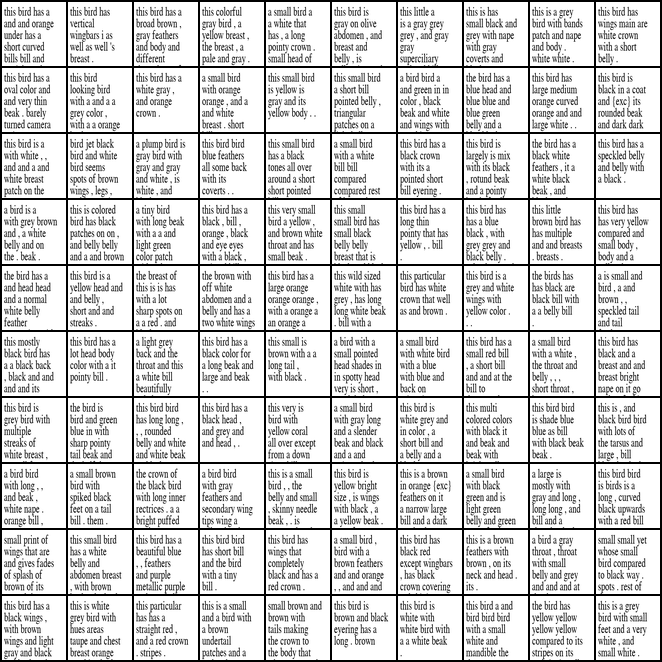}
  \caption{MMVAE, $\beta=9$}
\end{subfigure}
\begin{subfigure}[t]{.24\linewidth}
  \centering
  \includegraphics[width=1.0\linewidth]{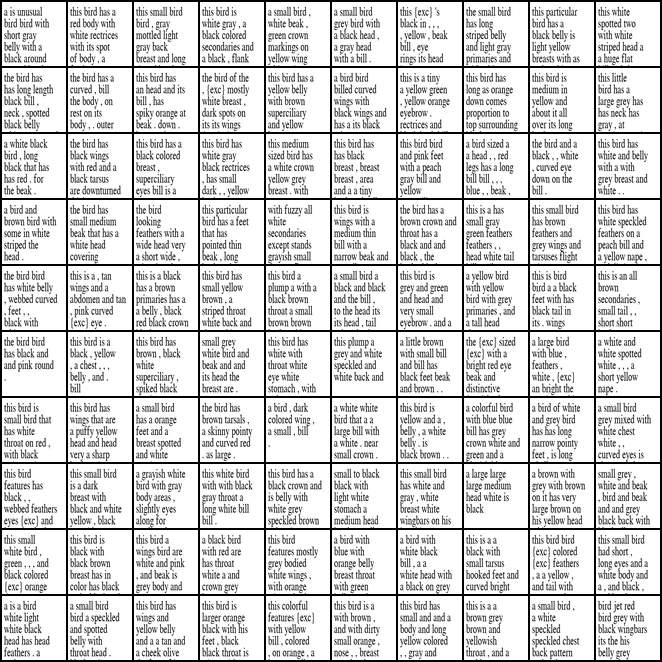}
  \caption{MoPoE-VAE, $\beta=9$}
\end{subfigure}
\vskip 0.20cm
\caption{%
  Qualitative results for the unconditional generation using prior samples. For
  PolyMNIST (Subfigures (a) to (d)) and Translated-PolyMNIST (Subfigures (e) to
  (h)), we show 20 samples for each modality. For CUB, we show 100 generated
  images (Subfigures (i) to (l)) and 100 generated captions (Subfigures (m) to
  (p)) respectively. Best viewed zoomed and in color.
}
\label{fig:qualitative_unconditional}
\end{center}
\end{figure}

\begin{figure}[t]
\captionsetup[subfigure]{font=scriptsize,labelfont=scriptsize}
\begin{center}
\begin{subfigure}[t]{.27\linewidth}
  \centering
  \includegraphics[width=1.0\linewidth]{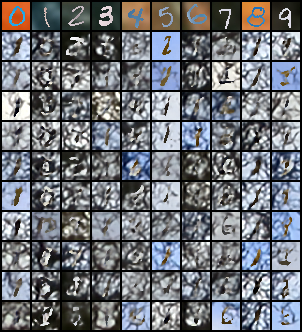}
  \caption{MVAE, $\beta=1$}
\end{subfigure}
\begin{subfigure}[t]{.27\linewidth}
  \centering
  \includegraphics[width=1.0\linewidth]{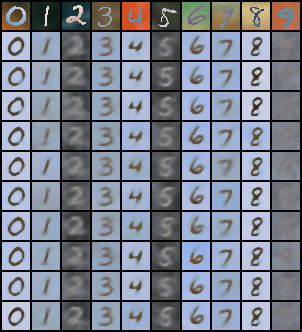}
  \caption{MMVAE, $\beta=1$}
\end{subfigure}
\begin{subfigure}[t]{.27\linewidth}
  \centering
  \includegraphics[width=1.0\linewidth]{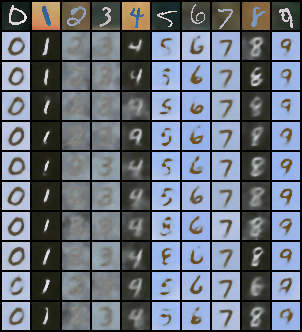}
  \caption{MoPoE-VAE, $\beta=1$}
\end{subfigure}
\vskip 0.30cm
\begin{subfigure}[t]{.27\linewidth}
  \centering
  \includegraphics[width=1.0\linewidth]{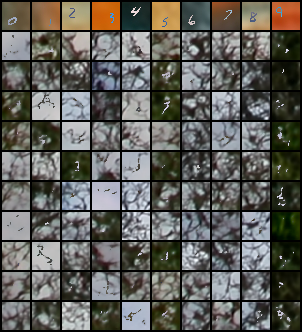}
  \caption{MVAE, $\beta=0.3$}
\end{subfigure}
\begin{subfigure}[t]{.27\linewidth}
  \centering
  \includegraphics[width=1.0\linewidth]{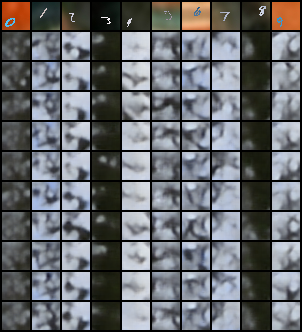}
  \caption{MMVAE, $\beta=0.3$}
\end{subfigure}
\begin{subfigure}[t]{.27\linewidth}
  \centering
  \includegraphics[width=1.0\linewidth]{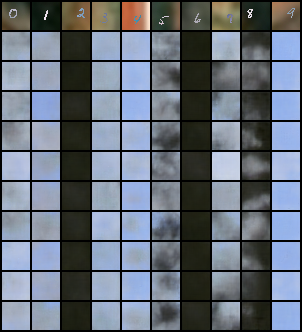}
  \caption{MoPoE-VAE, $\beta=0.3$}
\end{subfigure}
\vskip 0.30cm
\begin{subfigure}[t]{.27\linewidth}
  \centering
  \includegraphics[width=1.0\linewidth]{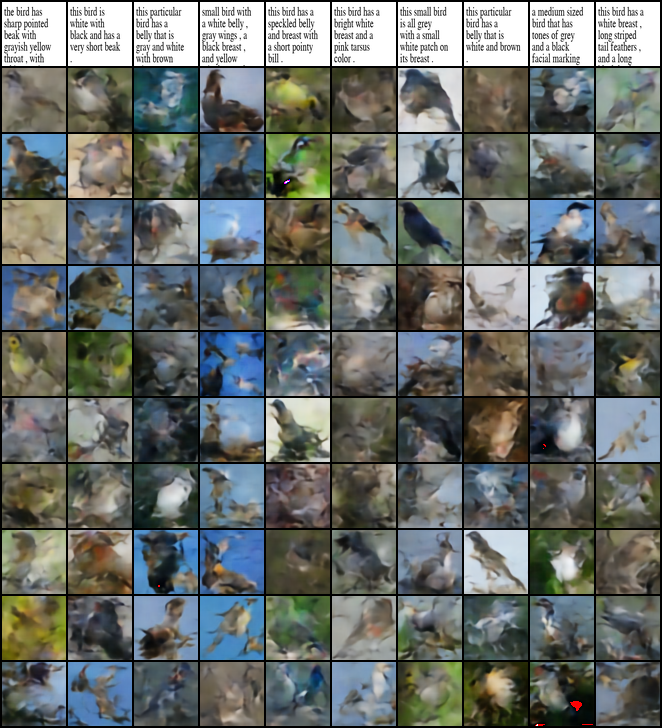}
  \caption{MVAE, $\beta=9.0$}
\end{subfigure}
\begin{subfigure}[t]{.27\linewidth}
  \centering
  \includegraphics[width=1.0\linewidth]{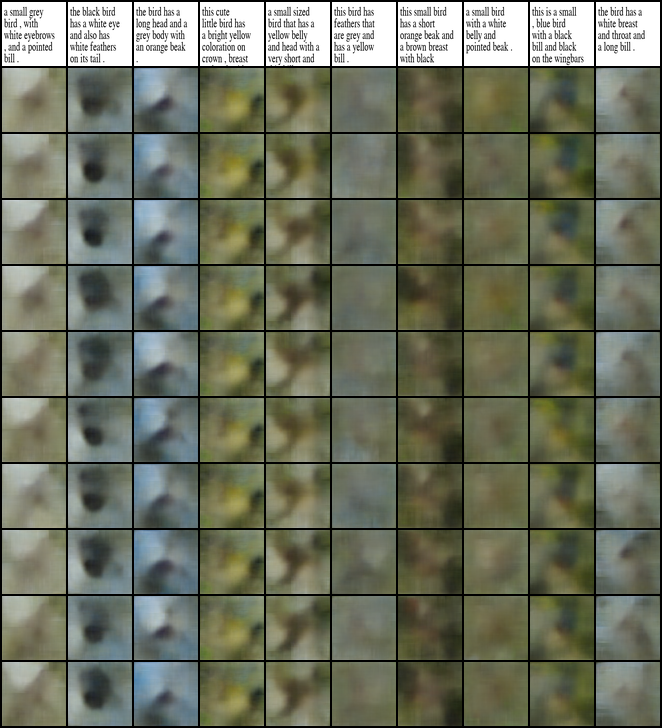}
  \caption{MMVAE, $\beta=9.0$}
\end{subfigure}
\begin{subfigure}[t]{.27\linewidth}
  \centering
  \includegraphics[width=1.0\linewidth]{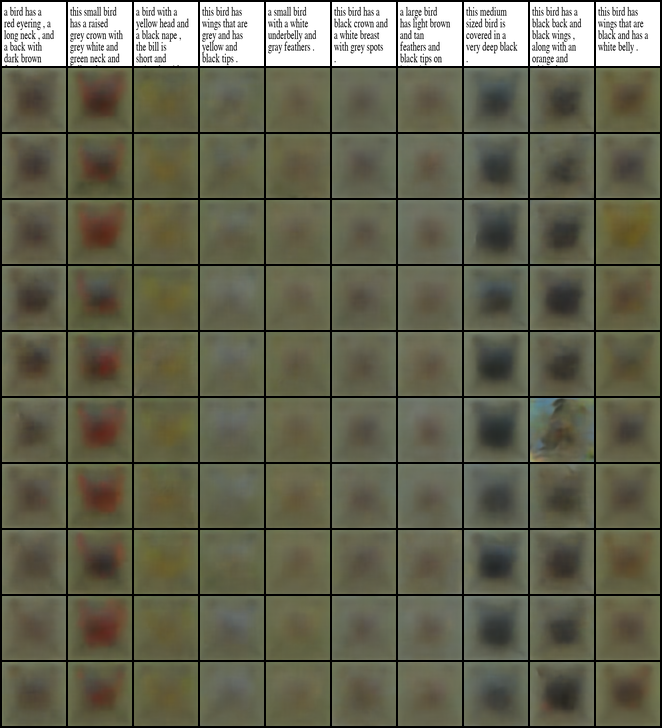}
  \caption{MoPoE-VAE, $\beta=9.0$}
\end{subfigure}
\vskip 0.30cm
\begin{subfigure}[t]{.27\linewidth}
  \centering
  \includegraphics[width=1.0\linewidth]{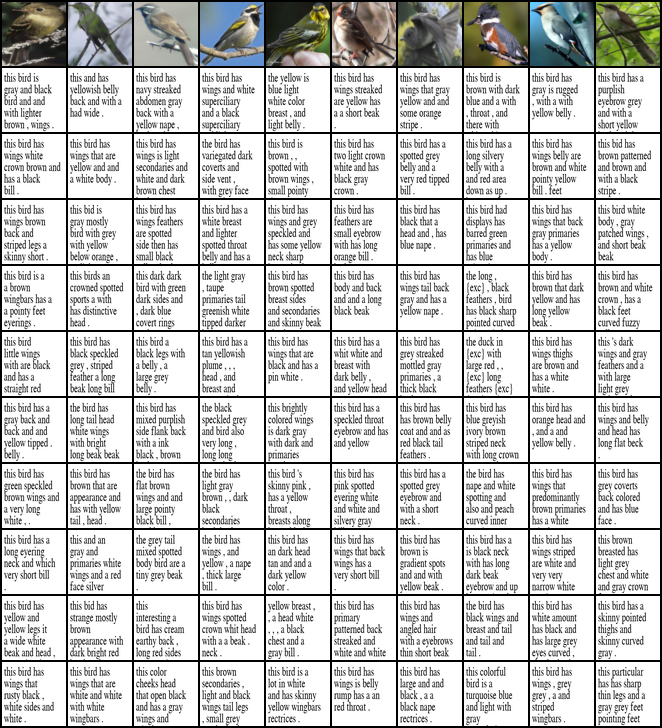}
  \caption{MVAE, $\beta=9.0$}
\end{subfigure}
\begin{subfigure}[t]{.27\linewidth}
  \centering
  \includegraphics[width=1.0\linewidth]{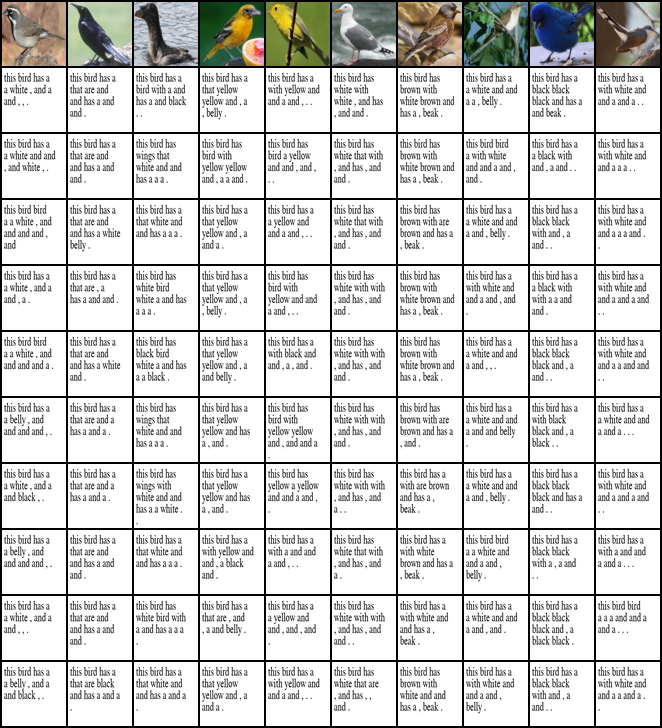}
  \caption{MMVAE, $\beta=9.0$}
\end{subfigure}
\begin{subfigure}[t]{.27\linewidth}
  \centering
  \includegraphics[width=1.0\linewidth]{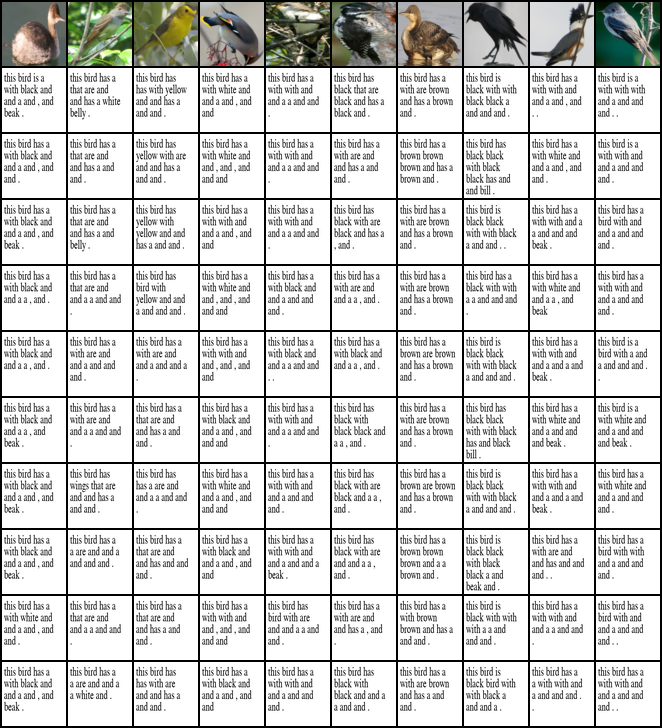}
  \caption{MoPoE-VAE, $\beta=9.0$}
\end{subfigure}
\vskip 0.20cm
\caption{%
  Qualitative results for the conditional generation across modalities. For
  PolyMNIST (Subfigures (a) to (c)) and Translated-PolyMNIST (Subfigures (d) to
  (f)), we show 10 conditionally generated samples of modality $X_1$ given the
  sample from modality $X_2$ that is shown in the first row of the respective
  subfigure.  For CUB, we show the generation of images given captions
  (Subfigures (g) to (i)), as well as the generation of captions given images
  (Subfigures (j) to (l)).  Best viewed zoomed and in color.
}
\label{fig:qualitative_conditional}
\end{center}
\end{figure}

\begin{figure}[t]
\begin{center}
\begin{subfigure}[t]{.45\linewidth}
  \centering
  \includegraphics[width=1.0\linewidth]{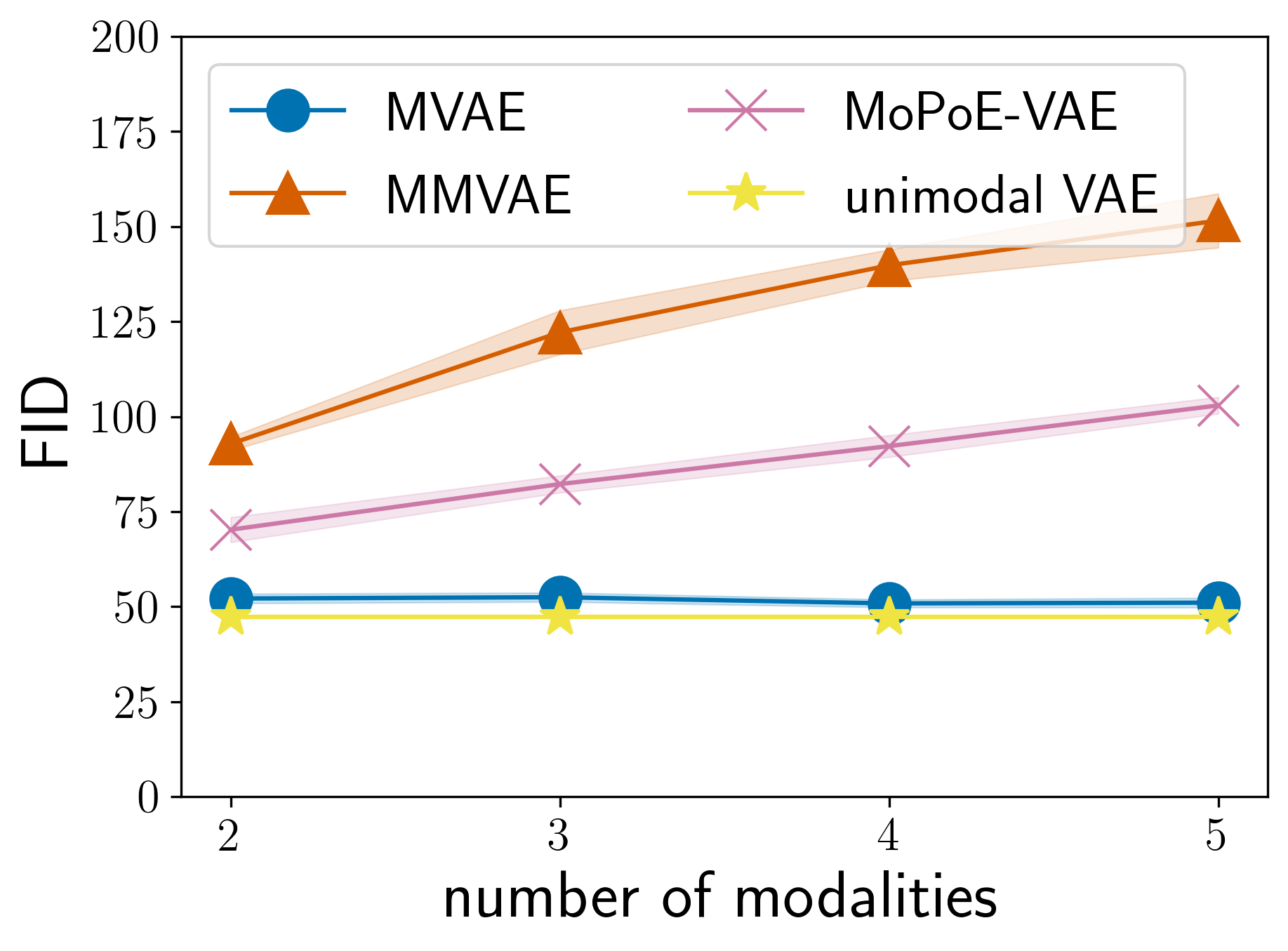}
  \caption{PolyMNIST}
\end{subfigure}
\hskip +0.135in
\begin{subfigure}[t]{.45\linewidth}
  \centering
  \includegraphics[width=1.0\linewidth]{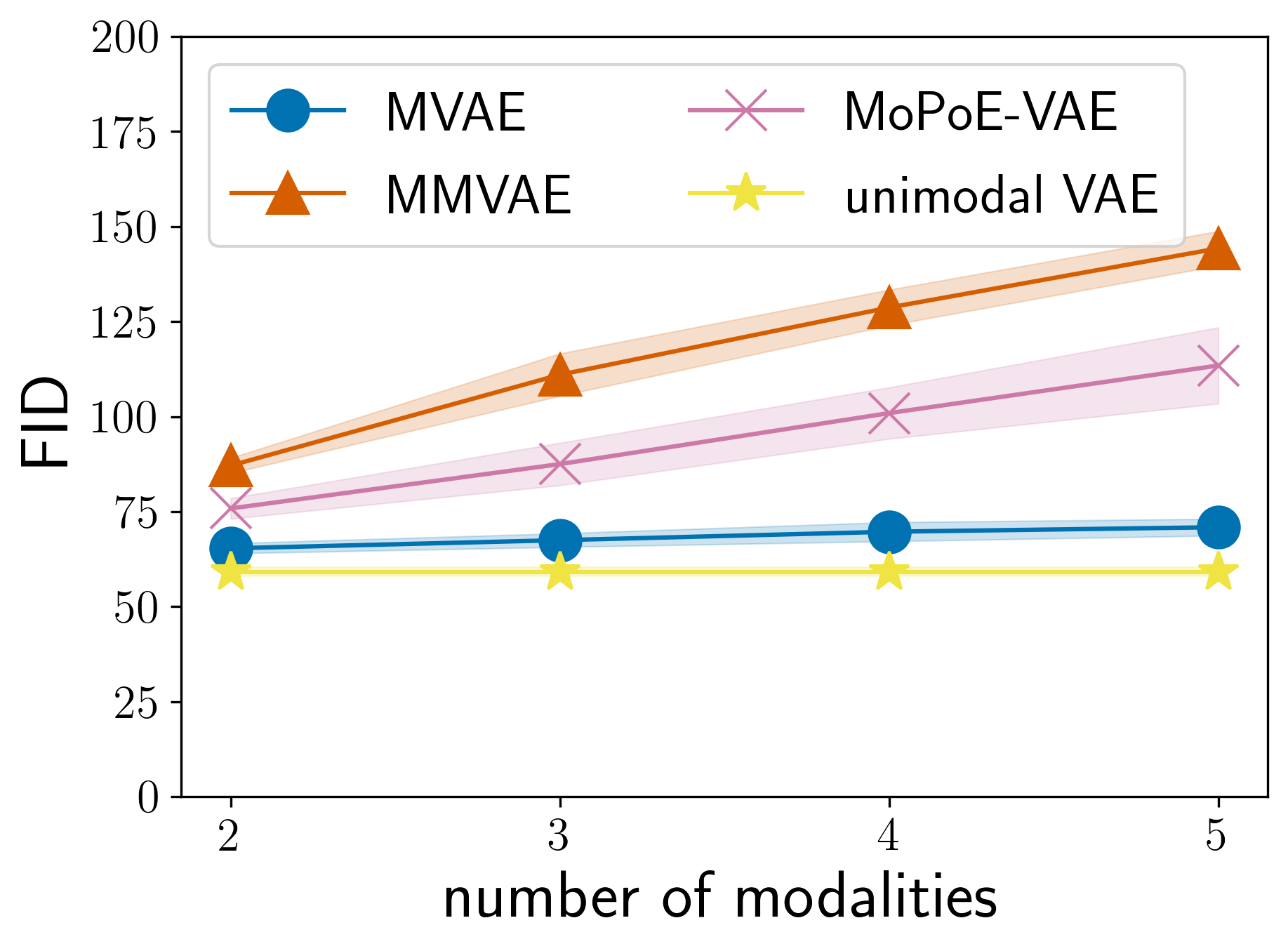}
  \caption{Translated-PolyMNIST}
\end{subfigure}%
\caption{%
  Generative quality as a function of the number of modalities.  In contrast to
  \Cref{fig:num_mod_ablation}, here we ``repeat'' the same modality, to verify
  that the generative quality also declines when the modality-specific
  variation of all modalities is similar.  All models are trained with
  $\beta=1$ on PolyMNIST and $\beta=0.3$ on Translated-PolyMNIST\@. The results
  are averaged over three seeds and all modalities; the bands show one standard
  deviation respectively.  For the unimodal VAE, which uses only a single
  modality, the average and standard deviation are plotted as a constant.
}%
\label{fig:num_mod_ablation_repeated_modality}
\end{center}
\end{figure}

\begin{figure}[t]
\begin{center}
\begin{subfigure}[t]{.45\linewidth}
  \centering
  \includegraphics[width=1.0\linewidth]{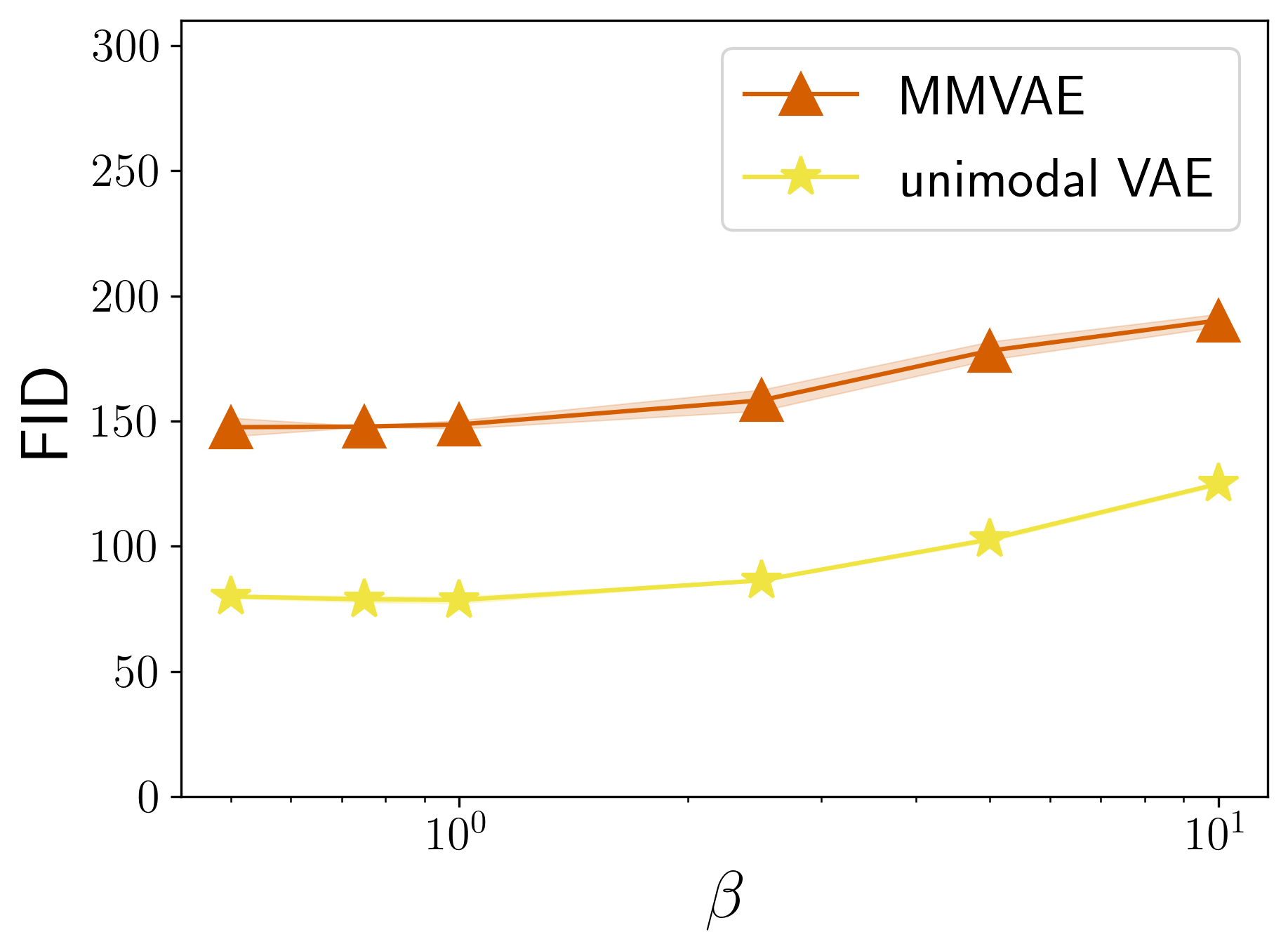}
  \caption{FID}
\end{subfigure}
\hskip +0.075in
\begin{subfigure}[t]{.46\linewidth}
  \centering
  \includegraphics[width=1.0\linewidth]{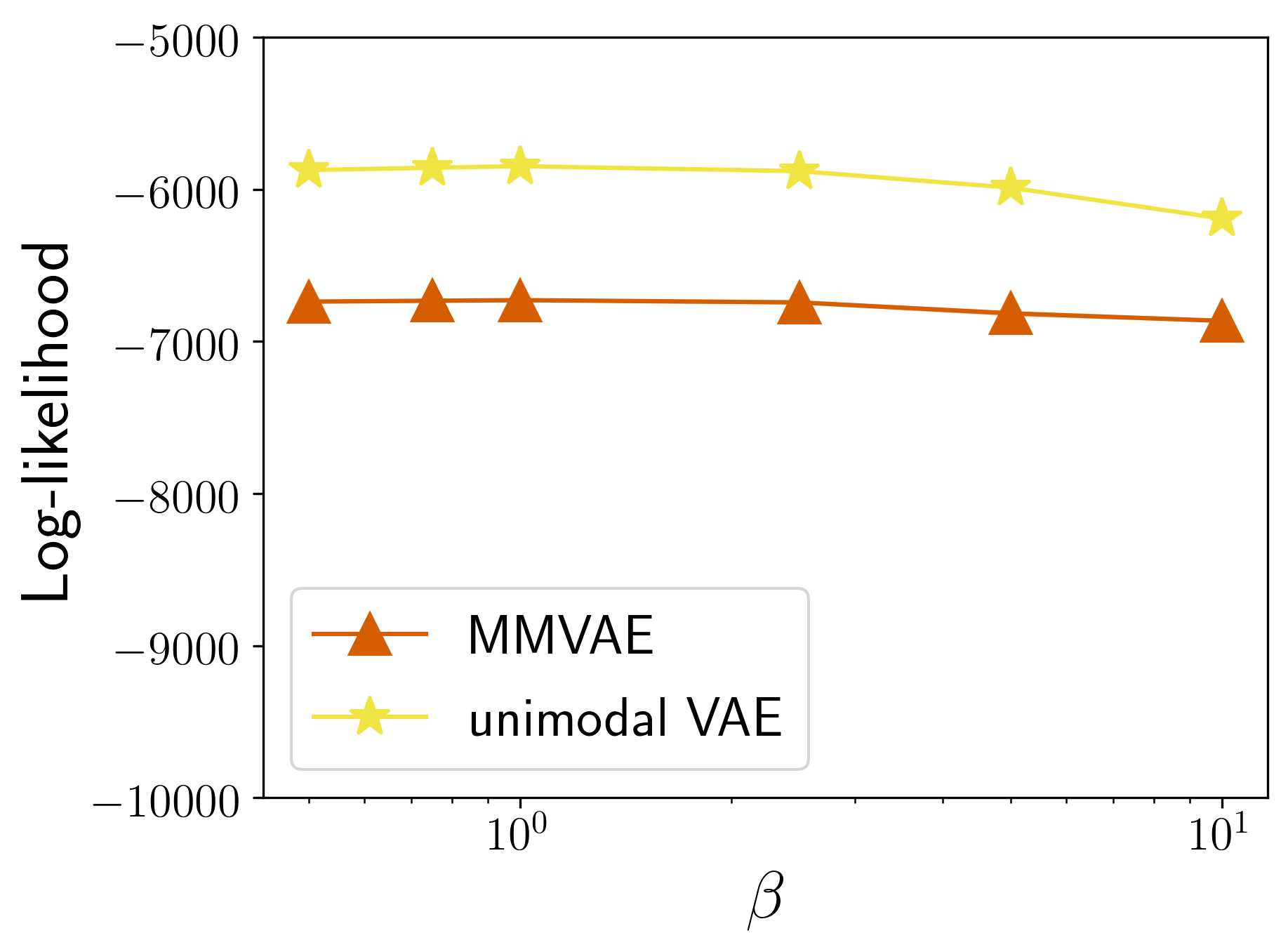}
  \caption{Joint log-likelihood}
\end{subfigure}
\caption{%
  PolyMNIST $\beta$-ablation using the official implementation of the MMVAE\@.
  In particular, for both the MMVAE and the unimodal VAE, we use the DReG
  objective, importance sampling, as well as a learned prior.  Points denote
  the value of the respective metric averaged over three seeds and bands show
  one standard deviation respectively.
}
\label{fig:vanilla_polymnist_beta_ablation_mmvae_official}
\end{center}
\end{figure}

\begin{figure}[t]
\begin{center}
\begin{subfigure}[t]{.32\linewidth}
  \centering
  \includegraphics[width=1.0\linewidth]{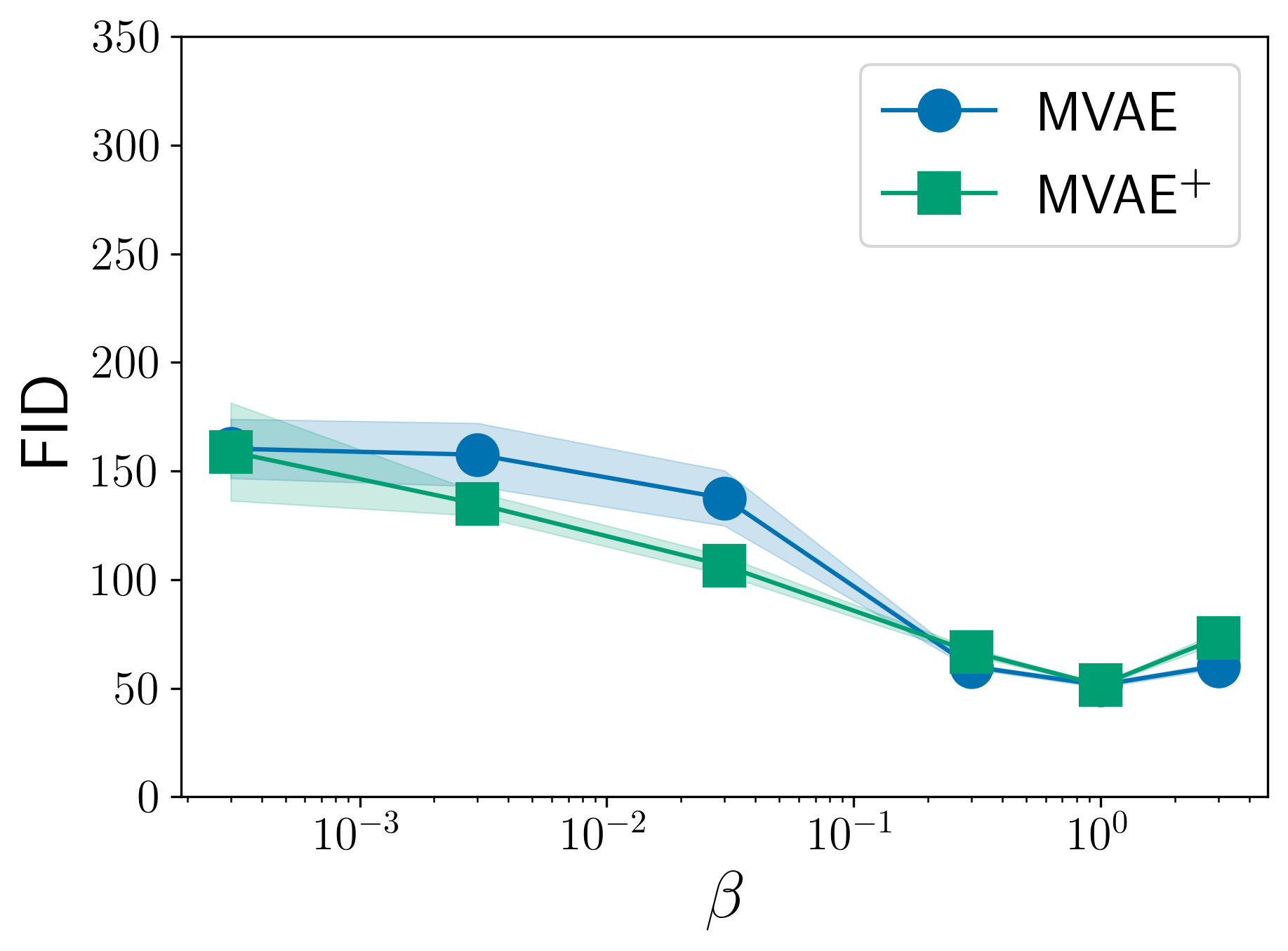}
  \caption{FID}
\end{subfigure}
\begin{subfigure}[t]{.32\linewidth}
  \centering
  \includegraphics[width=0.98\linewidth]{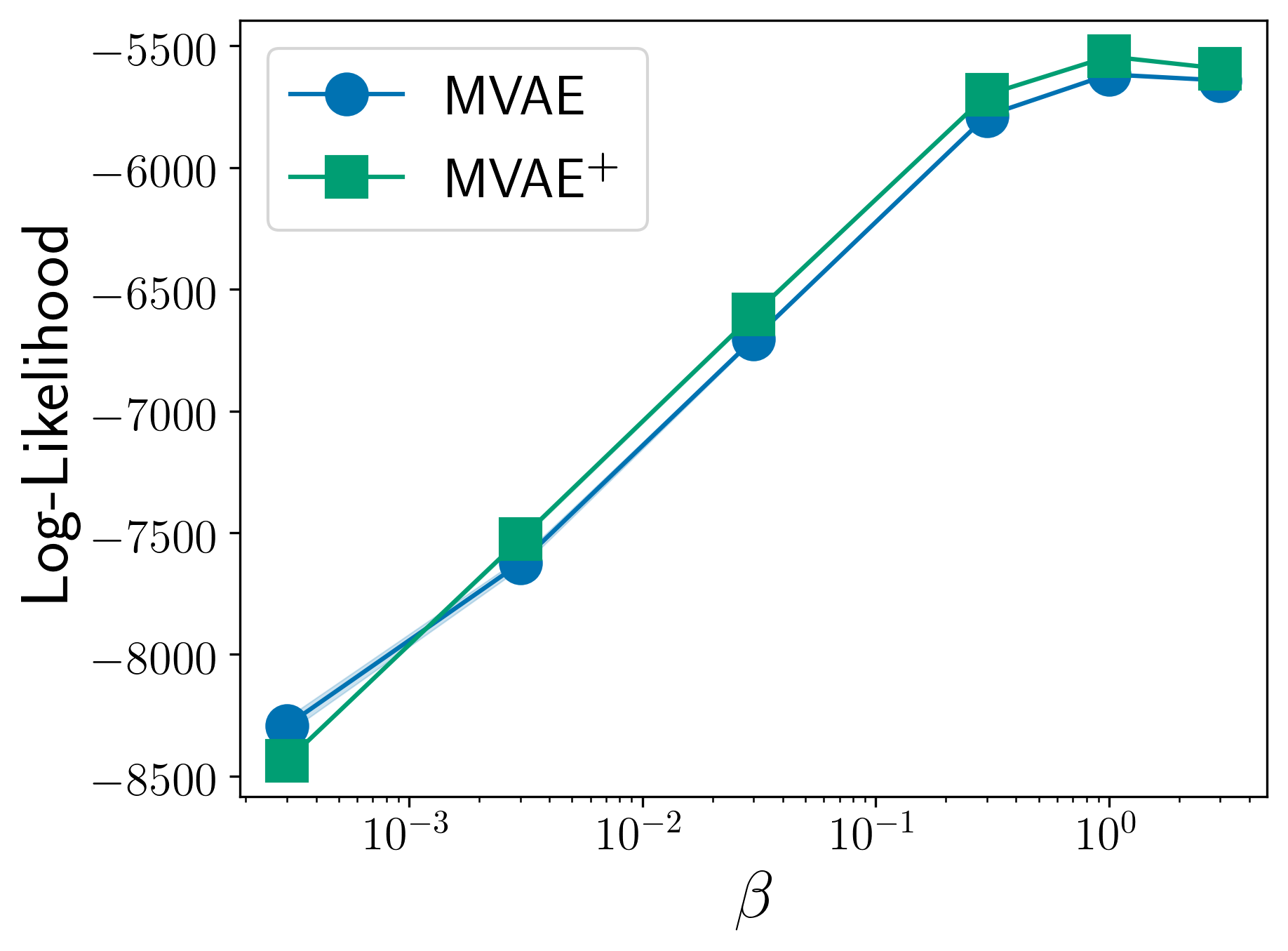}
  \caption{Joint log-likelihood}
\end{subfigure}%
\hskip +0.035in
\begin{subfigure}[t]{.32\linewidth}
  \centering
  \includegraphics[width=1.0\linewidth]{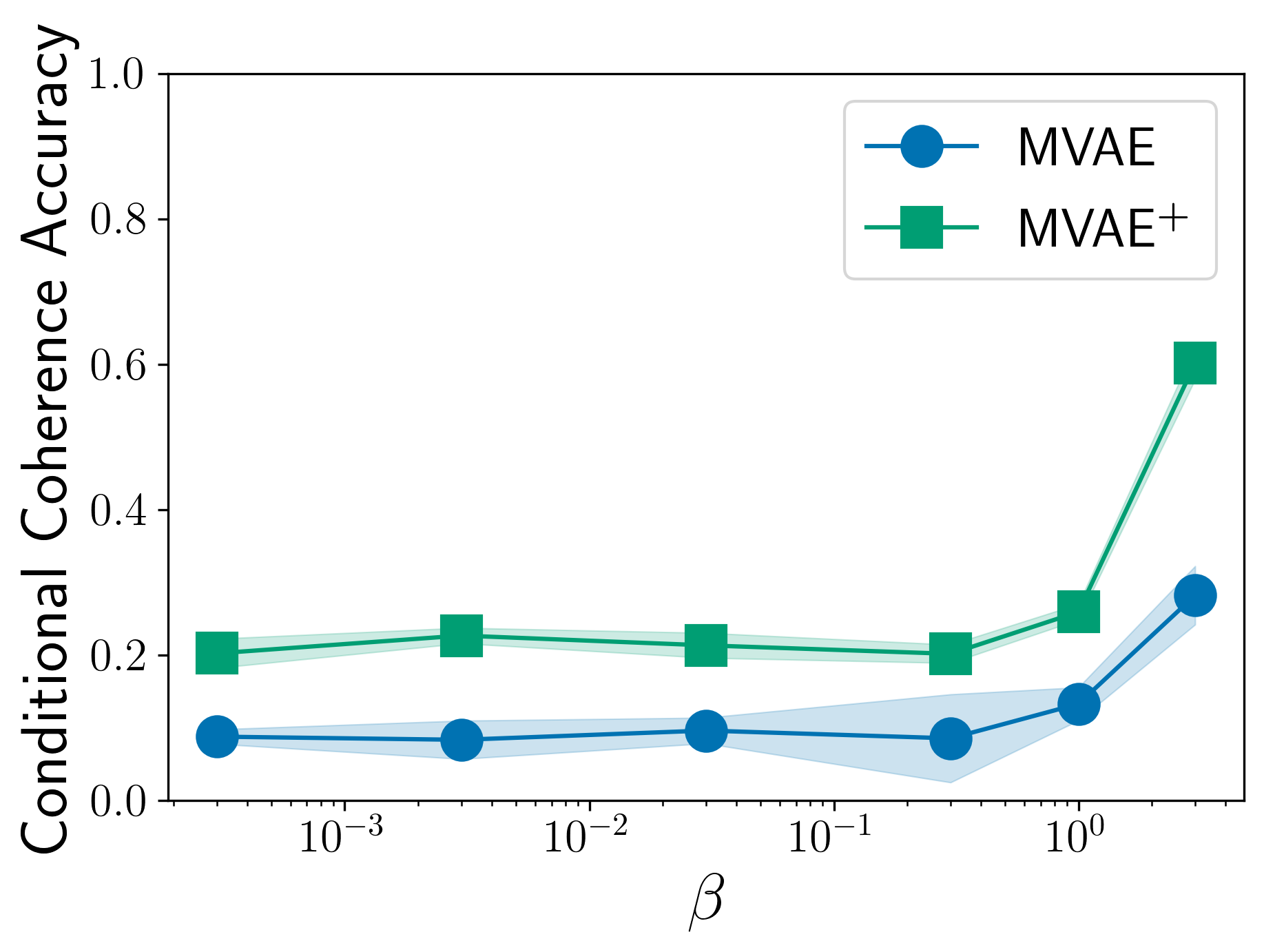}
  \caption{Generative coherence}
\end{subfigure}
\caption{PolyMNIST $\beta$-ablation, comparing MVAE with and without additional
  ELBO sub-sampling. MVAE$^{{+}}$ denotes the model with additional ELBO
  sub-sampling.  Points denote the value of the respective metric averaged over
  three seeds and bands show one standard deviation respectively.
}
\label{fig:vanilla_polymnist_beta_ablation_mvaeplus}
\end{center}
\end{figure}

\end{document}